\documentclass[11pt]{article}

\usepackage{iftex}
\ifPDFTeX
  \usepackage[utf8]{inputenc}
  \usepackage[T1]{fontenc}
  \usepackage{lmodern}
\fi

\usepackage[a4paper,margin=25mm]{geometry}
\usepackage{setspace}
\usepackage{microtype}

\usepackage{amsmath,amssymb,amsfonts,amsthm,mathtools,bm}

\usepackage{booktabs}
\usepackage{array}
\usepackage{enumitem}
\usepackage{algorithm}
\usepackage{algpseudocode}
\usepackage{graphicx}

\usepackage{xcolor}
\usepackage{kotex}

\usepackage[numbers,sort&compress]{natbib}
\usepackage[colorlinks=true,linkcolor=blue,citecolor=red,urlcolor=blue]{hyperref}
\usepackage[nameinlink,capitalize,noabbrev]{cleveref}
\crefname{assumption}{Assumption}{Assumptions}
\Crefname{assumption}{Assumption}{Assumptions}
\crefname{algorithm}{Algorithm}{Algorithms}
\Crefname{algorithm}{Algorithm}{Algorithms}
\makeatletter
\providecommand*{\theHalgorithm}{\thealgorithm}
\@ifundefined{theHALG@line}{%
  \newcommand*{\theHALG@line}{\theHalgorithm.\arabic{ALG@line}}%
}{%
  \renewcommand*{\theHALG@line}{\theHalgorithm.\arabic{ALG@line}}%
}
\makeatother
\numberwithin{equation}{section}
\newtheorem{definition}{Definition}[section]
\newtheorem{assumption}{Assumption}[section]
\newtheorem{proposition}{Proposition}[section]
\newtheorem{lemma}{Lemma}[section]
\newtheorem{theorem}{Theorem}[section]
\newtheorem{corollary}{Corollary}[section]

\newcommand{\R}{\mathbb{R}}
\newcommand{\jsr}{\rho}
\newcommand{\E}{\mathbb{E}}

\newcommand{\co}{\operatorname{co}}
\newcommand{\diag}{\operatorname{diag}}

\title{Heavy-Ball Q-Learning with Residual Weighting Correction}

\author{%
\small Donghwan Lee\\
\small Department of Electrical Engineering\\
\small Korea Advanced Institute of Science and Technology (KAIST)\\
\small Daejeon 34141, South Korea (email: \texttt{donghwan@kaist.ac.kr})
}

\date{}

\begin{document}

\maketitle

\begin{abstract}
This paper proposes a corrected heavy-ball Q-learning method for reinforcement
learning (RL) and establishes convergence of its deterministic mean dynamics. It also identifies conditions
under which the method is theoretically guaranteed to converge faster than
standard Q-learning. The same construction is then extended to Q-learning with
linear function approximation, where analogous convergence and acceleration
statements are derived for the corresponding corrected fixed point. The sampled stochastic versions are treated through conditional-mean recursions and, in the stated linear-function-approximation setting, finite-time bounds. The analysis is based on a switched linear system (SLS)
representation of Q-learning algorithms and on the joint spectral radius (JSR)
of the associated switching families. This SLS viewpoint is not commonly used in
standard analyses of Q-learning, and it provides a complementary framework and
new insight into how heavy-ball momentum can accelerate Q-learning.
\end{abstract}

\begin{center}
\end{center}

\section{Introduction}
\label{sec:introduction}

Q-learning~\citep{watkins1992qlearning} is one of the most fundamental and
widely used methods in reinforcement learning (RL)~\citep{sutton2018reinforcement}.
In the tabular discounted setting, it learns the optimal Q-function through
Bellman optimality updates without requiring an explicit transition model, and
it also underlies modern deep variants such as deep
Q-learning~\citep{mnih2015human}. The usual convergence argument relies on the
contraction property of the Bellman optimality operator, but this contraction can
be weak when the discount factor is close to one. Therefore, many methods have
tried to improve the speed of Q-learning or closely related value-iteration
recursions. Examples include speedy Q-learning~\citep{azar2011speedy},
momentum-based Q-learning~\citep{weng2019momentum}, proportional-integral-derivative (PID)-type value
iteration~\citep{farahmand2021pid}, first-order accelerated value
iteration~\citep{goyal2022firstorder}, and rank-one correction or deflation
methods~\citep{bertsekas1995rankone,lee2026deflated}. However, theoretical
guarantees that a modified recursion is faster than standard Q-learning remain
comparatively limited, especially for control problems in which the active greedy
policy can change along the trajectory.

This paper proposes a corrected heavy-ball Q-learning method obtained by
modifying the simplest heavy-ball Q-learning recursion by using the actual momentum coefficient $\alpha\eta$. The
correction is designed so that the associated mean mappings share a common
eigenvector. Along this common eigenvector, we derive conditions under which the
corrected heavy-ball recursion has a smaller certified rate than the standard
Q-learning recursion. The common eigenvector makes the analysis tractable and is
the reason such a theoretical acceleration guarantee can be obtained. In the
directions orthogonal to this common eigenvector, the present argument does not
give the same acceleration guarantee. The contribution is therefore a precise
certificate in one analytically identifiable direction, rather than a
full-direction acceleration theorem, and this gives a useful perspective on the
geometry of the heavy-ball Q-learning.

For this analysis, we model the heavy-ball Q-learning mean dynamics as a
switched linear system (SLS)~\citep{liberzon2003switching,lin2009stability,shorten2007stability}.
The greedy action in the Bellman maximum determines the active mode of the SLS,
and the resulting matrix family is analyzed through the joint spectral radius
(JSR)~\citep{rota1960note,jungers2009joint}. This switching-system viewpoint
provides a new route for studying the proposed heavy-ball Q-learning recursion
and gives additional intuition about when acceleration can be certified.

We also extend the method to linear function approximation and obtain analogous
conclusions for the corrected fixed point defined by the modified update, not
necessarily for the standard projected Bellman fixed point. The main analysis is
written for deterministic mean dynamics, which keeps the role of the common
eigenvector clear. The same update can be implemented as a model-free stochastic
RL recursion. For the constant-step-size tabular sampled recursion, the result is
a conditional-mean certificate; finite-time stochastic bounds are proved only in
the stated sampled linear-function-approximation setting. We also give a short
discussion of the independent and identically distributed (i.i.d.) observation
version. A Markovian observation version can be developed with standard stochastic approximation tools, as in
\citep{lee2026lyapunovcertified}, but that extension is not developed here.

\section{Related work}
\label{sec:related_works}

Momentum-based modifications of value iteration have been studied from
several perspectives. The proportional-integral-derivative (PID) accelerated value iteration framework
in~\citep{farahmand2021pid} interprets value iteration as a feedback-control
system and introduces proportional-derivative (PD), proportional-integral (PI),
and related PID variants. In fixed-policy settings, the PD/heavy-ball term changes the
eigenvalues of the value-iteration error dynamics through a second-order
companion polynomial associated with each eigenvalue of the transition matrix.
Momentum value computation and accelerated value iteration were also studied
in~\citep{goyal2022firstorder}, which connects value iteration with first-order
optimization methods and analyzes momentum-type value-iteration updates for
policy evaluation. Momentum-based accelerated Q-learning has also been studied
in stochastic approximation settings~\citep{weng2019momentum,bowen2021finitetime}.

Other acceleration mechanisms for dynamic programming iterations are related but
technically different. Adaptive relaxation and lookahead methods for value
iteration were studied in~\citep{herzberg1994accelerating,shlakhter2010acceleration};
Anderson acceleration applies a general fixed-point acceleration scheme to
dynamic programming iterations~\citep{geist2018anderson}; momentum value
iteration in~\citep{vieillard2020momentum} averages successive $Q$-functions
rather than using the heavy-ball Q-value iteration update analyzed here; and
anchored/Halpern-type value iteration gives another provably accelerated
Bellman-error mechanism for both Bellman consistency and optimality
operators~\citep{lee2023anchoring}. Speedy Q-learning is also relevant because it
uses two successive Q-estimates to obtain faster finite-sample bounds in a
model-free setting~\citep{azar2011speedy}. These works support the broader idea
that past iterates or modified Bellman updates can accelerate Bellman-type
recursions.

In contrast, this paper proposes a modified heavy-ball Q-learning recursion that
yields a theoretically verifiable acceleration certificate. To the best of our
knowledge, the algorithm studied here has not appeared previously, and the
convergence and acceleration analysis developed below is new. Unlike much of the
existing momentum-based value-iteration and Q-learning literature, which focuses mainly on policy
evaluation or stochastic approximation bounds, this paper treats the control
setting for a simple constant-step-size heavy-ball Q-learning recursion. The proof
technique is also different: we represent the recursion as a switched linear
system and compare the resulting switching families through the joint spectral
radius rather than relying on fixed-policy spectral or contraction arguments.

The slow constant direction of discounted dynamic programming iterations has a
long history. Rank-one correction and extrapolation methods exploit the dominant
stochastic matrix direction to improve value-iteration behavior~\citep{bertsekas1995rankone}.
Recent rank-one or deflation-based methods also use dominant-direction
information, including rank-one modified value iteration and deflated dynamics
value iteration~\citep{kolarijani2025rankone,lee2025deflateddynamics}. More
recent deflation-based approaches remove or quotient out the component along the
common eigenvector in order to expose faster transverse dynamics; see, for
example, the switching-geometry analysis of deflated Q-value iteration (Q-VI)~\citep{lee2026deflated}.
The present paper is related in that every standard Q-VI switching matrix treats
the constant-vector direction in the same way. This shared direction fixes the
ambient-space benchmark rate, even though the remaining policy-dependent
directions may decay faster. The mechanism studied here is different:
heavy-ball momentum does not deflate or remove this direction, but turns it into
a second-order linear time-invariant dynamics.

The JSR is a standard measure of worst-case exponential growth
for products of matrices under arbitrary switching; see, for
example,~\citep{jungers2009joint}. Value-iteration updates naturally lead to
SLS representations because the greedy action can change with
the iterate. The SLS analysis used here is close to the direct switching representation of
Q-learning~\citep{lee2026lyapunovcertified}. The
projected-error viewpoint is also related to recent JSR
analyses of Q-VI geometry, including~\citep{lee2026geometry,lee2026deflated}.

\section{Preliminaries}
\label{sec:preliminaries}
\subsection{Notation}
\label{sec:notation}

The set of real numbers is denoted by $\R$; $\R^m$ is the $m$-dimensional
Euclidean space; and $\R^{m\times n}$ is the set of all $m\times n$ real
matrices. For a matrix $A$, $A^\top$ denotes its transpose. The identity
matrix is denoted by $I$. For a finite state set $\mathcal S$ and a finite
action set $\mathcal A$, $e_s$ and $e_a$ denote the standard basis vectors
associated with $s\in\mathcal S$ and $a\in\mathcal A$, respectively.  With the
action-block ordering used throughout, for the state-action index $i=(s,a)$,
write $e_i=e_a\otimes e_s$.  More generally,
standard basis vectors in other Euclidean spaces use the same notation when the
dimension is clear from context, and $\otimes$ denotes the Kronecker product.
For a finite set $\mathcal S$, $|\mathcal S|$ denotes its cardinality. For finite tabular state and action sets, set $n:=|\mathcal S||\mathcal A|$. Moreover, we write $\Delta_m:=\left\{q\in\R^m:q_i\geq0,\ \sum_{i=1}^m q_i=1\right\}$ for the probability simplex in $\R^m$. For a finite matrix family
$\mathcal H=\{A_1,\ldots,A_N\}$, $\co(\mathcal H)
:=\left\{\sum_{i=1}^N\lambda_i A_i:
\lambda_i\geq0,\ \sum_{i=1}^N\lambda_i=1\right\}$ denotes its convex hull. We also use standard matrix notation that appears repeatedly below. For a
vector $x$, $\|x\|_2$ is the Euclidean norm. For a square matrix $A$,
$\rho(A)$ denotes its ordinary spectral radius. After the matrix-family radius is
introduced in~\Cref{def:jsr}, the notation $\rho(\mathcal H)$ is used for that
quantity when the argument is a switching family.
For a matrix $B$, $\operatorname{range}(B)$ denotes its column space,
$B\succ0$ means that $B$ is symmetric positive definite, and
$\lambda_{\max}(B)$ denotes the largest eigenvalue when $B$ is symmetric.
For square matrices $A$ and $B$ of the same size, we write $A\sim B$ when
they are similar, that is, when $B=S^{-1}AS$ for some nonsingular matrix $S$.
Expectations are denoted by $\mathbb E[\cdot]$.

\subsection{Switched linear systems}
\label{sec:switching_systems}

The stability certificates used later are stated in the language of switched
systems, so we first recall the basic model before specializing it to the
Bellman-induced switching families. Consider the discrete-time switching
affine system (SAS)~\citep{liberzon2003switching,lin2009stability,shorten2007stability}
\begin{align*}
  x_{k+1}=A_{\sigma_k}x_k+b_{\sigma_k},
\end{align*}
where each index $i\in\{1,2,\ldots,M\}$, equivalently each affine pair
$(A_i,b_i)$, is called a \emph{mode}, and $\sigma_k$ is the switching signal that
selects the active mode at time $k$. The matrix $A_{\sigma_k}$ is selected from
the prescribed family $\mathcal H:=\{A_1,A_2,\ldots,A_M\}$, which is called a
\emph{switching family}; $b_{\sigma_k}$ is a mode-dependent affine term. When
$b_{\sigma_k}=0$, the deterministic part reduces to a switched linear system
(SLS), $x_{k+1}=A_{\sigma_k}x_k$. The worst-case exponential rate of the SLS
family is characterized by the joint spectral radius (JSR), defined as follows.
\begin{definition}
\label{def:jsr}
For a bounded set of matrices $\mathcal H\subset\R^{m\times m}$, its JSR is
\begin{align*}
\jsr(\mathcal H)
:=
\lim_{k\to\infty}
\sup_{A_1,\ldots,A_k\in\mathcal H}
\|A_k\cdots A_1\|^{1/k}.
\end{align*}
\end{definition}
We note that the JSR is independent of the chosen submultiplicative
norm~\citep{rota1960note,jungers2009joint}. When $\mathcal H$ is finite, the
supremum for each fixed product length is a maximum over products generated by
matrices in $\mathcal H$. For a finite family $\mathcal H$, the notation
$\jsr(\co(\mathcal H))$ means the JSR computed when each factor in a product is
allowed to be any convex combination of matrices in $\mathcal H$.
Throughout the later JSR certificates, $\rho(\mathcal H)$ denotes this same JSR
value when the argument is a switching family.

Three standard JSR facts used below, convex-hull invariance, similarity invariance, and block upper triangular decomposition, are collected in~\Cref{app:standard-jsr-and-selector-lemmas}. In particular, because the policy-mode families used below are finite, passing from deterministic modes to their convex hull does not change the JSR.

\subsection{Joint spectral radius and Lyapunov certificates}
\label{sec:joint_spectral_radius}

The JSR in~\Cref{def:jsr} turns arbitrary switched products into a
single worst-case exponential rate. An SLS is uniformly exponentially stable under arbitrary switching if there
exist constants $C\geq1$
and $\eta\in(0,1)$ such that $\|A_{\sigma_{k-1}}\cdots A_{\sigma_0}x\|_2\leq C\eta^k\|x\|_2$ for every horizon $k\geq0$, every initial state $x\in\R^m$, and every switching
sequence. A \emph{common Lyapunov function} for $\mathcal H$ is a positive definite
function that decreases along every mode. In the analysis below, the Bellman
maximum induces stochastic-policy switching, and the Lyapunov functions are built from products of the corresponding mode matrices.
The following finite-family piecewise-quadratic construction~\citep{lee2026lyapunovcertified,hushenzhang2010generating} is the Lyapunov certificate used in the deterministic arguments.
\setcounter{lemma}{2}
\begin{lemma}
\label{lem:common_lyapunov_construction}
Let $\mathcal H=\{A_1,A_2,\ldots,A_M\}\subset\R^{m\times m}$ and fix
$\epsilon>0$ such that $\beta_\epsilon:=\rho(\mathcal H)+\epsilon\in(0,1)$.
For a word $\sigma=(\sigma_1,\ldots,\sigma_k)\in\{1,\ldots,M\}^k$, write
\begin{align*}
A_\sigma:=A_{\sigma_k}\cdots A_{\sigma_1},
\end{align*}
with the convention that the empty word gives $A_\sigma=I$. Define
\begin{align*}
V_\epsilon^\infty(x)
:=\sum_{k=0}^\infty \beta_\epsilon^{-2k}
\max_{\sigma\in\{1,\ldots,M\}^k}\|A_\sigma x\|_2^2,
\qquad x\in\R^m.
\end{align*}
Then $V_\epsilon^\infty$ is finite for every $x$, and there exists
$C_\epsilon>0$ such that
\begin{align*}
\|x\|_2^2\leq V_\epsilon^\infty(x)\leq C_\epsilon\|x\|_2^2,
\qquad \forall x\in\R^m.
\end{align*}
The function $p_\epsilon(x):=\sqrt{V_\epsilon^\infty(x)}$ is a norm on
$\R^m$, and every mode satisfies
\begin{align*}
p_\epsilon(A_i x)\leq\beta_\epsilon p_\epsilon(x),
\qquad \forall x\in\R^m,
\qquad i=1,\ldots,M.
\end{align*}
\end{lemma}
\begin{proof}
The construction and proof are given in~\citep{lee2026lyapunovcertified,hushenzhang2010generating}; we
omit the proof here.
\end{proof}
Throughout the sequel, whenever this construction is applied to a switching
family with JSR less than one, we call the resulting $V_\epsilon^\infty$ a \emph{JSR
Lyapunov function} for that family, and we call the associated norm $p_\epsilon$
a \emph{JSR Lyapunov norm}. The following lemma is the common bridge from a JSR bound to convergence of the corresponding error recursion.
\begin{lemma}
\label{lem:standard_jsr_convergence}
\label{lem:basic_jsr_convergence}
Let $\mathcal H=\{A_1,\ldots,A_M\}\subset\R^{n\times n}$ be finite and suppose
$\rho(\mathcal H)<1$. Consider any error recursion
\begin{align*}
x_{k+1}=A_kx_k,
\qquad A_k\in\co(\mathcal H),
\qquad k\in\{0,1,\ldots\}.
\end{align*}
Then, for every $\epsilon>0$ such that
$\beta_\epsilon:=\rho(\mathcal H)+\epsilon<1$, the Lyapunov function
$V_\epsilon^\infty$ and norm $p_\epsilon$ from~\Cref{lem:common_lyapunov_construction}, applied to
$\mathcal H$, satisfy
\begin{align*}
V_\epsilon^\infty(x_{k+1})
\leq
\beta_\epsilon^2 V_\epsilon^\infty(x_k),
\qquad
p_\epsilon(x_{k+1})\leq \beta_\epsilon p_\epsilon(x_k).
\end{align*}
Consequently, if $C_\epsilon$ is the constant from~\Cref{lem:common_lyapunov_construction}, then
\begin{align*}
p_\epsilon(x_k)\leq \beta_\epsilon^k p_\epsilon(x_0),
\qquad
\|x_k\|_2
\leq
\beta_\epsilon^k p_\epsilon(x_0)
\leq
\sqrt{C_\epsilon}\,\beta_\epsilon^k\|x_0\|_2,
\end{align*}
and hence $x_k\to0$.
\end{lemma}
\begin{proof}
By~\Cref{lem:common_lyapunov_construction}, each generator satisfies
\(p_\epsilon(A_i x)\le \beta_\epsilon p_\epsilon(x)\).  If
\(A=\sum_{i=1}^M\lambda_i A_i\in\co(\mathcal H)\), where
\(\lambda_i\ge0\) and \(\sum_i\lambda_i=1\), then convexity of the norm
\(p_\epsilon\) gives
\begin{align*}
p_\epsilon(Ax)
&=p_\epsilon\!\left(\sum_{i=1}^M\lambda_i A_i x\right) \\
&\le \sum_{i=1}^M\lambda_i p_\epsilon(A_i x) \\
&\le \beta_\epsilon p_\epsilon(x).
\end{align*}
Applying this bound with \(A=A_k\) yields
\(p_\epsilon(x_{k+1})\le \beta_\epsilon p_\epsilon(x_k)\).  Squaring gives the
corresponding inequality for \(V_\epsilon^\infty=p_\epsilon^2\), and iteration
and~\Cref{lem:common_lyapunov_construction} give the Euclidean estimate.
\end{proof}

The rest of the paper uses this JSR implication as the main convergence
certificate. Once an error recursion has been written as an SLS whose mode
family has JSR below one, convergence and an exponential rate bound follow
from~\Cref{lem:basic_jsr_convergence}. After each deterministic mean
recursion is introduced, the corresponding sampled stochastic RL recursion is
also given. For the sampled versions, a JSR bound below one for the associated
conditional-mean switching family gives the drift certificate used in the
stochastic analysis. Exact stochastic convergence is not asserted for the
constant-step-size tabular sampled recursion; finite-time error bounds are proved
only in the sampled linear-function-approximation setting in
\Cref{sec:sampled_stochastic_hb_lfa}. This paper treats only the independent and identically distributed (i.i.d.)
sampling case to keep the formulas transparent. The same conditional-mean
and noise decomposition can be combined with Markovian-observation
stochastic-approximation arguments, as in~\citep{chen2024lyapunov}, to
extend the setting beyond i.i.d. samples.

\subsection{Discounted Markov decision processes}
\label{sec:discounted_mdps}

We consider a finite discounted Markov decision process (MDP), the standard model for reinforcement learning (RL) control~\citep{puterman1994markov,bertsekas1996neuro}, with state-space
$\mathcal S=\{1,\ldots,|\mathcal S|\}$, action-space
$\mathcal A=\{1,\ldots,|\mathcal A|\}$, transition probability
$P(s'\mid s,a)$, real-valued one-step reward $r(s,a,s')$, expected reward $R(s,a):=\sum_{s'\in\mathcal S}P(s'\mid s,a)r(s,a,s')$, and discount factor $\gamma\in(0,1)$. State-action functions are viewed as
vectors in $\R^{n}$ using the action-block ordering $(1,1),(2,1),\ldots,(|\mathcal S|,1),
(1,2),(2,2),\ldots,(|\mathcal S|,|\mathcal A|)$. All matrices and vectors indexed by state-action pairs use this ordering. Define
\begin{align*}
P:=
\begin{bmatrix}
P_1\\
\vdots\\
P_{|\mathcal A|}
\end{bmatrix}
\in\R^{n\times |\mathcal S|},
\qquad
R:=
\begin{bmatrix}
R(\cdot,1)\\
\vdots\\
R(\cdot,|\mathcal A|)
\end{bmatrix}
\in\R^{n},
\end{align*}
where $P_a=P(\cdot\mid\cdot,a)\in\R^{|\mathcal S|\times|\mathcal S|}$.
Let $\Theta$ denote the set of deterministic stationary policies
$\pi:\mathcal S\to\mathcal A$. For any stochastic policy
$\mu:\mathcal S\to\Delta_{|\mathcal A|}$, define
\begin{align*}
\Pi^\mu:=
\begin{bmatrix}
\mu(1)^\top\otimes e_1^\top\\
\mu(2)^\top\otimes e_2^\top\\
\vdots\\
\mu(|\mathcal S|)^\top\otimes e_{|\mathcal S|}^\top
\end{bmatrix}
\in\R^{|\mathcal S|\times n}.
\end{align*}
For a deterministic policy $\pi\in\Theta$, the same notation $\Pi^\pi$ is used by
identifying $\pi(s)$ with its one-hot encoding.
For
$Q\in\R^{n}$, define
\begin{align*}
V_Q(s):=\max_{a\in\mathcal A}Q(s,a),
\qquad
V_Q:=(V_Q(1),\ldots,V_Q(|\mathcal S|))^\top.
\end{align*}
The Bellman optimality operator is $F(Q):=R+\gamma P V_Q$.
Let $Q^\star$ denote the unique fixed point of the Bellman optimality operator, $Q^\star=F(Q^\star)$. For a tie-broken greedy policy $\pi_Q$ satisfying
\begin{align*}
  \pi_Q(s)\in\operatorname*{arg\,max}_{a\in\mathcal A}Q(s,a),
  \qquad s\in\mathcal S,
\end{align*}
one has $V_Q=\Pi^{\pi_Q}Q$ and
\begin{equation*}
  F(Q)=R+\gamma P\Pi^{\pi_Q}Q.
\end{equation*}
The following lemma is used repeatedly to write Bellman differences as linear maps depending on stochastic policies.
\begin{lemma}
\label{lem:bellman_difference_selector}
For any two vectors $Q,\bar Q\in\R^{n}$, there exists a stochastic policy
$\mu_{Q,\bar Q}$ such that
\begin{equation*}
  V_Q-V_{\bar Q}=\Pi^{\mu_{Q,\bar Q}}(Q-\bar Q).
\end{equation*}
Consequently,
\begin{equation}
  F(Q)-F(\bar Q)=\gamma P\Pi^{\mu_{Q,\bar Q}}(Q-\bar Q).
  \label{eq:bellman-difference-selector-general}
\end{equation}
Moreover, for every stochastic policy $\mu$,
\begin{equation*}
  P\Pi^\mu\in\co\left\{P\Pi^\pi:\pi\in\Theta\right\},
\end{equation*}
and $P\Pi^\mu$ is row-stochastic and satisfies $P\Pi^\mu\mathbf{1}=\mathbf{1}$.
\end{lemma}
\begin{proof}
The proof can be found in~\citep{lee2026lyapunovcertified}.
\end{proof}

We use \(d\) to denote a state-action sampling distribution on ${\mathcal S}\times {\mathcal A}$. In the i.i.d. observation model, \(d\) is the sampling distribution of \((s_k,a_k)\); in the Markovian observation model, \(d\) is the stationary state-action distribution of the behavior-induced chain.
Throughout the paper, we assume that the sampling distribution satisfies \(d(s,a)>0\) for every \((s,a)\in {\mathcal S}\times {\mathcal A}\).
This distribution forms the full-support state-action weight vector $d\in\R^n$ with $d^\top\mathbf{1}=1$, and we define
\begin{equation*}
  D:=\operatorname{diag}(d),
  \qquad
  d_{\max}:=\max_{(s,a)\in\mathcal S\times\mathcal A} d(s,a) .
\end{equation*}

\section{Corrected Q-learning}
\label{sec:constant_preserving_d_weighting}

In this section, we first introduce the deterministic standard Q-learning (QL) update in~\Cref{alg:direct-d-standard-update} as the baseline against which the subsequent convergence rates will be
compared.
\begin{algorithm}[H]
\caption{Deterministic standard Q-learning}
\label{alg:direct-d-standard-update}
\begin{algorithmic}
\Require $Q_0\in\R^n$, step-size $\alpha>0$, and $D=\diag(d)$
\For{$k=0,1,2,\ldots$}
  \State Update
  \[
    Q_{k+1}\gets Q_k+\alpha D\{F(Q_k)-Q_k\}.
  \]
\EndFor
\end{algorithmic}
\end{algorithm}
This deterministic recursion is a natural starting point for stochastic RL
extensions with sampled coordinate updates. However, after inserting $D$, the
common all-ones eigenvector argument in~\citep{lee2026spectralhbqvi}, which is
central to the acceleration analysis developed below, no longer applies in
general: the corresponding matrices need not map the all-ones vector to a scalar
multiple of itself unless the diagonal weights in $D$ are uniform.
Motivated by this obstruction, this section introduces a modification that keeps
a nonuniform diagonal weighting on the Bellman residual components while
restoring the common eigenvector needed for the SLS-based JSR analysis below.
\begin{definition}
\label{def:constant-preserving-d-preconditioner}
Define the matrix
\begin{equation*}
  \Pi_d:=\mathbf{1}d^\top .
\end{equation*}
For a scalar $\delta>0$, define the correction matrix
\begin{equation}
  H:=\delta\Pi_d+D(I-\Pi_d)
  =D+(\delta\mathbf{1}-d)d^\top .
  \label{eq:constant-preserving-d-preconditioner}
\end{equation}
\end{definition}
The next lemma states the property that makes this correction useful: multiplication by $H$ maps a constant vector to a constant vector.
\begin{lemma}
\label{lem:hd-preserves-one}
The matrix $H$ in~\Cref{eq:constant-preserving-d-preconditioner} satisfies
\begin{equation}
  H\mathbf{1}=\delta\mathbf{1}.
  \label{eq:hd-preserves-one}
\end{equation}
\end{lemma}
\begin{proof}
Using $d^\top\mathbf 1=1$, one has $\Pi_d\mathbf 1=\mathbf 1d^\top\mathbf 1=\mathbf 1$ and $(I-\Pi_d)\mathbf 1=\mathbf 1-\mathbf 1=0$.
Substituting these identities into~\Cref{eq:constant-preserving-d-preconditioner}
gives $H\mathbf 1
  =\delta\Pi_d\mathbf 1+D(I-\Pi_d)\mathbf 1
    =\delta\mathbf 1$, which completes the proof.
\end{proof}
\begin{lemma}
\label{lem:cpd-H-nonsingular}
The matrix $H$ in~\Cref{eq:constant-preserving-d-preconditioner} is nonsingular.
\end{lemma}
\begin{proof}
Take an arbitrary $x\in\R^n$ and separate it into its $d$-weighted constant
part and the remaining $d$-mean-free part. Set
$a:=d^\top x$ and $z:=x-a\mathbf 1$. Then
\begin{equation*}
  x=a\mathbf 1+z=(d^\top x)\mathbf 1+z,
  \qquad
  d^\top z=d^\top x-a d^\top\mathbf 1=0.
\end{equation*}
This decomposition is unique: if $x=a\mathbf 1+z$ with $d^\top z=0$, then
multiplying by $d^\top$ gives $a=d^\top x$. Hence $\Pi_d x=a\mathbf 1$ and
$(I-\Pi_d)x=z$. The identity $Hx=0$ is therefore equivalent to
\begin{equation*}
  \delta(d^\top x)\mathbf 1+Dz=0.
\end{equation*}
Multiplication by $d^\top D^{-1}$ gives
\begin{align*}
  0
  &=\delta(d^\top x)d^\top D^{-1}\mathbf 1+d^\top z
    =\delta(d^\top x)d^\top D^{-1}\mathbf 1.
\end{align*}
Because $d$ has full support, $D^{-1}$ exists and
$d^\top D^{-1}\mathbf 1=\sum_i d_i(1/d_i)=n>0$. Since $\delta>0$, we obtain
$d^\top x=0$. Then $Dz=0$, and the positivity of the diagonal entries of $D$
implies $z=0$. Hence $x=0$, so $H$ is nonsingular.
\end{proof}
\begin{algorithm}[H]
\caption{Deterministic corrected Q-learning}
\label{alg:cpd-standard-update}
\begin{algorithmic}
\Require $Q_0\in\R^n$, step-size $\alpha>0$, and $H$ from~\Cref{eq:constant-preserving-d-preconditioner}
\For{$k=0,1,2,\ldots$}
  \State Update
  \[
    Q_{k+1}\gets Q_k+\alpha H\{F(Q_k)-Q_k\}.
  \]
\EndFor
\end{algorithmic}
\end{algorithm}
The proposed construction replaces the stochastic weighting matrix $D$ in the update of~\Cref{alg:direct-d-standard-update} with the correction matrix $H$. The resulting algorithm is given in~\Cref{alg:cpd-standard-update} and is called corrected Q-learning (CQL). To analyze its fixed point, define the corresponding map
\begin{equation}
  g(Q):=Q+\alpha H\{F(Q)-Q\}.
  \label{eq:cpd-standard-map}
\end{equation}
The next lemma shows that the correction does not move the Bellman fixed point $Q^\star$.
\begin{lemma}
\label{lem:cpd-standard-fixed-point}
The map $g$ in~\Cref{eq:cpd-standard-map} has the unique fixed point $Q^\star$.
\end{lemma}
\begin{proof}
If $\bar Q$ is a fixed point of $g$, then $H\{F(\bar Q)-\bar Q\}=0$.
By~\Cref{lem:cpd-H-nonsingular}, $F(\bar Q)-\bar Q=0$ since $H$ is nonsingular. The discounted Bellman
optimality operator has the unique fixed point $Q^\star$, so $\bar Q=Q^\star$.
Conversely, $F(Q^\star)=Q^\star$ immediately implies that $Q^\star$ is fixed by
$g$.
\end{proof}

Subtracting $Q^\star=F(Q^\star)$ from the update in~\Cref{alg:cpd-standard-update} gives the error recursion
\begin{align}
  Q_{k+1}-Q^\star=(Q_k-Q^\star) +\alpha H\{F(Q_k)-F(Q^\star)-(Q_k-Q^\star)\}.\label{eq:1}
\end{align}

We next show that the error recursion in~\Cref{eq:1} admits an exact SLS representation.
\begin{lemma}
\label{lem:cpd-standard-sls}
The error recursion in~\Cref{eq:1} satisfies the SLS
\begin{equation}
  Q_{k+1}-Q^\star=A_{\mu_k}^{\rm CQL}(Q_k-Q^\star).
  \label{eq:cpd-standard-error-system}
\end{equation}
where for any stochastic policy $\mu$,
\begin{equation}
  A_\mu^{\rm CQL}
  :=I+\alpha H(\gamma P\Pi^\mu-I)
  \label{eq:cpd-standard-mode}
\end{equation}
\end{lemma}
\begin{proof}
Consider the error recursion in~\Cref{eq:1}. By~\Cref{eq:bellman-difference-selector-general}, we have $F(Q_k)-F(Q^\star)=\gamma P\Pi^{\mu_k}(Q_k-Q^\star)$, and therefore
\begin{align*}
  Q_{k+1}-Q^\star
  &=\{I+ \alpha H(\gamma P\Pi^{\mu_k}-I)\}(Q_k-Q^\star)\\
  &=A_{\mu_k}^{\rm CQL}(Q_k-Q^\star),
\end{align*}
which is~\Cref{eq:cpd-standard-error-system}.
\end{proof}

The corresponding switching family associated with deterministic policies is
\begin{equation*}
  \mathcal A^{\rm CQL}
  :=\left\{A_\pi^{\rm CQL}:\pi\in\Theta\right\}.
\end{equation*}
It follows directly that, for every stochastic selector $\mu_k$,
$A_{\mu_k}^{\rm CQL}\in\co(\mathcal A^{\rm CQL})$.
Using the identity in~\Cref{eq:hd-preserves-one}, one can prove that every
switching mode in $\co(\mathcal A^{\rm CQL})$ has a common eigenvector.
\begin{lemma}
\label{lem:cpd-common-vector}
For every stochastic policy $\mu$, it holds that
\begin{equation}
  A_\mu^{\rm CQL}\mathbf{1}=\bigl(1-\alpha\delta(1-\gamma)\bigr)\mathbf{1}.
  \label{eq:cpd-common-vector}
\end{equation}
\end{lemma}
\begin{proof}
For every stochastic policy $\mu$, one has $(\gamma P\Pi^\mu-I)\mathbf 1
  =\gamma\mathbf 1-\mathbf 1
  =-(1-\gamma)\mathbf 1$. Using~\Cref{eq:hd-preserves-one}, we obtain
\begin{align*}
  A_\mu^{\rm CQL}\mathbf 1
  &=\left[I+\alpha H(\gamma P\Pi^\mu-I)\right]\mathbf 1\\
  &=\mathbf 1-\alpha(1-\gamma)H\mathbf 1\\
  &=\left(1-\alpha\delta(1-\gamma)\right)\mathbf 1.
\end{align*}
This proves the claim.
\end{proof}
Therefore, the all-ones direction is again a common invariant direction,
equivalently a common eigenvector, with scalar factor
$1-\alpha\delta(1-\gamma)$. Throughout the paper, we require this scalar
constant-direction factor to lie strictly between zero and one. This condition
stabilizes only the common all-ones mode; stability of the full switching family
also depends on the projected dynamics introduced below.
\begin{assumption}
\label{ass:cpd-constant-mode-range}
The parameters $\alpha$ and $\delta$ satisfy
\begin{equation*}
  0<1-\alpha\delta(1-\gamma)<1.
\end{equation*}
\end{assumption}

The correction matrix $H$ should be viewed as a structural correction
rather than as a universal accelerator. It is introduced in place of $D$ to
recover the common eigenvector property in~\Cref{eq:cpd-common-vector} for all
switching modes. This property alone does not imply that CQL is always faster
than QL. The value of $H$ is that it creates the decomposition needed to certify
the acceleration of the heavy-ball CQL recursion introduced next.
Indeed, adding a heavy-ball momentum term directly to QL may also improve convergence, and in some instances it may even be faster than CQL with a heavy-ball momentum term. The
difficulty is that, without the common eigenvector in~\Cref{eq:cpd-common-vector}, an analogous JSR-based certificate is harder to
obtain for the heavy-ball QL without the modification. The subsequent theory and numerical examples show cases in which CQL with momentum is both theoretically
certified and empirically observed to be faster than the corresponding corrected or standard baseline.

The analysis in this paper separates this all-ones direction from its orthogonal complement.
For this purpose, we introduce the following projection and the associated projected SLS.
Let $U\in\R^{n\times(n-1)}$ have orthonormal columns spanning
$\operatorname{span}\{\mathbf{1}\}^\perp$, so that $U^\top U=I_{n-1}$,
$U^\top\mathbf{1}=0$, and
$UU^\top=I-n^{-1}\mathbf{1}\mathbf{1}^\top$. Define the projected
standard family
\begin{equation}
  \bar A_\pi^{\rm CQL}:=U^\top A_\pi^{\rm CQL}U,
  \qquad
  \bar{\mathcal A}^{\rm CQL}
  :=\left\{\bar A_\pi^{\rm CQL}:\pi\in\Theta\right\}.
  \label{eq:cpd-projected-standard-family}
\end{equation}
The matrix $\bar A_\pi^{\rm CQL}$ is the mode induced by
$A_\pi^{\rm CQL}$ on the orthogonal component. The next lemma makes this
statement precise at each iteration.
\begin{lemma}
\label{lem:cpd-projected-error-sls}
Let $Q_k$ follow~\Cref{alg:cpd-standard-update}. The projected error
$U^\top(Q_k-Q^\star)$ satisfies the SLS
\begin{equation}
  U^\top(Q_{k+1}-Q^\star)
  =U^\top A_{\mu_k}^{\rm CQL}U\,U^\top(Q_k-Q^\star),
  \label{eq:cpd-projected-error-system}
\end{equation}
where each projected mode belongs to
$\co(\bar{\mathcal A}^{\rm CQL})$.
\end{lemma}
\begin{proof}
By~\Cref{eq:cpd-standard-error-system}, $Q_{k+1}-Q^\star=A_{\mu_k}^{\rm CQL}(Q_k-Q^\star)$. Because $q=n^{-1/2}\mathbf 1$ and the columns of $U$ are
orthonormal and orthogonal to $q$, the matrix $[q\ U]$ is orthogonal. Hence
\begin{equation*}
  qq^\top+UU^\top
  =[q\ U]
    \begin{bmatrix}q^\top\\ U^\top\end{bmatrix}
  =I.
\end{equation*}
Therefore,
\begin{align*}
  Q_k-Q^\star
  &=(qq^\top+UU^\top)(Q_k-Q^\star)\\
  &=q\,q^\top(Q_k-Q^\star)+U\,U^\top(Q_k-Q^\star)\\
  &=\frac{1}{n}\mathbf 1\mathbf 1^\top(Q_k-Q^\star)
    +UU^\top(Q_k-Q^\star).
\end{align*}
Equivalently, the first term is the component along $\operatorname{span}\{\mathbf 1\}$, and the second term is the component in
$\operatorname{span}\{\mathbf 1\}^\perp$.
The common-vector identity in~\Cref{eq:cpd-common-vector} gives
$A_{\mu_k}^{\rm CQL}\mathbf 1=
\{1-\alpha\delta(1-\gamma)\}\mathbf 1$, hence
$U^\top A_{\mu_k}^{\rm CQL}q=0$. Therefore
\begin{align*}
  U^\top(Q_{k+1}-Q^\star)
  &=U^\top A_{\mu_k}^{\rm CQL}
    \{q q^\top(Q_k-Q^\star)+UU^\top(Q_k-Q^\star)\}\\
  &=U^\top A_{\mu_k}^{\rm CQL}U\,U^\top(Q_k-Q^\star),
\end{align*}
which is~\Cref{eq:cpd-projected-error-system}. The map
$P\Pi^\mu\mapsto U^\top[I+\alpha H(\gamma P\Pi^\mu-I)]U$ is affine in
$P\Pi^\mu$. The convex-hull statement in~\Cref{lem:bellman_difference_selector} then implies
$U^\top A_{\mu_k}^{\rm CQL}U\in\co(\bar{\mathcal A}^{\rm CQL})$.
\end{proof}

This decomposition gives a direct expression for the full JSR $\rho(\mathcal A^{\rm CQL})$ in terms of the
constant scalar block corresponding to the common eigenvector and the projected SLS family $\bar{\mathcal A}^{\rm CQL}$.
\begin{lemma}
\label{lem:cpd-standard-jsr-decomposition}
The JSR $\rho(\mathcal A^{\rm CQL})$ satisfies the general identity
\begin{equation}
  \rho(\mathcal A^{\rm CQL})
  =\max\left\{\left|1-\alpha\delta(1-\gamma)\right|,\rho(\bar{\mathcal A}^{\rm CQL})\right\}.
  \label{eq:cpd-standard-jsr-decomposition-general}
\end{equation}
Under~\Cref{ass:cpd-constant-mode-range}, this reduces to
\begin{equation}
  \rho(\mathcal A^{\rm CQL})
  =\max\left\{1-\alpha\delta(1-\gamma),\rho(\bar{\mathcal A}^{\rm CQL})\right\}.
  \label{eq:cpd-standard-jsr-decomposition}
\end{equation}
\end{lemma}
\begin{proof}
Let
\begin{equation*}
  q:=n^{-1/2}\mathbf 1,
  \qquad
  S:=\begin{bmatrix}q&U\end{bmatrix}.
\end{equation*}
Since $U^\top U=I_{n-1}$, $U^\top\mathbf 1=0$, and
$UU^\top=I-n^{-1}\mathbf 1\mathbf 1^\top$, the matrix $S$ is orthogonal. Indeed,
\begin{equation*}
  S^\top S=
  \begin{bmatrix}
    q^\top q&q^\top U\\
    U^\top q&U^\top U
  \end{bmatrix}
  =I_n,
  \qquad
  SS^\top=qq^\top+UU^\top
  =n^{-1}\mathbf 1\mathbf 1^\top+I-n^{-1}\mathbf 1\mathbf 1^\top=I_n.
\end{equation*}
For every deterministic policy $\pi$, \Cref{eq:cpd-common-vector} gives $A_\pi^{\rm CQL}q
  =\bigl(1-\alpha\delta(1-\gamma)\bigr)q$. Therefore, in the orthogonal basis $S$, we can derive
\begin{align*}
  S^\top A_\pi^{\rm CQL}S
  &=
  \begin{bmatrix}
    q^\top A_\pi^{\rm CQL}q & q^\top A_\pi^{\rm CQL}U\\
    U^\top A_\pi^{\rm CQL}q & U^\top A_\pi^{\rm CQL}U
  \end{bmatrix} \\
  &=
  \begin{bmatrix}
    1-\alpha\delta(1-\gamma) & q^\top A_\pi^{\rm CQL}U\\
    0 & \bar A_\pi^{\rm CQL}
  \end{bmatrix},
\end{align*}
where we use the fact that $q^\top A_\pi^{\rm CQL}q=1-\alpha\delta(1-\gamma)$, $U^\top A_\pi^{\rm CQL}q
  =\bigl(1-\alpha\delta(1-\gamma)\bigr)U^\top q=0$, and $U^\top A_\pi^{\rm CQL}U=\bar A_\pi^{\rm CQL}$ by
\Cref{eq:cpd-projected-standard-family}. Thus all matrices in
$\mathcal A^{\rm CQL}$ are simultaneously similar to block upper triangular
matrices with scalar diagonal block $1-\alpha\delta(1-\gamma)$ and projected
block $\bar A_\pi^{\rm CQL}$. By the similarity invariance in
\Cref{lem:similarity-jsr}, we may compute the JSR in this common basis. Applying
\Cref{lem:block-triangular-jsr} gives
\Cref{eq:cpd-standard-jsr-decomposition-general}.  By
\Cref{ass:cpd-constant-mode-range}, $1-\alpha\delta(1-\gamma)>0$, so the
absolute value can be removed and
\Cref{eq:cpd-standard-jsr-decomposition} follows.
\end{proof}

The decomposition immediately implies that the projected SLS family $\bar{\mathcal A}^{\rm CQL}$ cannot have a
larger JSR than the full SLS family ${\mathcal A}^{\rm CQL}$.
\begin{corollary}
\label{lem:cpd-weak-inequality}
The following inequality holds:
\begin{equation*}
  \rho(\bar{\mathcal A}^{\rm CQL})\leq \rho({\mathcal A}^{\rm CQL}).
\end{equation*}
\end{corollary}
\begin{proof}
The claim follows directly from~\Cref{lem:cpd-standard-jsr-decomposition}.
\end{proof}
The same decomposition also gives the corresponding lower bound for the
constant scalar mode.
\begin{corollary}
\label{cor:cpd-constant-mode-lower-bound}
Under~\Cref{ass:cpd-constant-mode-range}, we have
\begin{equation*}
  1-\alpha\delta(1-\gamma)\leq \rho({\mathcal A}^{\rm CQL}).
\end{equation*}
\end{corollary}
\begin{proof}
The claim follows directly from~\Cref{lem:cpd-standard-jsr-decomposition}.
\end{proof}

This paper focuses on the case in which the projected SLS family
$\bar{\mathcal A}^{\rm CQL}$ is strictly faster than the full SLS family
${\mathcal A}^{\rm CQL}$. In that regime, the decomposition in
\Cref{eq:cpd-standard-jsr-decomposition} shows that the standard rate is governed
by the common all-ones eigenvalue. The proposed modifications can then improve
convergence along the corresponding invariant direction without being blocked by
the projected SLS dynamics represented by $\bar{\mathcal A}^{\rm CQL}$.
\begin{assumption}
\label{ass:cpd-projected-gap}
Throughout the paper, we assume that the following strict inequality holds:
\begin{equation*}
  \rho(\bar{\mathcal A}^{\rm CQL})<\rho({\mathcal A}^{\rm CQL}).
\end{equation*}
\end{assumption}

Under this assumption, together with~\Cref{ass:cpd-constant-mode-range}, the decomposition in~\Cref{eq:cpd-standard-jsr-decomposition} shows that $\rho(\mathcal A^{\rm CQL})$ is exactly the eigenvalue of the common all-ones eigenvector.
\begin{lemma}
\label{lem:cpd-standard-benchmark}
Under~\Cref{ass:cpd-constant-mode-range,ass:cpd-projected-gap},
\begin{equation}
  \rho(\mathcal A^{\rm CQL})=1-\alpha\delta(1-\gamma).
  \label{eq:cpd-standard-benchmark}
\end{equation}
\end{lemma}
\begin{proof}
The claim follows directly from~\Cref{lem:cpd-standard-jsr-decomposition}.
\end{proof}
Combining~\Cref{ass:cpd-projected-gap,lem:cpd-standard-benchmark} with
\Cref{ass:cpd-constant-mode-range}, we obtain the equivalent form
\begin{equation*}
  \rho(\bar{\mathcal A}^{\rm CQL})<1-\alpha\delta(1-\gamma)<1.
\end{equation*}

The projected modes given in~\Cref{eq:cpd-projected-standard-family} have a useful structural feature: although the tuning
parameter $\delta$ changes the scalar block associated with the common
all-ones direction and may affect the upper-right coupling in the block
triangular representation, it does not affect the projected diagonal block. This
result is formally summarized in the following.
\begin{lemma}
\label{lem:cpd-projected-mode-delta-free}
For every deterministic policy $\pi$, the projected standard mode satisfies
\begin{equation}
  \bar A_\pi^{\rm CQL}
  =I_{n-1}
  +\alpha U^\top D(I-\Pi_d)(\gamma P\Pi^\pi-I)U.
  \label{eq:cpd-projected-mode-delta-free}
\end{equation}
Consequently, the matrices in $\bar{\mathcal A}^{\rm CQL}$ do not contain
$\delta$.
\end{lemma}
\begin{proof}
From~\Cref{eq:cpd-standard-mode,eq:constant-preserving-d-preconditioner},
\begin{align*}
  \bar A_\pi^{\rm CQL}
  &=U^\top\left\{I+
    \alpha\left[\delta\Pi_d+D(I-\Pi_d)\right](\gamma P\Pi^\pi-I)
    \right\}U \\
  &=I_{n-1}
    +\alpha\delta U^\top\Pi_d(\gamma P\Pi^\pi-I)U
    +\alpha U^\top D(I-\Pi_d)(\gamma P\Pi^\pi-I)U.
\end{align*}
Since
\begin{equation*}
  U^\top\Pi_d=U^\top\mathbf{1}d^\top=0,
\end{equation*}
the middle term is zero.  This proves~\Cref{eq:cpd-projected-mode-delta-free}.
The right-hand side of~\Cref{eq:cpd-projected-mode-delta-free} contains the MDP
term $P\Pi^\pi$, the step-size $\alpha$, the diagonal weighting $D$, and
$\Pi_d=\mathbf{1}d^\top$, but it does not contain $\delta$.
\end{proof}
Intuitively, changing $\delta$ moves the all-ones common eigenvalue in~\Cref{eq:cpd-common-vector} and may change the upper-right coupling from projected coordinates into the constant coordinate, while leaving the projected SLS family in~\Cref{eq:cpd-projected-standard-family} unchanged by~\Cref{eq:cpd-projected-mode-delta-free}. This separation is useful when choosing $\delta$: in the JSR decomposition, the dependence on $\delta$ enters through the constant-direction scalar block rather than through the projected JSR term in~\Cref{eq:cpd-standard-jsr-decomposition}.

\section{Heavy-ball corrected Q-learning}
\label{sec:cpd_tabular_heavy_ball}

Before introducing the corrected heavy-ball recursion, it is useful to contrast
it with the deterministic heavy-ball Q-value iteration studied
in~\citep{lee2026spectralhbqvi}. That work considers the unweighted heavy-ball Q-value iteration update
\begin{equation}
  Q_{k+1}
  =Q_k+\alpha\{F(Q_k)-Q_k\}+\eta(Q_k-Q_{k-1}).
  \label{eq:unweighted-hb-qvi-update}
\end{equation}
Its acceleration mechanism is to improve the convergence rate along the common
eigenvector direction. In the present paper the momentum coefficient is scaled by the step-size, so the actual coefficient is $\alpha\eta$. A stochastic heavy-ball Q-learning analogue with this scaling, however, has the conditional mean dynamics
\begin{equation*}
  Q_{k+1}=Q_k+\alpha D\{F(Q_k)-Q_k\}+\alpha\eta(Q_k-Q_{k-1}),
\end{equation*}
in which the diagonal sampling matrix $D$ multiplies the Bellman residual. As noted after~\Cref{alg:direct-d-standard-update},
this multiplication generally destroys the common eigenvector property unless
$D$ is uniform. The corrected matrix $H$ in~\Cref{eq:constant-preserving-d-preconditioner}
is introduced precisely to retain nonuniform sampling weights while restoring the
common eigenvector direction needed by the JSR analysis below.

This section studies the effect of adding a heavy-ball momentum term to the CQL update in~\Cref{alg:cpd-standard-update}. The resulting update is given in~\Cref{alg:cpd-hb-update} and is called heavy-ball corrected Q-learning (HBCQL). Throughout the HBCQL analysis, $\eta$ denotes a momentum gain and the actual momentum coefficient is $\alpha\eta$.
\begin{algorithm}[H]
\caption{Deterministic HBCQL}
\label{alg:cpd-hb-update}
\begin{algorithmic}
\Require $Q_{-1},Q_0\in\R^n$, step-size $\alpha>0$, momentum gain $\eta\geq0$, and $H$ from~\Cref{eq:constant-preserving-d-preconditioner}
\For{$k=0,1,2,\ldots$}
  \State Update
  \[
    Q_{k+1}
    \gets Q_k+\alpha H\{F(Q_k)-Q_k\}+\alpha\eta(Q_k-Q_{k-1}).
  \]
\EndFor
\end{algorithmic}
\end{algorithm}

To proceed, define the augmented map
\begin{equation*}
  g(Q,Q^-)
  :=
  \begin{bmatrix}
    Q+\alpha H\{F(Q)-Q\}+\alpha\eta(Q-Q^-)\\
    Q
  \end{bmatrix}.
\end{equation*}
The augmented map has the same fixed point in both blocks as the corrected standard map in~\Cref{eq:cpd-standard-map}.
\begin{lemma}
\label{lem:cpd-hb-fixed-point}
A pair $(\bar Q,\bar Q^-)$ is a fixed point of $g$, in the sense that
\begin{equation*}
  g(\bar Q,\bar Q^-)
  =\begin{bmatrix}\bar Q\\\bar Q^-\end{bmatrix},
\end{equation*}
if and only if $\bar Q=\bar Q^-=Q^\star$.
\end{lemma}
\begin{proof}
If $(\bar Q,\bar Q^-)$ is a fixed point of $g$, then the second block gives
$\bar Q=\bar Q^-$.  Substituting this identity into the first block gives
\begin{equation*}
  H\{F(\bar Q)-\bar Q\}=0.
\end{equation*}
By~\Cref{lem:cpd-H-nonsingular}, $H$ is nonsingular, so
$F(\bar Q)=\bar Q$ and hence $\bar Q=Q^\star$.  Therefore
$\bar Q^-=Q^\star$ as well.  Conversely, $F(Q^\star)=Q^\star$ immediately gives
$g(Q^\star,Q^\star)=\begin{bmatrix}Q^\star\\Q^\star\end{bmatrix}$.
\end{proof}
Subtracting $Q^\star=F(Q^\star)$ from the HBCQL update in~\Cref{alg:cpd-hb-update} gives the error recursion
\begin{align}
  Q_{k+1}-Q^\star
  &=(Q_k-Q^\star)
    +\alpha H\{F(Q_k)-F(Q^\star)-(Q_k-Q^\star)\} +\alpha\eta\{(Q_k-Q^\star)-(Q_{k-1}-Q^\star)\}.\label{eq:2}
\end{align}
The next lemma shows that the error recursion in~\Cref{eq:2} can be expressed as an SLS.
\begin{lemma}
\label{lem:cpd-hb-sls}
The error recursion in~\Cref{eq:2} satisfies
\begin{equation}
  \begin{bmatrix}Q_{k+1}-Q^\star\\Q_k-Q^\star\end{bmatrix}
  =A_{\mu_k}^{\rm HBCQL}
  \begin{bmatrix}Q_k-Q^\star\\Q_{k-1}-Q^\star\end{bmatrix},
  \label{eq:cpd-hb-error-system}
\end{equation}
where for a stochastic policy $\mu$,
\begin{equation*}
  A_\mu^{\rm HBCQL}:=
  \begin{bmatrix}
    A_\mu^{\rm CQL}+\alpha\eta I&-\alpha\eta I\\
    I&0
  \end{bmatrix}.
\end{equation*}
\end{lemma}
\begin{proof}
Using~\Cref{eq:bellman-difference-selector-general},
\begin{align*}
  Q_{k+1}-Q^\star
  &=\left[I+\alpha H(\gamma P\Pi^{\mu_k}-I)+\alpha\eta I\right]
    (Q_k-Q^\star)-\alpha\eta(Q_{k-1}-Q^\star)\\
  &=(A_{\mu_k}^{\rm CQL}+\alpha\eta I)(Q_k-Q^\star)
    -\alpha\eta(Q_{k-1}-Q^\star).
\end{align*}
Stacking this identity with the identity
$Q_k-Q^\star=Q_k-Q^\star$ gives~\Cref{eq:cpd-hb-error-system}. Since the
upper-left block depends affinely on $P\Pi^{\mu_k}$ and the convex-hull statement in
\Cref{lem:bellman_difference_selector} holds,
\begin{equation*}
  A_{\mu_k}^{\rm HBCQL}
  \in\co\{A_\pi^{\rm HBCQL}:\pi\in\Theta\}
  =\co(\mathcal A^{\rm HBCQL}).
\end{equation*}
\end{proof}
The corresponding SLS family is defined and denoted by
\begin{equation}
  \mathcal A^{\rm HBCQL}
  :=\left\{A_\pi^{\rm HBCQL}:\pi\in\Theta\right\}.
  \label{eq:cpd-hb-family}
\end{equation}
Moreover, for every stochastic selector $\mu_k$, one has
$A_{\mu_k}^{\rm HBCQL}\in\co(\mathcal A^{\rm HBCQL})$.
For the augmented heavy-ball recursion, the one-dimensional common eigenvector
argument on $\operatorname{span}\{\mathbf 1\}$ becomes an argument with
respect to a two-dimensional invariant subspace. To show this, define the
subspace
\begin{equation}
  \mathcal I
  :=\left\{
  \begin{bmatrix}a\mathbf 1\\ b\mathbf 1\end{bmatrix}:a,b\in\R
  \right\}.
  \label{eq:cpd-hb-constant-subspace}
\end{equation}
If the augmented error belongs to this subspace, then the next augmented
error also belongs to it, and the scalar coefficients are updated by the
following $2\times2$ matrix:
\begin{equation*}
  C:=
  \begin{bmatrix}
    \bigl(1-\alpha\delta(1-\gamma)\bigr)+\alpha\eta&-\alpha\eta\\
    1&0
  \end{bmatrix}.
\end{equation*}
This invariance is summarized in the next lemma.
\begin{lemma}
\label{lem:cpd-hb-constant-subspace}
For every stochastic policy $\mu$, the subspace
$\mathcal I$ in~\Cref{eq:cpd-hb-constant-subspace} is invariant
under $A_\mu^{\rm HBCQL}$. More precisely, for every $a,b\in\R$,
\begin{equation}
  A_\mu^{\rm HBCQL}
  \begin{bmatrix}a\mathbf 1\\ b\mathbf 1\end{bmatrix}
  =
  \begin{bmatrix}a_+\mathbf 1\\ b_+\mathbf 1\end{bmatrix},
  \qquad
  \begin{bmatrix}a_+\\ b_+\end{bmatrix}
  =C\begin{bmatrix}a\\ b\end{bmatrix}.
  \label{eq:cpd-hb-constant-subspace-action}
\end{equation}
\end{lemma}
\begin{proof}
Using~\Cref{eq:cpd-common-vector}, we have $A_\mu^{\rm CQL}\mathbf 1 =\bigl(1-\alpha\delta(1-\gamma)\bigr)\mathbf 1$. Hence, we obtain
\begin{align*}
  A_\mu^{\rm HBCQL}
  \begin{bmatrix}a\mathbf 1\\ b\mathbf 1\end{bmatrix}
  &=
  \begin{bmatrix}
    (A_\mu^{\rm CQL}+\alpha\eta I)a\mathbf 1-\alpha\eta b\mathbf 1\\
    a\mathbf 1
  \end{bmatrix}\\
  &=
  \begin{bmatrix}
    \left(\bigl(1-\alpha\delta(1-\gamma)\bigr)+\alpha\eta\right)a\mathbf 1
    -\alpha\eta b\mathbf 1\\
    a\mathbf 1
  \end{bmatrix}.
\end{align*}
The scalar coordinates therefore satisfy
\begin{equation*}
  \begin{bmatrix}a_+\\ b_+\end{bmatrix}
  =
  \begin{bmatrix}
    \bigl(1-\alpha\delta(1-\gamma)\bigr)+\alpha\eta&-\alpha\eta\\
    1&0
  \end{bmatrix}
  \begin{bmatrix}a\\ b\end{bmatrix}
  =C\begin{bmatrix}a\\ b\end{bmatrix},
\end{equation*}
which proves both invariance and~\Cref{eq:cpd-hb-constant-subspace-action}.
\end{proof}

Thus, the common eigenvector direction $\operatorname{span}\{\mathbf 1\}$ of CQL is lifted, in HBCQL, to the invariant subspace $\mathcal I$ of the augmented state. On $\operatorname{span}\{\mathbf 1\}$, CQL follows the first-order scalar dynamics
\[
  x_{k+1}=\bigl(1-\alpha\delta(1-\gamma)\bigr)x_k,\quad x_k\in \operatorname{span}\{\mathbf 1\},
\]
whereas HBCQL follows the second-order dynamics on $\mathcal I$ governed by the matrix $C$ in~\Cref{eq:cpd-hb-constant-subspace-action} in the scalar coordinates of the augmented vector.  Namely, if
\[
  x_k=\begin{bmatrix}a_k\mathbf 1\\ b_k\mathbf 1\end{bmatrix}\in\mathcal I,
\]
then
\[
  \begin{bmatrix}a_{k+1}\\ b_{k+1}\end{bmatrix}
  =C\begin{bmatrix}a_k\\ b_k\end{bmatrix}.
\]
The preceding lemma therefore implies that the augmented SLS family of HBCQL cannot have a JSR
smaller than the spectral radius of $C$. To explain this, define, for any
stochastic policy $\mu$, the following projected mode and its associated SLS family:
\begin{equation}
  \bar A_\mu^{\rm HBCQL}:=
  \begin{bmatrix}
    U^\top A_\mu^{\rm CQL}U+\alpha\eta I_{n-1}&-\alpha\eta I_{n-1}\\
    I_{n-1}&0
  \end{bmatrix},
  \qquad
  \bar{\mathcal A}^{\rm HBCQL}
  :=\left\{\bar A_\pi^{\rm HBCQL}:\pi\in\Theta\right\}.
  \label{eq:cpd-projected-hb-family}
\end{equation}
The next lemma shows that the projected augmented error associated with~\Cref{eq:cpd-hb-error-system} can be represented by an SLS.
\begin{lemma}
\label{lem:cpd-hb-projected-error-sls}
Let $Q_k$ follow~\Cref{alg:cpd-hb-update}. Then
\begin{equation}
  \begin{bmatrix}U^\top(Q_{k+1}-Q^\star)\\U^\top(Q_k-Q^\star)
  \end{bmatrix}
  =\bar A_{\mu_k}^{\rm HBCQL}
  \begin{bmatrix}U^\top(Q_k-Q^\star)\\U^\top(Q_{k-1}-Q^\star)
  \end{bmatrix},
  \label{eq:cpd-hb-projected-error-system}
\end{equation}
and the mode matrix in \Cref{eq:cpd-hb-projected-error-system} belongs to
$\co(\bar{\mathcal A}^{\rm HBCQL})$.
\end{lemma}
\begin{proof}
By~\Cref{eq:cpd-hb-error-system},
\begin{equation*}
  Q_{k+1}-Q^\star
  =(A_{\mu_k}^{\rm CQL}+\alpha\eta I)(Q_k-Q^\star)
  -\alpha\eta(Q_{k-1}-Q^\star).
\end{equation*}
As in the proof of~\Cref{lem:cpd-projected-error-sls},
$U^\top A_{\mu_k}^{\rm CQL}q=0$ for $q=n^{-1/2}\mathbf 1$.  Therefore
\begin{align*}
  U^\top(Q_{k+1}-Q^\star)
  &=(U^\top A_{\mu_k}^{\rm CQL}U+\alpha\eta I_{n-1})U^\top(Q_k-Q^\star)
    -\alpha\eta U^\top(Q_{k-1}-Q^\star).
\end{align*}
Stacking this identity with the identity
$U^\top(Q_k-Q^\star)=U^\top(Q_k-Q^\star)$ gives
\Cref{eq:cpd-hb-projected-error-system}.  The affine convex-hull relation for
$U^\top A_{\mu_k}^{\rm CQL}U$ and \Cref{eq:cpd-projected-hb-family} imply that
the projected heavy-ball mode belongs to
$\co(\bar{\mathcal A}^{\rm HBCQL})$.
\end{proof}

As in the CQL analysis, we can now decompose the augmented SLS family in~\Cref{eq:cpd-hb-family} into the constant-error subspace in~\Cref{eq:cpd-hb-constant-subspace} and the projected dynamics in~\Cref{eq:cpd-projected-hb-family}.
\begin{lemma}
\label{lem:cpd-hb-jsr-decomposition}
The HBCQL SLS family in~\Cref{eq:cpd-hb-family} satisfies
\begin{equation}
  \rho(\mathcal A^{\rm HBCQL})
  =\max\left\{\rho(C),\rho(\bar{\mathcal A}^{\rm HBCQL})\right\}.
  \label{eq:cpd-hb-jsr-decomposition}
\end{equation}
\end{lemma}
\begin{proof}
Let
\begin{equation*}
  q:=n^{-1/2}\mathbf 1,
  \qquad
  S:=\begin{bmatrix}q&U\end{bmatrix}.
\end{equation*}
The matrix $S$ is orthogonal.  By the proof of
\Cref{lem:cpd-standard-jsr-decomposition}, each deterministic standard mode has
\begin{equation*}
  S^\top A_\pi^{\rm CQL}S
  =
  \begin{bmatrix}
    1-\alpha\delta(1-\gamma) & q^\top A_\pi^{\rm CQL}U\\
    0 & \bar A_\pi^{\rm CQL}
  \end{bmatrix}.
\end{equation*}
Conjugating the current and previous coordinates by the same orthogonal basis gives
\begin{align*}
  \begin{bmatrix}S^\top&0\\0&S^\top\end{bmatrix}
  A_\pi^{\rm HBCQL}
  \begin{bmatrix}S&0\\0&S\end{bmatrix}
  &=
  \begin{bmatrix}
    S^\top(A_\pi^{\rm CQL}+\alpha\eta I)S & -\alpha\eta I\\
    I & 0
  \end{bmatrix} \\
  &=
  \begin{bmatrix}
    1-\alpha\delta(1-\gamma)+\alpha\eta & q^\top A_\pi^{\rm CQL}U
      & -\alpha\eta & 0\\
    0 & \bar A_\pi^{\rm CQL}+\alpha\eta I_{n-1}
      & 0 & -\alpha\eta I_{n-1}\\
    1 & 0 & 0 & 0\\
    0 & I_{n-1} & 0 & 0
  \end{bmatrix},
\end{align*}
where the block coordinates are ordered as current constant, current projected,
previous constant, and previous projected. Permuting the second and third block
coordinates gives the similar matrix
\begin{equation*}
  \begin{bmatrix}
    1-\alpha\delta(1-\gamma)+\alpha\eta & -\alpha\eta
      & q^\top A_\pi^{\rm CQL}U & 0\\
    1 & 0 & 0 & 0\\
    0 & 0 & \bar A_\pi^{\rm CQL}+\alpha\eta I_{n-1} & -\alpha\eta I_{n-1}\\
    0 & 0 & I_{n-1} & 0
  \end{bmatrix}.
\end{equation*}
The upper-left block in this matrix is
\begin{equation*}
  \begin{bmatrix}
    1-\alpha\delta(1-\gamma)+\alpha\eta & -\alpha\eta\\
    1 & 0
  \end{bmatrix}
  =C,
\end{equation*}
and the lower-right block is
\begin{equation*}
  \begin{bmatrix}
    \bar A_\pi^{\rm CQL}+\alpha\eta I_{n-1} & -\alpha\eta I_{n-1}\\
    I_{n-1} & 0
  \end{bmatrix}
  =\bar A_\pi^{\rm HBCQL}.
\end{equation*}
The only nonzero upper-right coupling, from the projected second-order block into the
constant second-order block, is the row $q^\top A_\pi^{\rm CQL}U$ in the current
coordinate. Thus, equivalently,
\begin{equation*}
  A_\pi^{\rm HBCQL}
  \sim
  \begin{bmatrix}
    C & E_\pi\\
    0 & \bar A_\pi^{\rm HBCQL}
  \end{bmatrix},
  \qquad
  E_\pi:=
  \begin{bmatrix}
    q^\top A_\pi^{\rm CQL}U & 0\\
    0 & 0
  \end{bmatrix}.
\end{equation*}
The same similarity and permutation are used for every $\pi\in\Theta$, so the
whole family is simultaneously block upper triangular with diagonal families
$\{C\}$ and $\bar{\mathcal A}^{\rm HBCQL}$. By the similarity invariance in
\Cref{lem:similarity-jsr}, we may compute the JSR in this common basis. Applying
\Cref{lem:block-triangular-jsr} gives
\begin{equation*}
  \rho(\mathcal A^{\rm HBCQL})
  =\max\{\rho(C),\rho(\bar{\mathcal A}^{\rm HBCQL})\}.
\end{equation*}
\end{proof}
The decomposition in~\Cref{eq:cpd-hb-jsr-decomposition} gives the following two lower bounds.
\begin{corollary}
\label{cor:cpd-hb-common-lower-bound}
The following inequality holds:
\begin{equation*}
  \rho(C)\leq \rho(\mathcal A^{\rm HBCQL}).
\end{equation*}
\end{corollary}
\begin{proof}
The claim follows from~\Cref{eq:cpd-hb-jsr-decomposition} because the maximum
is at least its first entry.
\end{proof}

\begin{corollary}
\label{cor:cpd-hb-projected-lower-bound}
The following inequality holds:
\begin{equation*}
  \rho(\bar{\mathcal A}^{\rm HBCQL})\leq \rho(\mathcal A^{\rm HBCQL}).
\end{equation*}
\end{corollary}
\begin{proof}
The claim follows from~\Cref{eq:cpd-hb-jsr-decomposition} because the maximum
is at least its second entry.
\end{proof}
A natural question is whether the HBCQL recursion can improve the convergence rate. The common eigenvector direction can be accelerated by a suitable positive momentum. The next theorem gives a sufficient condition under which HBCQL has a strictly smaller JSR than CQL.
\begin{theorem}
\label{thm:cpd-global-jsr-improvement}
Under~\Cref{ass:cpd-constant-mode-range,ass:cpd-projected-gap}, there exists $\eta_0>0$ such that every momentum gain satisfying
\begin{equation}
  0<\alpha\eta<\min\{\bigl(1-\alpha\delta(1-\gamma)\bigr)^2,\eta_0\}
  \label{eq:cpd-eta-condition}
\end{equation}
satisfies
\begin{equation}
  \rho(\mathcal A^{\rm HBCQL})
  <\rho(\mathcal A^{\rm CQL})=\bigl(1-\alpha\delta(1-\gamma)\bigr)<1.
  \label{eq:cpd-global-improvement}
\end{equation}
\end{theorem}
\begin{proof}
By~\Cref{ass:cpd-constant-mode-range},
\begin{equation*}
  0<1-\alpha\delta(1-\gamma)<1.
\end{equation*}
From~\Cref{eq:cpd-standard-jsr-decomposition} and
\Cref{ass:cpd-projected-gap}, the projected block cannot attain the maximum;
hence
\begin{equation*}
  \rho(\mathcal A^{\rm CQL})=1-\alpha\delta(1-\gamma),
  \qquad
  \rho(\bar{\mathcal A}^{\rm CQL})<1-\alpha\delta(1-\gamma).
\end{equation*}
At $\eta=0$, \Cref{eq:cpd-projected-hb-family} gives, for each
$\pi\in\Theta$,
\begin{equation*}
  \left.\bar A_\pi^{\rm HBCQL}\right|_{\eta=0}
  =
  \begin{bmatrix}
    \bar A_\pi^{\rm CQL}&0\\
    I_{n-1}&0
  \end{bmatrix}.
\end{equation*}
For every word $\pi_1,\ldots,\pi_k$ with $k\ge1$, the corresponding product is
\begin{equation*}
  \left.\bar A_{\pi_k}^{\rm HBCQL}\right|_{\eta=0}
  \cdots
  \left.\bar A_{\pi_1}^{\rm HBCQL}\right|_{\eta=0}
  =
  \begin{bmatrix}
    \bar A_{\pi_k}^{\rm CQL}\cdots\bar A_{\pi_1}^{\rm CQL}&0\\
    \bar A_{\pi_{k-1}}^{\rm CQL}\cdots\bar A_{\pi_1}^{\rm CQL}&0
  \end{bmatrix},
\end{equation*}
where the lower-left block is interpreted as $I_{n-1}$ when $k=1$. Equivalently,
the zero-momentum projected family is block lower triangular with diagonal
families $\bar{\mathcal A}^{\rm CQL}$ and $\{0\}$.  Hence
\Cref{lem:block-triangular-jsr} gives
\begin{equation*}
  \rho\left(\left.\bar{\mathcal A}^{\rm HBCQL}\right|_{\eta=0}\right)
  =\max\{\rho(\bar{\mathcal A}^{\rm CQL}),0\}
  =\rho(\bar{\mathcal A}^{\rm CQL}).
\end{equation*}
By continuity of the JSR for finite matrix families, there exists $\eta_0>0$ such
that
\begin{equation*}
  \rho(\bar{\mathcal A}^{\rm HBCQL})<1-\alpha\delta(1-\gamma),
  \qquad 0<\alpha\eta<\eta_0.
\end{equation*}
If, in addition,
\begin{equation*}
  0<\alpha\eta<\bigl(1-\alpha\delta(1-\gamma)\bigr)^2,
\end{equation*}
then~\Cref{app:prop-cpd-constant-mode-acceleration} gives
\begin{equation*}
  \rho(C)<1-\alpha\delta(1-\gamma).
\end{equation*}
Using~\Cref{eq:cpd-hb-jsr-decomposition},
\begin{align*}
  \rho(\mathcal A^{\rm HBCQL})
  &=\max\{\rho(C),\rho(\bar{\mathcal A}^{\rm HBCQL})\}\\
  &<1-\alpha\delta(1-\gamma)
   =\rho(\mathcal A^{\rm CQL}).
\end{align*}
This proves~\Cref{eq:cpd-global-improvement}.  For admissible stochastic
Bellman-difference selectors, the active modes lie in the convex hull of the
finite deterministic-policy family by the convex-hull statement in
\Cref{lem:bellman_difference_selector}. Since this family is finite,
\Cref{lem:convex-hull-jsr} applies, so passing to the convex hull does not change
the JSR.  Since the resulting JSR is less than one, the uniform exponential
stability statement follows from~\Cref{lem:basic_jsr_convergence}.
\end{proof}
\Cref{thm:cpd-global-jsr-improvement} implies that, under its hypotheses,
HBCQL converges to $Q^\star$ and has a strictly smaller JSR certificate than the
CQL recursion in~\Cref{alg:cpd-standard-update}. This is a fixed-step-size JSR acceleration statement: the actual momentum coefficient is $\alpha\eta$. Consequently, when $\eta$ is kept fixed and $\alpha\to0$, the common-mode improvement over CQL is second order in $\alpha$ rather than a leading-order heavy-ball acceleration.
Note that~\Cref{thm:cpd-global-jsr-improvement} assumes the strict inequality in~\Cref{ass:cpd-projected-gap}. Because this condition may be hard to verify
directly, the following corollary gives a simple sufficient interval for $\delta$ that automatically guarantees~\Cref{ass:cpd-projected-gap}.
\begin{corollary}
\label{lem:cpd-admissible-delta-eta}
Suppose that $\rho(\bar{\mathcal A}^{\rm CQL})<1$. For any
\begin{equation}
  0<\delta<\frac{1-\rho(\bar{\mathcal A}^{\rm CQL})}{\alpha(1-\gamma)},
  \label{eq:cpd-admissible-delta-interval}
\end{equation}
\Cref{ass:cpd-projected-gap} holds. After such a $\delta$ has been fixed, any
momentum gain satisfying~\Cref{eq:cpd-eta-condition} also satisfies the
hypotheses of~\Cref{thm:cpd-global-jsr-improvement}.
\end{corollary}
\begin{proof}
For fixed MDP data, $\alpha$, $D$, and $d$, the projected matrices satisfy
\begin{equation*}
  \bar A_\pi^{\rm CQL}
  =I_{n-1}+\alpha U^\top D(I-\Pi_d)(\gamma P\Pi^\pi-I)U,
\end{equation*}
by~\Cref{eq:cpd-projected-mode-delta-free}.  Hence
$\rho(\bar{\mathcal A}^{\rm CQL})$ is independent of $\delta$.  The strict
projected-gap condition is therefore
\begin{equation*}
  \rho(\bar{\mathcal A}^{\rm CQL})<1-\alpha\delta(1-\gamma),
\end{equation*}
together with $0<1-\alpha\delta(1-\gamma)<1$.  The interval in
\Cref{eq:cpd-admissible-delta-interval} gives
\begin{equation*}
  0<\alpha\delta(1-\gamma)<1-\rho(\bar{\mathcal A}^{\rm CQL})\le 1,
\end{equation*}
and therefore
\begin{equation*}
  \rho(\bar{\mathcal A}^{\rm CQL})<1-\alpha\delta(1-\gamma)<1.
\end{equation*}
Thus both the constant-mode range and projected-gap requirements hold.  Once
$\delta$ is fixed,~\Cref{eq:cpd-eta-condition} is precisely the remaining
small-actual-momentum condition in~\Cref{thm:cpd-global-jsr-improvement}.
\end{proof}

\section{Hyperparameter conditions for improving over standard Q-learning}
The previous theorem compares HBCQL with CQL ($\eta =0$, equivalently $\alpha\eta=0$).
We also state a useful sufficient condition under which HBCQL is faster than the standard QL benchmark in~\Cref{alg:direct-d-standard-update}, obtained by setting the momentum term to zero and not
applying the correction.
Since $D$ is nonsingular, the fixed point of
\Cref{alg:direct-d-standard-update} is $Q^\star$. Subtracting $Q^\star=F(Q^\star)$ from the update given in~\Cref{alg:direct-d-standard-update} gives the
error recursion
\begin{align}
  Q_{k+1}-Q^\star
  &=(Q_k-Q^\star)
    +\alpha D\{F(Q_k)-F(Q^\star)-(Q_k-Q^\star)\}.\label{eq:3}
\end{align}
The next lemma uses~\Cref{eq:bellman-difference-selector-general} to derive an SLS model associated with the error recursion in~\Cref{eq:3}.
\begin{lemma}
\label{lem:direct-d-standard-sls}
If $Q_k$ follows
\Cref{alg:direct-d-standard-update}, then the error recursion in~\Cref{eq:3} satisfies the SLS
\begin{equation}
  Q_{k+1}-Q^\star=A_{\mu_k}^{\rm QL}(Q_k-Q^\star),
  \label{eq:direct-d-standard-error-system}
\end{equation}
where for any stochastic policy $\mu$,
\begin{equation}
  A_\mu^{\rm QL}:=I+\alpha D(\gamma P\Pi^\mu-I).
  \label{eq:direct-d-standard-mode}
\end{equation}
\end{lemma}
\begin{proof}
Using~\Cref{eq:bellman-difference-selector-general}, the error recursion in~\Cref{eq:3} can be written as
\begin{align*}
  Q_{k+1}-Q^\star
  &=\{I+\alpha D(\gamma P\Pi^{\mu_k}-I)\}(Q_k-Q^\star)\\
  &=A_{\mu_k}^{\rm QL}(Q_k-Q^\star),
\end{align*}
which proves~\Cref{eq:direct-d-standard-error-system}.  The affine dependence
on $P\Pi^{\mu_k}$ and the convex-hull statement in
\Cref{lem:bellman_difference_selector} give
$A_{\mu_k}^{\rm QL}\in\co(\mathcal A^{\rm QL})$.
\end{proof}
The corresponding SLS family is
\begin{equation}
  \mathcal A^{\rm QL}:=\left\{A_\pi^{\rm QL}:\pi\in\Theta\right\}.
  \label{eq:direct-d-standard-family}
\end{equation}
We now compare CQL with the standard QL benchmark. The correction in CQL is designed to preserve a common eigenvector, but this structural property alone does not imply an improvement over QL. The next result gives a sufficient condition under which the JSR of CQL is no larger than that of the QL benchmark.
\begin{lemma}
\label{lem:cpd-cql-no-larger-than-direct-d}
Assume
\begin{equation*}
  0<\alpha d(s,a)\le 1
  \quad\text{for every }(s,a)\in\mathcal S\times\mathcal A,
\end{equation*}
and
\begin{equation*}
  d_{\max}\le \delta
  <\frac{1-\rho(\bar{\mathcal A}^{\rm CQL})}{\alpha(1-\gamma)}.
\end{equation*}
Then
\begin{equation*}
  \rho(\mathcal A^{\rm CQL})\le \rho(\mathcal A^{\rm QL}).
\end{equation*}
The interval for $\delta$ is nonempty if and only if
\begin{equation*}
  \rho(\bar{\mathcal A}^{\rm CQL})<1-\alpha d_{\max}(1-\gamma).
\end{equation*}
\end{lemma}
\begin{proof}
The upper bound on $\delta$ gives
\begin{equation*}
  \rho(\bar{\mathcal A}^{\rm CQL})<1-\alpha\delta(1-\gamma).
\end{equation*}
Since JSRs are nonnegative, the same upper bound gives
$1-\alpha\delta(1-\gamma)>0$. Also, $d_{\max}\le\delta$ and full support of
$d$ imply $\delta>0$, so $1-\alpha\delta(1-\gamma)<1$. Hence
\Cref{ass:cpd-constant-mode-range} holds, and the decomposition in
\Cref{eq:cpd-standard-jsr-decomposition} yields
\begin{equation*}
  \rho(\mathcal A^{\rm CQL})=1-\alpha\delta(1-\gamma).
\end{equation*}
By~\Cref{lem:direct-d-nonnegative-row-sum}, every deterministic direct mode
satisfies
\begin{equation*}
  \rho(A_\pi^{\rm QL})\ge 1-\alpha d_{\max}(1-\gamma),
  \qquad \pi\in\Theta.
\end{equation*}
Since the JSR of a family is at least the spectral radius of each one-step mode,
\begin{equation*}
  \rho(\mathcal A^{\rm QL})\ge 1-\alpha d_{\max}(1-\gamma).
\end{equation*}
Using $d_{\max}\le\delta$, we obtain
\begin{equation*}
  \rho(\mathcal A^{\rm CQL})
  =1-\alpha\delta(1-\gamma)
  \le 1-\alpha d_{\max}(1-\gamma)
  \le \rho(\mathcal A^{\rm QL}),
\end{equation*}
which proves the claimed comparison. Finally,
\begin{equation*}
  d_{\max}<\frac{1-\rho(\bar{\mathcal A}^{\rm CQL})}{\alpha(1-\gamma)}
\end{equation*}
is equivalent to
\begin{equation*}
  \rho(\bar{\mathcal A}^{\rm CQL})<1-\alpha d_{\max}(1-\gamma),
\end{equation*}
and this proves the nonempty-interval statement.
\end{proof}

The remaining question is whether HBCQL can be faster than the QL benchmark
itself. Combining the previous comparison with the heavy-ball improvement
theorem~\Cref{thm:cpd-global-jsr-improvement} gives the following sufficient
condition for a strict JSR comparison of HBCQL against QL.
\begin{lemma}
\label{lem:cpd-hbcql-faster-than-direct-d}
Assume
\begin{equation*}
  0<\alpha d(s,a)\le 1
  \quad\text{for every }(s,a)\in\mathcal S\times\mathcal A,
\end{equation*}
and
\begin{equation*}
  d_{\max}\le \delta
  <\frac{1-\rho(\bar{\mathcal A}^{\rm CQL})}{\alpha(1-\gamma)}.
\end{equation*}
Then the hypotheses of~\Cref{thm:cpd-global-jsr-improvement} are satisfied, namely
\Cref{ass:cpd-constant-mode-range,ass:cpd-projected-gap} hold. Let
$\eta_0>0$ be the corresponding small-actual-momentum radius guaranteed by
\Cref{thm:cpd-global-jsr-improvement}. If
\begin{equation*}
  0<\alpha\eta<\min\{\bigl(1-\alpha\delta(1-\gamma)\bigr)^2,\eta_0\},
\end{equation*}
then
\begin{equation*}
  \rho(\mathcal A^{\rm HBCQL})<\rho(\mathcal A^{\rm QL}).
\end{equation*}
\end{lemma}
\begin{proof}
Because $0<\alpha d(s,a)$ for every $(s,a)$ and $d_{\max}\le\delta$, we have
$\delta>0$. The upper bound on $\delta$ gives
\begin{equation*}
  \rho(\bar{\mathcal A}^{\rm CQL})<1-\alpha\delta(1-\gamma).
\end{equation*}
Since the left-hand side is nonnegative, this also gives
$0<1-\alpha\delta(1-\gamma)<1$. Hence
\Cref{ass:cpd-constant-mode-range,ass:cpd-projected-gap} hold, and
\Cref{thm:cpd-global-jsr-improvement} applies. Therefore
\begin{equation*}
  \rho(\mathcal A^{\rm HBCQL})<\rho(\mathcal A^{\rm CQL})
\end{equation*}
for the stated range of the actual momentum coefficient $\alpha\eta$. Combining this strict inequality with
\Cref{lem:cpd-cql-no-larger-than-direct-d} gives
\begin{equation*}
  \rho(\mathcal A^{\rm HBCQL})<\rho(\mathcal A^{\rm QL}).
\end{equation*}
\end{proof}

The previous lemma packages two requirements into the interval for $\delta$: the
upper bound gives the hypotheses of~\Cref{thm:cpd-global-jsr-improvement}, while
$d_{\max}\le\delta$ gives the comparison between CQL and QL. The following
corollary separates these two roles. It assumes the projected gap directly and
states the conclusion in terms of the hyperparameters.
\begin{corollary}
\label{cor:cpd-hb-faster-than-direct-d}
Assume the following hyperparameter conditions hold:
\begin{enumerate}[label=(\roman*)]
  \item $0<\alpha d(s,a)\le 1$ for every $(s,a)\in\mathcal S\times\mathcal A$;
  \item $d_{\max}\le \delta<\frac{1}{\alpha(1-\gamma)}$;
  \item \Cref{ass:cpd-projected-gap} holds;
  \item $0<\alpha\eta<\min\{\rho(\mathcal A^{\rm CQL})^2,\eta_0\}$,
  where $\eta_0>0$ is the small-actual-momentum radius from
  \Cref{thm:cpd-global-jsr-improvement}.
\end{enumerate}
Then
\begin{equation*}
  \rho(\mathcal A^{\rm HBCQL})<\rho(\mathcal A^{\rm QL}).
\end{equation*}
In particular, with the boundary choice $\delta=d_{\max}$, condition (ii)
becomes
\begin{equation*}
  \alpha< \frac{1}{d_{\max}(1-\gamma)}.
\end{equation*}
Together with condition (i), it is sufficient to require
\begin{equation*}
  0<\alpha\le \frac{1}{d_{\max}}.
\end{equation*}
If~\Cref{ass:cpd-projected-gap} is also verified for this choice, then it is
sufficient to choose
\begin{equation*}
  0<\alpha\eta<\min\left\{\rho(\mathcal A^{\rm CQL})^2,\eta_0\right\}.
\end{equation*}
\end{corollary}
\begin{proof}
By~\Cref{lem:direct-d-nonnegative-row-sum}, every deterministic direct mode
satisfies
\begin{equation*}
  \rho(A_\pi^{\rm QL})\ge 1-\alpha d_{\max}(1-\gamma),
  \qquad \pi\in\Theta.
\end{equation*}
Therefore
\begin{equation*}
  \rho(\mathcal A^{\rm QL})\ge 1-\alpha d_{\max}(1-\gamma).
\end{equation*}
On the other hand, $\delta\ge d_{\max}$ implies
\begin{equation*}
  1-\alpha\delta(1-\gamma)
  \le 1-\alpha d_{\max}(1-\gamma).
\end{equation*}
Condition (ii) gives $0<1-\alpha\delta(1-\gamma)<1$, so
\Cref{ass:cpd-constant-mode-range} holds. Condition (iii) and
\Cref{lem:cpd-standard-benchmark} give
\begin{equation*}
  \rho(\mathcal A^{\rm CQL})=1-\alpha\delta(1-\gamma).
\end{equation*}
Then condition (iv), together with~\Cref{thm:cpd-global-jsr-improvement}, gives
\begin{equation*}
  \rho(\mathcal A^{\rm HBCQL})<\rho(\mathcal A^{\rm CQL}).
\end{equation*}
Combining these inequalities yields
\begin{equation*}
  \rho(\mathcal A^{\rm HBCQL})
  <\rho(\mathcal A^{\rm CQL})
  =1-\alpha\delta(1-\gamma)
  \le 1-\alpha d_{\max}(1-\gamma)
  \le \rho(\mathcal A^{\rm QL}),
\end{equation*}
which proves the claim. The simplified conditions follow by setting
$\delta=d_{\max}$.
\end{proof}

\section{Numerical illustration}
\label{sec:numerical_toy_example}

This section gives an example for~\Cref{alg:cpd-standard-update,alg:cpd-hb-update}. The purpose of this example is to show in a reproducible finite example
how the heavy-ball momentum term can accelerate convergence. We consider an MDP with $\mathcal S=\{1,2\}$, $\mathcal A=\{1,2\}$, and \(\gamma=0.9\). The
transition-to-next-state probabilities are
\begin{equation*}
  P_s=
  \begin{bmatrix}
    0.85&0.15\\
    0.25&0.75\\
    0.70&0.30\\
    0.10&0.90
  \end{bmatrix}.
\end{equation*}
The one-step reward entries are
$r(1,1,1)=r(1,1,2)=1.0$, $r(1,2,1)=r(1,2,2)=0.2$,
$r(2,1,1)=r(2,1,2)=0.1$, and $r(2,2,1)=r(2,2,2)=0.8$.
The rows of $P_s$ and the entries of all state-action vectors in this section follow the ordering $(1,1),(2,1),(1,2),(2,2)$. The nonuniform sampling vector is set to $d=(0.45,0.40,0.05,0.10)^\top$ and $D=\operatorname{diag}(d)$, and the shared numerical parameters are
\begin{equation*}
  \alpha=1.10,
  \qquad \delta=0.80,
  \qquad \alpha\eta=0.30.
\end{equation*}
For the simulation, we use the initial values \(Q_{-1}=Q_0=(3,1.5,-2.5,0)^\top\), and the optimal Q-function is obtained by value iteration and is $  Q^\star\approx
  (9.16923,7.93692,8.28615,8.55385)^\top$.
For comparative analysis, define the relative error
\begin{equation*}
  e_k^Q=\frac{\left\|Q_k-Q^\star\right\|_2}
  {\left\|Q_0-Q^\star\right\|_2}.
\end{equation*}
\Cref{fig:toy-tabular-convergence} compares \Cref{alg:cpd-standard-update} and
\Cref{alg:cpd-hb-update}. The figure shows that HBCQL reduces the relative error faster than CQL on this instance. This example is an empirical transient-acceleration illustration; it is not claimed to verify the projected-gap sufficient condition in~\Cref{ass:cpd-projected-gap}.
\begin{figure}[H]
\centering
\includegraphics[width=0.78\textwidth]{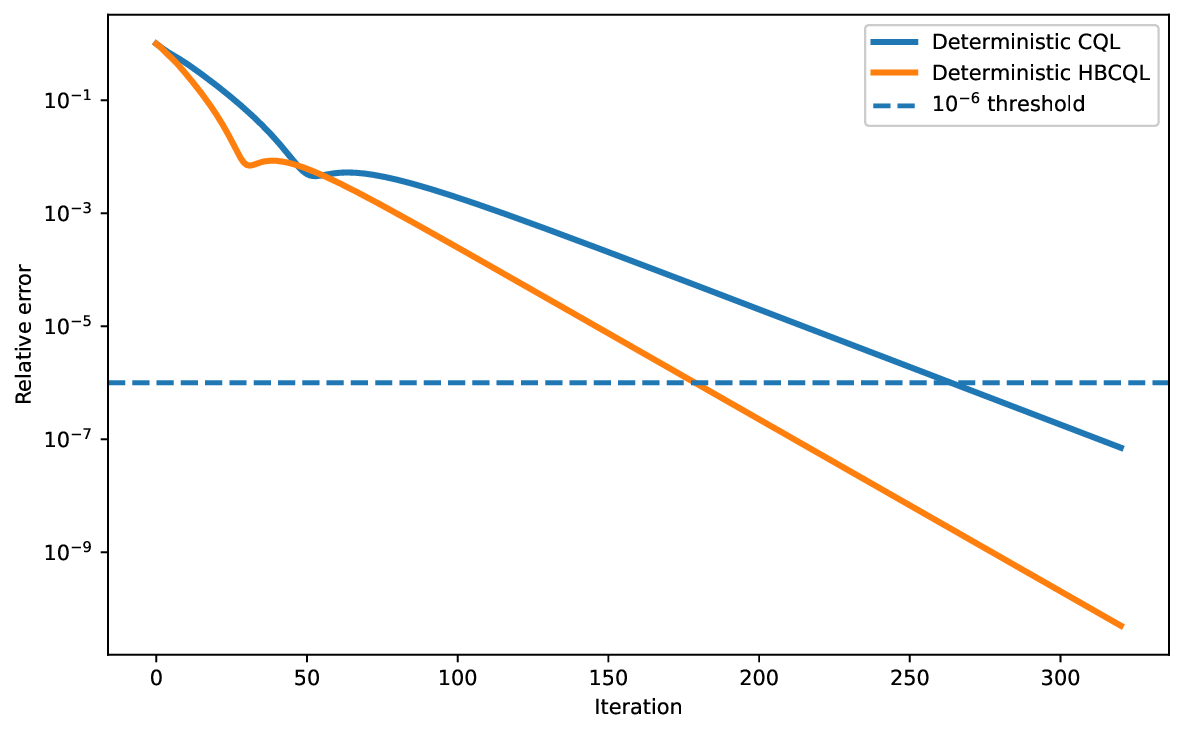}
\caption{The HBCQL in~\Cref{alg:cpd-hb-update} reaches the same relative \(Q\)-error
threshold in fewer iterations than the CQL in~\Cref{alg:cpd-standard-update}.}
\label{fig:toy-tabular-convergence}
\end{figure}

\section{Stochastic heavy-ball corrected Q-learning}
\label{sec:sampled_stochastic_q_learning}

This section studies a model-free stochastic counterpart of the HBCQL recursion in~\Cref{alg:cpd-hb-update}.
We first recall standard tabular Q-learning. Given $Q_k$, draw
$i_k=(s_k,a_k)\sim d$, draw $s_k^\prime\sim P(\cdot\mid s_k,a_k)$, and define
$\Delta_k:=r(s_k,a_k,s_k^\prime)+\gamma\max_{b\in\mathcal A}Q_k(s_k^\prime,b)-Q_k(s_k,a_k)$.
The standard sampled update changes the sampled coordinate only:
\begin{equation*}
  Q_{k+1}=Q_k+\alpha_k\Delta_ke_{i_k}.
\end{equation*}
Equivalently, it can be written as
\begin{equation*}
  Q_{k+1}(s,a)=
  \begin{cases}
  Q_k(s,a)+\alpha_k\Delta_k, & (s,a)=i_k,\\
  Q_k(s,a), & (s,a)\ne i_k .
  \end{cases}
\end{equation*}
The complete procedure is stated in~\Cref{alg:sampled-standard-q-learning}.
\begin{algorithm}[H]
\caption{Stochastic standard Q-learning}
\label{alg:sampled-standard-q-learning}
\begin{algorithmic}
\Require $Q_0\in\R^n$, step-sizes $\alpha_k>0$, and full-support sampling distribution $d$
\For{$k=0,1,2,\ldots$}
  \State Draw $i_k=(s_k,a_k)\sim d$
  \State Draw $s_k^\prime\sim P(\cdot\mid s_k,a_k)$ and compute $\Delta_k$
  \State Update
  \[
    Q_{k+1}\gets Q_k+\alpha_k\Delta_ke_{i_k}.
  \]
\EndFor
\end{algorithmic}
\end{algorithm}
Throughout this section, samples are drawn i.i.d. from the full-support distribution $d$. Markovian sampling is not treated here in order to keep the presentation focused, but the same conditional-mean switching viewpoint can be combined with Markovian stochastic-approximation tools, as in~\citep{lee2026lyapunovcertified}. A replay buffer can also be used to approximate the independent sampling required below, following the experience-replay idea used in deep Q-networks (DQN)~\citep{mnih2015human}.

We now consider a sampled counterpart of the HBCQL in~\Cref{alg:cpd-hb-update}. At time $k$, we sample
\begin{equation*}
  i_k\sim d,
  \qquad
  j_k\sim d,
\end{equation*}
where $j_k$ is independent of $i_k$ and of the transition sample used below.
After $i_k=(s_k,a_k)$ is selected, draw $s'_k$ according to
$P(\cdot\mid s_k,a_k)$ and form
\begin{equation}
  \Delta_k
  :=r(s_k,a_k,s'_k)
    +\gamma\max_{b\in\mathcal A}Q_k(s'_k,b)-Q_k(s_k,a_k).
  \label{eq:sampled-hb-td-error}
\end{equation}
The stochastic HBCQL is written as
\begin{equation}
  Q_{k+1}
  =Q_k
  +\alpha\Delta_k\bigl(e_{i_k}+\delta\mathbf{1}-e_{j_k}\bigr)
  +\alpha\eta(Q_k-Q_{k-1}).
  \label{eq:sampled-hb-d-randomized-update}
\end{equation}
For each state-action pair $(s,a)$, this is equivalently written as
\begin{equation*}
  Q_{k+1}(s,a)
  =Q_k(s,a)
  +\alpha\Delta_k
  \left(\mathbf 1_{\{(s,a)=i_k\}}+\delta-\mathbf 1_{\{(s,a)=j_k\}}\right)
  +\alpha\eta\{Q_k(s,a)-Q_{k-1}(s,a)\}.
\end{equation*}
The complete pseudocode is given in~\Cref{alg:sampled-hb-d-randomized-coordinate}.
\begin{algorithm}[H]
\caption{Heavy-ball CQL}
\label{alg:sampled-hb-d-randomized-coordinate}
\begin{algorithmic}
\Require $Q_{-1},Q_0\in\R^n$, step-size $\alpha>0$, momentum gain $\eta\geq0$, scalar $\delta>0$, and full-support $d$
\For{$k=0,1,2,\ldots$}
  \State Draw $i_k=(s_k,a_k)\sim d$ and $j_k\sim d$ independently
  \State Draw $s'_k\sim P(\cdot\mid s_k,a_k)$ and compute $\Delta_k$ by~\Cref{eq:sampled-hb-td-error}
  \ForAll{$(s,a)\in\mathcal S\times\mathcal A$}
    \State Update
\[
      Q_{k+1}(s,a)
      \gets Q_k(s,a)
      +\alpha\Delta_k
      \left(\mathbf 1_{\{(s,a)=i_k\}}+\delta-\mathbf 1_{\{(s,a)=j_k\}}\right)
      +\alpha\eta\{Q_k(s,a)-Q_{k-1}(s,a)\}.
    \]
  \EndFor
\EndFor
\end{algorithmic}
\end{algorithm}
Setting $\eta=0$ in~\Cref{alg:sampled-hb-d-randomized-coordinate} gives the sampled stochastic version of CQL. Compared with classical Q-learning, the update uses two independent state-action indices, $i_k$ and $j_k$. This increases the sampling burden, but it is not the double-sampling issue associated with residual-gradient methods in~\citep{baird1995residual}: only one next-state sample is drawn from the selected pair $i_k$, while the second draw is an independent state-action index used to approximate the centering term. In practice, the second index can be drawn from a replay buffer.
Equivalently, with the noise
\begin{equation*}
  \xi_{k+1}
  :=\Delta_k\bigl(e_{i_k}+\delta\mathbf{1}-e_{j_k}\bigr)
    -H\{F(Q_k)-Q_k\},
\end{equation*}
one has $\mathbb E[\xi_{k+1}\mid Q_k,Q_{k-1}]=0$, and the recursion in~\Cref{eq:sampled-hb-d-randomized-update} can equivalently be written as
\begin{equation}
  Q_{k+1}
  =Q_k+
  \alpha H\{F(Q_k)-Q_k\}
  +\alpha\eta(Q_k-Q_{k-1})
  +\alpha\xi_{k+1}.
  \label{eq:sampled-hb-d-randomized-noise-form}
\end{equation}
The noise form in~\Cref{eq:sampled-hb-d-randomized-noise-form} shows that~\Cref{eq:sampled-hb-d-randomized-update} is an unbiased stochastic approximation of the deterministic HBCQL update in~\Cref{alg:cpd-hb-update}. The next lemma gives the corresponding conditional-mean identity.
\begin{lemma}
\label{lem:sampled-hb-unbiased}
The following conditional-mean identity holds:
\begin{equation}
  \mathbb E[Q_{k+1}-Q_k\mid Q_k,Q_{k-1}]
  =\alpha H\{F(Q_k)-Q_k\}+\alpha\eta(Q_k-Q_{k-1}).
  \label{eq:sampled-hb-d-randomized-conditional-mean}
\end{equation}
\end{lemma}
\begin{proof}
The sampling identities are
\begin{equation*}
  \mathbb E[\Delta_k e_{i_k}\mid Q_k]=D\{F(Q_k)-Q_k\},
  \qquad
  \mathbb E[\Delta_k\mid Q_k]=d^\top\{F(Q_k)-Q_k\} .
\end{equation*}
Because $j_k\sim d$ is independent of the random variables used to form
$\Delta_k$,
\begin{equation*}
  \mathbb E[\Delta_k e_{j_k}\mid Q_k]
  =\bigl(d^\top\{F(Q_k)-Q_k\}\bigr)d .
\end{equation*}
Therefore,
\begin{align*}
  &\mathbb E\left[
    \Delta_k\bigl(e_{i_k}+\delta\mathbf{1}-e_{j_k}\bigr)
    \mid Q_k
  \right] \\
  &\quad=D\{F(Q_k)-Q_k\}
    +\delta\mathbf{1}d^\top\{F(Q_k)-Q_k\}
    -dd^\top\{F(Q_k)-Q_k\} \\
  &\quad=\left[D+(\delta\mathbf{1}-d)d^\top\right]\{F(Q_k)-Q_k\}
  =H\{F(Q_k)-Q_k\}.
\end{align*}
Taking the conditional mean of
\Cref{eq:sampled-hb-d-randomized-update} gives
\Cref{eq:sampled-hb-d-randomized-conditional-mean}.
\end{proof}

The conditional-mean update in~\Cref{eq:sampled-hb-d-randomized-conditional-mean} has the same drift as the deterministic HBCQL recursion analyzed in the previous section. Hence its right-hand side admits the same SLS representation as~\Cref{eq:cpd-hb-error-system}. Therefore, the deterministic HBCQL JSR certificates apply to the conditional-mean recursion associated with~\Cref{eq:sampled-hb-d-randomized-update}. This statement is only a conditional-mean certificate: with a constant step-size, the martingale term in~\Cref{eq:sampled-hb-d-randomized-noise-form} remains, and the stochastic iterates generally track a neighborhood of the fixed point rather than converging exactly. An explicit finite-time stochastic bound is not proved in this tabular section.

\section{Numerical illustration}
\label{sec:stochastic_toy_q_learning}
This section adds an experiment for~\Cref{alg:sampled-hb-d-randomized-coordinate}, where we use the same MDP and parameter settings as in~\Cref{sec:numerical_toy_example}.
The only change is that the Bellman residual is now observed through transition
samples rather than evaluated as a full vector.
We compare two choices of this update, \(\alpha\eta=0\) (CQL counterpart) and \(\alpha\eta=0.45\) (HBCQL), against
standard tabular Q-learning with the same samples. All three methods are run with common random numbers, so differences between the curves
are not due to different random seeds. The numerical parameters are $\alpha=0.020$, $\delta=0.80$, $\alpha\eta\in\{0,0.45\}$, and $K=4000$.
The reported curve is the median over \(400\) independent Monte Carlo paths of the relative error
\begin{equation*}
  e_k^{\rm mc}
  =\frac{\|Q_k-Q^\star\|_2}{\|Q_0-Q^\star\|_2}.
\end{equation*}
Because the step-size is constant, the stochastic methods should be interpreted
as approaching a noise neighborhood of \(Q^\star\), not as converging exactly to
zero error.
\Cref{fig:toy-stochastic-qlearning-comparison} shows the median relative
\(Q\)-error curves. The CQL update (\(\alpha\eta=0\)) is already
much faster than standard Q-learning in this nonuniformly sampled
example. Adding the heavy-ball term gives a further transient acceleration: the
median first passage time below \(10^{-2}\) drops from about \(2693\) iterations
to about \(1556\) iterations.
\begin{figure}[H]
\centering
\includegraphics[width=0.78\textwidth]{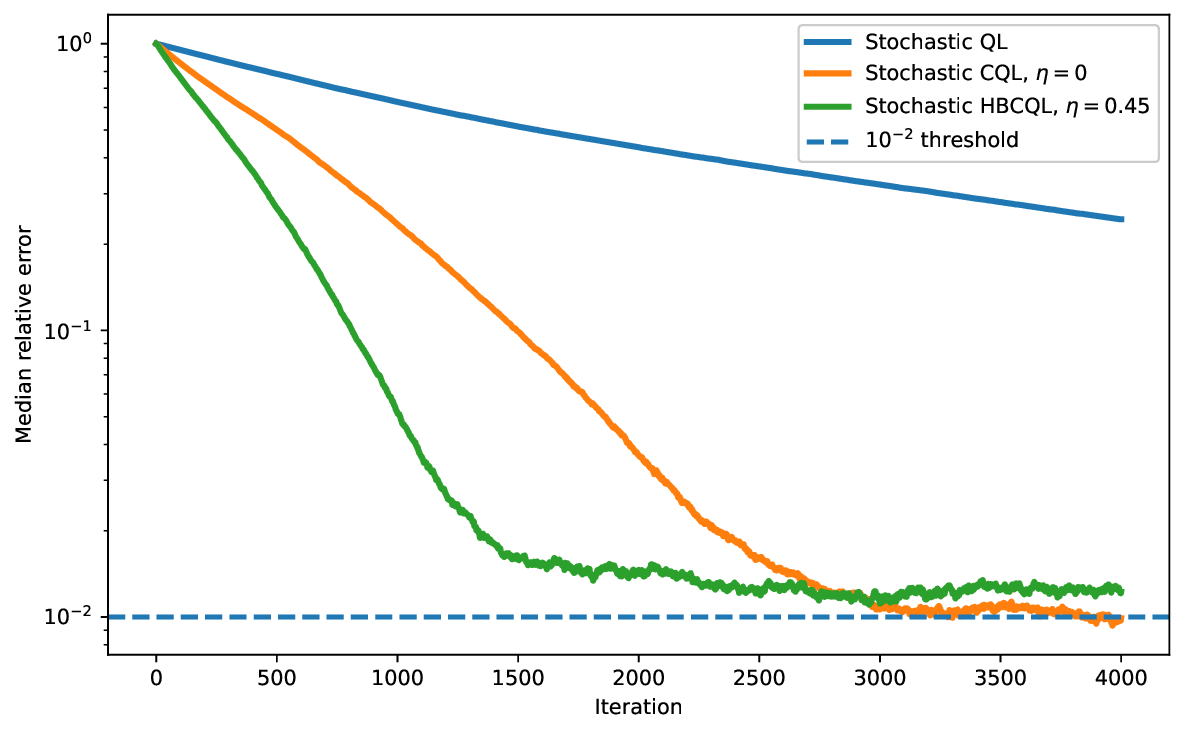}
\caption{The plotted quantity is the median relative
\(Q\)-error over 400 Monte Carlo paths.}
\label{fig:toy-stochastic-qlearning-comparison}
\end{figure}

\section{Corrected Q-learning with linear function approximation}
\label{sec:lfa_constant_preserving_weighting}

The correction method from the previous sections can be extended to deterministic Q-learning with linear function approximation, referred to as linear Q-learning (LQL). In this section, convergence of the corrected algorithms refers to the parameter fixed point defined by the corrected map below. The corresponding value vector $\Phi\theta^\star$ generally need not equal $Q^\star$, and $\theta^\star$ generally need not be the standard $D$-weighted projected Bellman solution. Let the parameter vector be $\theta\in\R^m$, and approximate the Q-vector by $\Phi\theta\in\R^n$, where the feature matrix $\Phi\in\R^{n\times m}$ has one row for each state-action pair and one column for each feature.
We impose the standard rank condition on the feature matrix.
\begin{assumption}
\label{ass:lfa-full-column-rank}
The feature matrix $\Phi\in\R^{n\times m}$ has full column rank.
\end{assumption}
To use the ideas from the previous sections, we introduce the following assumption.
\begin{assumption}
\label{ass:lfa-constant-feature-direction}
There exists $c\in\R^m$ such that
\begin{equation*}
  \Phi c=\mathbf{1}.
\end{equation*}
\end{assumption}
This assumption is easy to enforce in tabular feature constructions: one may include an intercept feature, namely a column of $\Phi$ filled with ones. Then the corresponding coordinate vector can be chosen as $c$.
The standard deterministic LQL update is summarized in~\Cref{alg:lfa-direct-phi-d-update}.
\begin{algorithm}[H]
\caption{Deterministic LQL}
\label{alg:lfa-direct-phi-d-update}
\begin{algorithmic}
\Require $\theta_0\in\R^m$, step-size $\alpha>0$, features $\Phi$, and $D=\diag(d)$
\For{$k=0,1,2,\ldots$}
  \State Update
  \[
    \theta_{k+1}\gets \theta_k+\alpha\Phi^\top D\{F(\Phi\theta_k)-\Phi\theta_k\}.
  \]
\EndFor
\end{algorithmic}
\end{algorithm}
The corresponding map is defined as
\begin{equation*}
  g(\theta):=\theta+\alpha\Phi^\top D\{F(\Phi\theta)-\Phi\theta\}.
\end{equation*}
Its fixed points, $g(\theta)=\theta$, are characterized by the associated residual equation below.
\begin{lemma}
\label{lem:lfa-direct-fixed-point-equation}
A parameter vector $\bar\theta\in\R^m$ is a fixed point of $g$ if and only if
\begin{equation}
  \Phi^\top D\{F(\Phi\bar\theta)-\Phi\bar\theta\}=0.
  \label{eq:lfa-direct-fixed-point-equation}
\end{equation}
\end{lemma}
\begin{proof}
A fixed point of $g$ satisfies
\begin{equation*}
  \bar\theta=\bar\theta+\alpha\Phi^\top D\{F(\Phi\bar\theta)-\Phi\bar\theta\}.
\end{equation*}
Because $\alpha>0$, this is equivalent to
\Cref{eq:lfa-direct-fixed-point-equation}. Conversely, suppose that
\Cref{eq:lfa-direct-fixed-point-equation} holds. Substituting this identity into the update in~\Cref{alg:lfa-direct-phi-d-update}, with $\theta_k=\bar\theta$, gives
\begin{align*}
  \theta_{k+1}
  &=\bar\theta+
    \alpha\Phi^\top D\{F(\Phi\bar\theta)-\Phi\bar\theta\}\\
  &=\bar\theta+
    \alpha\cdot 0
    =\bar\theta .
\end{align*}
Thus the update leaves $\bar\theta$ unchanged, so $\bar\theta$ is a fixed point
of $g$.
\end{proof}

Subtracting the fixed-point identity in~\Cref{eq:lfa-direct-fixed-point-equation} from the update in~\Cref{alg:lfa-direct-phi-d-update} gives
\begin{align}
  \theta_{k+1}-\bar\theta =\theta_k-\bar\theta
    +\alpha\Phi^\top D\{F(\Phi\theta_k)-F(\Phi\bar\theta)
       -\Phi(\theta_k-\bar\theta)\}.\label{eq:5}
\end{align}
The next lemma states an SLS error system associated with the above error recursion.
\begin{lemma}
\label{lem:lfa-direct-difference-system}
Assume that $\bar\theta$ is a fixed point of $g$. Then the error recursion in~\Cref{eq:5} satisfies the SLS
\begin{equation}
  \theta_{k+1}-\bar\theta
  =A_{\mu_k}^{\rm LQL}(\theta_k-\bar\theta),
  \label{eq:lfa-direct-fixed-point-error}
\end{equation}
where for any stochastic policy $\mu$,
\begin{equation}
  A_\mu^{\rm LQL}
  :=I_m+\alpha\Phi^\top D(\gamma P\Pi^\mu-I)\Phi .
  \label{eq:lfa-direct-mode}
\end{equation}
\end{lemma}
\begin{proof}
By~\Cref{lem:bellman_difference_selector}, there is a stochastic policy
$\mu_k$ such that $F(\Phi\theta_k)-F(\Phi\bar\theta)
  =\gamma P\Pi^{\mu_k}\Phi(\theta_k-\bar\theta)$. Substituting this identity gives
\begin{align*}
  \theta_{k+1}-\bar\theta
  &=\{I_m+\alpha\Phi^\top D(\gamma P\Pi^{\mu_k}-I)\Phi\}
    (\theta_k-\bar\theta)\\
  &=A_{\mu_k}^{\rm LQL}(\theta_k-\bar\theta),
\end{align*}
which proves~\Cref{eq:lfa-direct-fixed-point-error}.
\end{proof}
For the matrix in~\Cref{eq:lfa-direct-mode}, the modes generally do not share a common eigenvector. Consequently, the common-eigenvector acceleration argument used in the tabular analysis does not apply directly. To remove this obstruction, we introduce the corrected map below.
\begin{definition}
\label{def:lfa-constant-preserving-map}
We define the correction matrix
\begin{equation}
  G :=\Phi^\top D+(\delta c-\Phi^\top D\mathbf{1})d^\top
  =\delta c d^\top+\Phi^\top D(I-\mathbf{1}d^\top),
  \qquad d^\top\mathbf{1}=1.
  \label{eq:lfa-constant-preserving-map}
\end{equation}
\end{definition}
The map $G$ has a structure similar to the tabular correction, as summarized in the following result.
\begin{lemma}
\label{lem:lfa-G-preserves-constant-feature}
The map in~\Cref{eq:lfa-constant-preserving-map} satisfies
\begin{equation}
  G\mathbf{1}=\delta c.
  \label{eq:lfa-G-preserves-constant-feature}
\end{equation}
\end{lemma}
\begin{proof}
Starting from~\Cref{eq:lfa-constant-preserving-map} and using
$d^\top\mathbf 1=1$, we compute
\begin{align*}
  G\mathbf 1
  &=\Phi^\top D\mathbf 1+(
      \delta c-\Phi^\top D\mathbf 1)d^\top\mathbf 1\\
  &=\Phi^\top D\mathbf 1+
      \delta c-\Phi^\top D\mathbf 1\\
  &=\delta c.
\end{align*}
Equivalently, from the second expression in~\Cref{eq:lfa-constant-preserving-map},
$G\mathbf 1=\delta c d^\top\mathbf 1+\Phi^\top D(I-\mathbf 1d^\top)\mathbf 1=
\delta c$.
\end{proof}

The corresponding deterministic corrected linear Q-learning (CLQL) is given in~\Cref{alg:lfa-cpd-standard-update}.
Note that this is exactly the tabular correction when $\Phi=I$ and $c=\mathbf{1}$.
\begin{algorithm}[H]
\caption{Deterministic CLQL}
\label{alg:lfa-cpd-standard-update}
\begin{algorithmic}
\Require $\theta_0\in\R^m$, step-size $\alpha>0$, features $\Phi$, and $G$ from~\Cref{eq:lfa-constant-preserving-map}
\For{$k=0,1,2,\ldots$}
  \State Update
  \[
    \theta_{k+1}
    \gets \theta_k+\alpha G\{F(\Phi\theta_k)-\Phi\theta_k\}.
  \]
\EndFor
\end{algorithmic}
\end{algorithm}
Define the corresponding feature-space map
\begin{equation}
  g(\theta)
  :=\theta+\alpha G\{F(\Phi\theta)-\Phi\theta\}.
  \label{eq:lfa-cpd-standard-map}
\end{equation}
The next lemma characterizes the fixed points of this map.
\begin{lemma}
\label{lem:lfa-cpd-fixed-point-equation}
A parameter vector $\theta^\star\in\R^m$ is a fixed point of the map $g$ in~\Cref{eq:lfa-cpd-standard-map} if and only if
\begin{equation}
  G\{F(\Phi\theta^\star)-\Phi\theta^\star\}=0.
  \label{eq:lfa-cpd-fixed-point-equation}
\end{equation}
\end{lemma}
\begin{proof}
A fixed point of $g$ satisfies $\theta^\star=\theta^\star+\alpha G\{F(\Phi\theta^\star)-\Phi\theta^\star\}$.
Because the step-size is positive, this is equivalent to $G\{F(\Phi\theta^\star)-\Phi\theta^\star\}=0$, which is~\Cref{eq:lfa-cpd-fixed-point-equation}.  Conversely, if
\Cref{eq:lfa-cpd-fixed-point-equation} holds, substituting it into the update in~\Cref{alg:lfa-cpd-standard-update} gives $\theta_{k+1}=\theta_k=\theta^\star$.
\end{proof}
Unlike the tabular matrix $H$, the map $G$ is generally not nonsingular. It is
often rectangular, since it maps a Bellman residual in $\R^n$ to a parameter
update in $\R^m$, and even in the square case its rank depends on $\Phi$, $D$,
$d$, and $c$. Hence, the feature-space analysis in the sequel uses the fixed-point equation
in~\Cref{eq:lfa-cpd-fixed-point-equation} and explicitly assumes uniqueness.
\begin{assumption}
\label{ass:lfa-cpd-unique-fixed-point}
The map $g$ in~\Cref{eq:lfa-cpd-standard-map} has a unique fixed point denoted by $\theta^\star$.
\end{assumption}

The fixed point condition in~\Cref{eq:lfa-cpd-fixed-point-equation} should be compared with the standard LQL update in~\Cref{alg:lfa-direct-phi-d-update}. A fixed point corresponding to the standard LQL satisfies $\Phi^\top D\{F(\Phi\theta)-\Phi\theta\}=0$, which is the usual $D$-weighted projected Bellman equation: the Bellman
residual is orthogonal to every feature column in the $D$-weighted inner
product. By contrast, a fixed point of the update in~\Cref{alg:lfa-cpd-standard-update} satisfies $G\{F(\Phi\theta)-\Phi\theta\}=0$.
Thus the two fixed-point equations generally need not have the same solution. The convergence and acceleration statements for CLQL and HBCLQL below are therefore statements about the unique corrected fixed point $\theta^\star$ in~\Cref{ass:lfa-cpd-unique-fixed-point}, not about $Q^\star$ or, in general, the standard projected Bellman solution.

Next, subtracting the fixed-point equation in~\Cref{eq:lfa-cpd-fixed-point-equation} from the update in~\Cref{alg:lfa-cpd-standard-update} gives the error recursion
\begin{align}
  \theta_{k+1}-\theta^\star =(\theta_k-\theta^\star)+
    \alpha G\{F(\Phi\theta_k)-F(\Phi\theta^\star)
    -\Phi(\theta_k-\theta^\star)\}.\label{eq:4}
\end{align}
We next rewrite this error recursion as an SLS.
\begin{lemma}
\label{lem:lfa-cpd-difference-system}
The error recursion in~\Cref{eq:4} satisfies the SLS
\begin{equation}
  \theta_{k+1}-\theta^\star
  =A_{\mu_k}^{\rm CLQL}(\theta_k-\theta^\star),
  \label{eq:lfa-cpd-fixed-point-error}
\end{equation}
where $\mu_k$ is a stochastic policy and for any stochastic policy $\mu$,
\begin{equation*}
  A_\mu^{\rm CLQL}
  :=I_m+\alpha G(\gamma P\Pi^\mu-I)\Phi .
\end{equation*}
\end{lemma}
\begin{proof}
By~\Cref{lem:bellman_difference_selector}, there is a stochastic policy
$\mu_k$ such that $F(\Phi\theta_k)-F(\Phi\theta^\star)=\gamma P\Pi^{\mu_k}\Phi(\theta_k-\theta^\star)$. Substitution gives $\theta_{k+1}-\theta^\star
=\{I_m+\alpha G(\gamma P\Pi^{\mu_k}-I)\Phi\} (\theta_k-\theta^\star)=A_{\mu_k}^{\rm CLQL}(\theta_k-\theta^\star)$, which proves~\Cref{eq:lfa-cpd-fixed-point-error}.
\end{proof}
The corresponding SLS family is
\begin{equation*}
  \mathcal A^{\rm CLQL}
  :=\left\{A_\pi^{\rm CLQL}:\pi\in\Theta\right\}.
\end{equation*}
As in the tabular case, every mode in $\mathcal A^{\rm CLQL}$ has a common eigenvector as shown below.
\begin{lemma}
\label{lem:lfa-cpd-common-vector}
For every stochastic policy $\mu$,
\begin{equation}
  A_\mu^{\rm CLQL}c
  =\left(1-\alpha\delta(1-\gamma)\right)c .
  \label{eq:lfa-cpd-common-vector}
\end{equation}
\end{lemma}
\begin{proof}
Since $\Phi c=\mathbf 1$ and $P\Pi^\mu\mathbf 1=\mathbf 1$, one has $(\gamma P\Pi^\mu-I)\Phi c =(\gamma P\Pi^\mu-I)\mathbf 1 =-(1-\gamma)\mathbf 1$. Using~\Cref{eq:lfa-G-preserves-constant-feature},
\begin{align*}
  A_\mu^{\rm CLQL}c
  &=\left[I_m+\alpha G(\gamma P\Pi^\mu-I)\Phi\right]c\\
  &=c-\alpha(1-\gamma)G\mathbf 1\\
  &=c-\alpha\delta(1-\gamma)c\\
  &=\left(1-\alpha\delta(1-\gamma)\right)c.
\end{align*}
This completes the proof.
\end{proof}
By~\Cref{eq:lfa-cpd-common-vector}, the direction $\operatorname{span}\{c\}$ is a common eigenvector direction of the SLS family $\mathcal A^{\rm CLQL}$. As in the tabular analysis, we decompose the feature-space dynamics into this common direction and the induced dynamics on the quotient space. This decomposition provides the same type of convergence benchmark as in the tabular analysis. Let $\bar{\mathcal A}^{\rm CLQL}$ denote the quotient family induced by $\mathcal A^{\rm CLQL}$ on $\R^m/\operatorname{span}\{c\}$.
\begin{proposition}
\label{prop:lfa-cpd-jsr-benchmark}
Under~\Cref{ass:cpd-constant-mode-range}, it holds that
\begin{equation}
  \rho(\mathcal A^{\rm CLQL})
  =\max\left\{1-\alpha\delta(1-\gamma),
  \rho(\bar{\mathcal A}^{\rm CLQL})\right\}.
  \label{eq:lfa-cpd-jsr-decomposition}
\end{equation}
\end{proposition}
\begin{proof}
Choose $W\in\R^{m\times(m-1)}$ so that $S:=[c\ W]$ is nonsingular,
and write
\begin{equation*}
  S^{-1}=\begin{bmatrix}\ell^\top\\ L\end{bmatrix}.
\end{equation*}
Then $Lc=0$ and $LW=I_{m-1}$. By~\Cref{eq:lfa-cpd-common-vector}, each deterministic mode satisfies
\begin{equation*}
  A_\pi^{\rm CLQL}c=\bigl(1-\alpha\delta(1-\gamma)\bigr)c.
\end{equation*}
Consequently,
\begin{align*}
  S^{-1}A_\pi^{\rm CLQL}S
  &=\begin{bmatrix}\ell^\top\\ L\end{bmatrix}
    A_\pi^{\rm CLQL}
    \begin{bmatrix}c&W\end{bmatrix}\\
  &=\begin{bmatrix}
      \ell^\top A_\pi^{\rm CLQL}c & \ell^\top A_\pi^{\rm CLQL}W\\
      L A_\pi^{\rm CLQL}c & L A_\pi^{\rm CLQL}W
    \end{bmatrix}\\
  &=\begin{bmatrix}
      1-\alpha\delta(1-\gamma) & \ell^\top A_\pi^{\rm CLQL}W\\
      0 & L A_\pi^{\rm CLQL}W
    \end{bmatrix}.
\end{align*}
The lower-right block is the matrix of the induced quotient mode in the basis
given by $W$, so write
\begin{equation*}
  \bar A_\pi^{\rm CLQL}=L A_\pi^{\rm CLQL}W,
  \qquad
  \bar{\mathcal A}^{\rm CLQL}
  =\{\bar A_\pi^{\rm CLQL}:\pi\in\Theta\}.
\end{equation*}
For any product $A_{\pi_k}^{\rm CLQL}\cdots A_{\pi_1}^{\rm CLQL}$, the same
similarity gives
\begin{equation*}
  S^{-1}A_{\pi_k}^{\rm CLQL}\cdots A_{\pi_1}^{\rm CLQL}S
  =\begin{bmatrix}
      \bigl(1-\alpha\delta(1-\gamma)\bigr)^k & B_k\\
      0 & \bar A_{\pi_k}^{\rm CLQL}\cdots \bar A_{\pi_1}^{\rm CLQL}
    \end{bmatrix}
\end{equation*}
for some upper-right block $B_k$. Hence the whole family is simultaneously
block upper triangular with diagonal families
$\{1-\alpha\delta(1-\gamma)\}$ and
$\bar{\mathcal A}^{\rm CLQL}$. The block-triangular JSR identity in
\Cref{lem:block-triangular-jsr} gives
\begin{equation*}
  \rho(\mathcal A^{\rm CLQL})
  =\max\left\{|1-\alpha\delta(1-\gamma)|,\rho(\bar{\mathcal A}^{\rm CLQL})\right\}.
\end{equation*}
Under~\Cref{ass:cpd-constant-mode-range}, $1-\alpha\delta(1-\gamma)>0$, so this is exactly
\Cref{eq:lfa-cpd-jsr-decomposition}.
\end{proof}

The block-triangular decomposition also gives the following lower bounds.
\begin{corollary}
\label{lem:lfa-cpd-weak-inequality}
The following inequality holds:
\begin{equation*}
  \rho(\bar{\mathcal A}^{\rm CLQL})\leq \rho(\mathcal A^{\rm CLQL}).
\end{equation*}
\end{corollary}
\begin{proof}
The block-triangular argument in the proof of~\Cref{prop:lfa-cpd-jsr-benchmark}
gives, before using~\Cref{ass:cpd-constant-mode-range},
\[
\rho(\mathcal A^{\rm CLQL})
=\max\left\{|1-\alpha\delta(1-\gamma)|,
\rho(\bar{\mathcal A}^{\rm CLQL})\right\}.
\]
The maximum is at least its second entry, so
$\rho(\bar{\mathcal A}^{\rm CLQL})\le \rho(\mathcal A^{\rm CLQL})$.
\end{proof}
\begin{lemma}\label{lem:lfa-cpd-constant-lower-bound}
Under~\Cref{ass:cpd-constant-mode-range}, the following inequality holds:
\begin{equation*}
  1-\alpha\delta(1-\gamma)\leq \rho(\mathcal A^{\rm CLQL}).
\end{equation*}
\end{lemma}
\begin{proof}
The claim follows directly from~\Cref{eq:lfa-cpd-jsr-decomposition}, since the maximum is at least its first entry.
\end{proof}
As in the tabular case, acceleration is certified under a strict JSR gap between the full SLS family $\mathcal A^{\rm CLQL}$ and the projected SLS family $\bar{\mathcal A}^{\rm CLQL}$ as summarized in the following assumption.
\begin{assumption}
\label{ass:lfa-cpd-projected-gap}
Suppose that the following strict inequality holds:
\begin{equation}
  \rho(\bar{\mathcal A}^{\rm CLQL})<\rho(\mathcal A^{\rm CLQL}).
  \label{eq:lfa-cpd-projected-gap}
\end{equation}
\end{assumption}
Under~\Cref{ass:cpd-constant-mode-range,ass:lfa-cpd-projected-gap}, the decomposition in~\Cref{eq:lfa-cpd-jsr-decomposition} shows that $\rho(\mathcal A^{\rm CLQL})$ is identical to the eigenvalue associated with the common eigenvector direction.
\begin{lemma}
\label{lem:lfa-cpd-standard-benchmark}
Under~\Cref{ass:cpd-constant-mode-range,ass:lfa-cpd-projected-gap}, we have
\begin{equation}
  \rho(\mathcal A^{\rm CLQL})=\bigl(1-\alpha\delta(1-\gamma)\bigr) .
  \label{eq:lfa-cpd-benchmark}
\end{equation}
\end{lemma}
\begin{proof}
By~\Cref{eq:lfa-cpd-jsr-decomposition},
\begin{align*}
  \rho(\mathcal A^{\rm CLQL})
  &=\max\left\{1-\alpha\delta(1-\gamma),
    \rho(\bar{\mathcal A}^{\rm CLQL})\right\}.
\end{align*}
The strict inequality in~\Cref{eq:lfa-cpd-projected-gap} rules out the second
entry as the maximum. Therefore,
\begin{align*}
  \rho(\mathcal A^{\rm CLQL})
  &=1-\alpha\delta(1-\gamma).
\end{align*}
This is exactly~\Cref{eq:lfa-cpd-benchmark}.
\end{proof}

We now add the heavy-ball momentum term to the CLQL recursion in~\Cref{alg:lfa-cpd-standard-update}. The resulting algorithm is summarized in~\Cref{alg:lfa-cpd-hb-update} and is called heavy-ball corrected linear Q-learning (HBCLQL). As above, $\eta$ is a momentum gain and $\alpha\eta$ is the actual momentum coefficient.
\begin{algorithm}[H]
\caption{Deterministic HBCLQL}
\label{alg:lfa-cpd-hb-update}
\begin{algorithmic}
\Require $\theta_{-1},\theta_0\in\R^m$, step-size $\alpha>0$, momentum gain $\eta\geq0$, features $\Phi$, and $G$ from~\Cref{eq:lfa-constant-preserving-map}
\For{$k=0,1,2,\ldots$}
  \State Update
  \begin{equation}
    \theta_{k+1}
    \gets \theta_k+\alpha G\{F(\Phi\theta_k)-\Phi\theta_k\}
    +\alpha\eta(\theta_k-\theta_{k-1}).
    \label{eq:lfa-cpd-hb-update}
  \end{equation}
\EndFor
\end{algorithmic}
\end{algorithm}

The next lemma writes the resulting second-order error recursion as a first-order
augmented SLS.
\begin{lemma}
\label{lem:lfa-cpd-hb-sls}
Around any fixed point $\theta^\star$ of the map in~\Cref{eq:lfa-cpd-standard-map}, the
heavy-ball error satisfies
\begin{equation}
  \begin{bmatrix}
    \theta_{k+1}-\theta^\star\\
    \theta_k-\theta^\star
  \end{bmatrix}
  =A_{\mu_k}^{\rm HBCLQL}
  \begin{bmatrix}
    \theta_k-\theta^\star\\
    \theta_{k-1}-\theta^\star
  \end{bmatrix}
  \label{eq:lfa-cpd-hb-error-system}
\end{equation}
for a stochastic Bellman-difference selector $\mu_k$, where
\begin{equation}
  A_\mu^{\rm HBCLQL}
  :=
  \begin{bmatrix}
    A_\mu^{\rm CLQL}+\alpha\eta I_m&-\alpha\eta I_m\\
    I_m&0
  \end{bmatrix}.
  \label{eq:lfa-cpd-hb-mode}
\end{equation}
\end{lemma}
\begin{proof}
Subtracting $\theta^\star$ from
\Cref{eq:lfa-cpd-hb-update} gives
\begin{align*}
  \theta_{k+1}-\theta^\star
  &=(\theta_k-\theta^\star)+
    \alpha G\{F(\Phi\theta_k)-F(\Phi\theta^\star)
    -\Phi(\theta_k-\theta^\star)\}\\
  &\qquad+\alpha\eta\{(\theta_k-\theta^\star)-(\theta_{k-1}-\theta^\star)\}.
\end{align*}
By the fixed-point error identity in~\Cref{lem:lfa-cpd-difference-system},
\begin{equation*}
  (\theta_k-\theta^\star)+\alpha G\{F(\Phi\theta_k)-F(\Phi\theta^\star)
  -\Phi(\theta_k-\theta^\star)\}
  =A_{\mu_k}^{\rm CLQL}(\theta_k-\theta^\star).
\end{equation*}
Therefore
\begin{equation*}
  \theta_{k+1}-\theta^\star
  =(A_{\mu_k}^{\rm CLQL}+\alpha\eta I_m)(\theta_k-\theta^\star)
  -\alpha\eta(\theta_{k-1}-\theta^\star).
\end{equation*}
Stacking $\theta_{k+1}-\theta^\star$ and $\theta_k-\theta^\star$ yields~\Cref{eq:lfa-cpd-hb-error-system} with the
mode in~\Cref{eq:lfa-cpd-hb-mode}.  The matrix $A_\mu^{\rm CLQL}$ depends
linearly on $P\Pi^\mu$, so the convex-hull statement in
\Cref{lem:bellman_difference_selector} gives
$A_{\mu_k}^{\rm HBCLQL}\in\co(\mathcal A^{\rm HBCLQL})$.
\end{proof}

The corresponding SLS family is given by
\begin{equation}
  \mathcal A^{\rm HBCLQL}
  :=\left\{A_\pi^{\rm HBCLQL}:\pi\in\Theta\right\}.
  \label{eq:lfa-cpd-hb-family}
\end{equation}
As in the tabular case, the augmented feature-space dynamics preserve the subspace
\begin{equation*}
  {\mathcal I}:=\left\{
  \begin{bmatrix}ac\\bc\end{bmatrix}:a,b\in\R
  \right\}
  \subset \R^{2m}.
\end{equation*}
Indeed, for every policy $\mu$, \Cref{eq:lfa-cpd-common-vector,eq:lfa-cpd-hb-mode} imply that $A_\mu^{\rm HBCLQL}$ maps any vector in ${\mathcal I}$ to another vector in ${\mathcal I}$. If the augmented error belongs to this subspace, then the next augmented error also belongs to it, and the scalar coefficients are updated by the following constant matrix:
\begin{equation*}
  C
  :=
  \begin{bmatrix}
    \bigl(1-\alpha\delta(1-\gamma)\bigr)+\alpha\eta&-\alpha\eta\\
    1&0
  \end{bmatrix}.
\end{equation*}

Using the quotient basis from~\Cref{prop:lfa-cpd-jsr-benchmark}, define
\begin{equation*}
  \bar A_\pi^{\rm HBCLQL}
  :=\begin{bmatrix}\bar A_\pi^{\rm CLQL}+\alpha\eta I_{m-1}&-\alpha\eta I_{m-1}\\ I_{m-1}&0\end{bmatrix},
  \qquad
  \bar{\mathcal A}^{\rm HBCLQL}
  :=\left\{\bar A_\pi^{\rm HBCLQL}:\pi\in\Theta\right\}.
\end{equation*}
The block-triangular decomposition in the same basis gives
\begin{equation}
  \rho(\mathcal A^{\rm HBCLQL})
  =\max\{\rho(C),\rho(\bar{\mathcal A}^{\rm HBCLQL})\}.
  \label{eq:lfa-cpd-hb-jsr-decomposition}
\end{equation}
\begin{proof}[Justification of~\Cref{eq:lfa-cpd-hb-jsr-decomposition}]
Let \(S=[c\ W]\) and \(S^{-1}=\begin{bmatrix}\ell^\top\\ L\end{bmatrix}\)
be the quotient basis used in the proof of~\Cref{prop:lfa-cpd-jsr-benchmark}.
In the augmented basis \(\diag(S,S)\), the coordinates are ordered as current
constant, current quotient, previous constant, and previous quotient.  For each
\(\pi\), this gives
\begin{equation*}
\begin{bmatrix}
1-\alpha\delta(1-\gamma)+\alpha\eta & \ell^\top A_\pi^{\rm CLQL}W & -\alpha\eta & 0\\
0 & \bar A_\pi^{\rm CLQL}+\alpha\eta I_{m-1} & 0 & -\alpha\eta I_{m-1}\\
1 & 0 & 0 & 0\\
0 & I_{m-1} & 0 & 0
\end{bmatrix}.
\end{equation*}
Permuting the coordinates to place the two constant coordinates first gives
\begin{equation*}
  A_\pi^{\rm HBCLQL}
  \sim
  \begin{bmatrix}
    C & E_\pi\\
    0 & \bar A_\pi^{\rm HBCLQL}
  \end{bmatrix},
  \qquad
  E_\pi:=
  \begin{bmatrix}
    \ell^\top A_\pi^{\rm CLQL}W & 0\\
    0 & 0
  \end{bmatrix}.
\end{equation*}
The same similarity and permutation are used for every mode, so
\Cref{lem:block-triangular-jsr} gives~\Cref{eq:lfa-cpd-hb-jsr-decomposition}.
\end{proof}
The next theorem gives conditions under which, when the feature-space fixed point
is unique, HBCLQL converges faster than CLQL.
\begin{theorem}
\label{thm:lfa-cpd-hb-improvement}
Under~\Cref{ass:cpd-constant-mode-range,ass:lfa-cpd-unique-fixed-point,ass:lfa-cpd-projected-gap}, there
exists $\eta_0>0$ such that every momentum gain satisfying
\begin{equation*}
  0<\alpha\eta<\min\{\bigl(1-\alpha\delta(1-\gamma)\bigr)^2,\eta_0\}
\end{equation*}
satisfies
\begin{equation}
  \rho(\mathcal A^{\rm HBCLQL})
  <\rho(\mathcal A^{\rm CLQL})=\bigl(1-\alpha\delta(1-\gamma)\bigr) .
  \label{eq:lfa-cpd-hb-improvement}
\end{equation}
\end{theorem}
\begin{proof}
The unique fixed-point assumption ensures that the JSR comparison below is a
convergence-rate comparison to the unique fixed point; the JSR calculation itself
uses the remaining assumptions.  From~\Cref{eq:lfa-cpd-projected-gap,eq:lfa-cpd-benchmark},
\begin{equation*}
  \rho(\mathcal A^{\rm CLQL})=1-\alpha\delta(1-\gamma),
  \qquad
  \rho(\bar{\mathcal A}^{\rm CLQL})<1-\alpha\delta(1-\gamma).
\end{equation*}
At $\eta=0$, the quotient augmented family induced by
\Cref{eq:lfa-cpd-hb-family} has matrices
\begin{equation*}
  \begin{bmatrix}\bar A_\pi^{\rm CLQL}&0\\I&0\end{bmatrix},
\end{equation*}
and therefore has JSR $\rho(\bar{\mathcal A}^{\rm CLQL})$. By continuity of the
JSR for finite matrix families~\citep[Chapter~1]{jungers2009joint}, there exists $\eta_0>0$ such that
\begin{equation*}
  \rho(\bar{\mathcal A}^{\rm HBCLQL})<1-\alpha\delta(1-\gamma),
  \qquad 0<\alpha\eta<\eta_0.
\end{equation*}
If $0<\alpha\eta<\bigl(1-\alpha\delta(1-\gamma)\bigr)^2$, then
\Cref{app:prop-lfa-cpd-hb-common-acceleration} gives
$\rho(C)<1-\alpha\delta(1-\gamma)$.  Using~\Cref{eq:lfa-cpd-hb-jsr-decomposition},
\begin{align*}
  \rho(\mathcal A^{\rm HBCLQL})
  &=\max\{\rho(C),\rho(\bar{\mathcal A}^{\rm HBCLQL})\}\\
  &<1-\alpha\delta(1-\gamma)
   =\rho(\mathcal A^{\rm CLQL}).
\end{align*}
This proves~\Cref{eq:lfa-cpd-hb-improvement}.
\end{proof}

This result establishes a condition in which HBCLQL improves over CLQL. As in the tabular case, the comparison is for a fixed step-size because the actual momentum coefficient is $\alpha\eta$; when $\eta$ is fixed and $\alpha$ is taken to zero, the common-mode improvement is second order in $\alpha$. Conditions under which both HBCLQL and CLQL improve over LQL can be obtained by the same argument as in the tabular case; the details are omitted to avoid repetition.

\section{Numerical illustration}
\label{sec:feature_numerical_toy_example}
This example compares~\Cref{alg:lfa-direct-phi-d-update,alg:lfa-cpd-standard-update,alg:lfa-cpd-hb-update}.
We use the same two-state, two-action MDP as in~\Cref{sec:numerical_toy_example}, together with the feature matrix specified below. Throughout this comparison, we use the parameters $\alpha=2.5,\delta= 2,\alpha\eta = 0.1$, and
\begin{equation*}
  \Phi=
  \begin{bmatrix}
    1& 1& 1\\
    1&-1& 1\\
    1& 1&-1\\
    1&-1&-1
  \end{bmatrix},
\end{equation*}
and one verifies that $c=(1,0,0)^\top$ satisfies $\Phi c=\mathbf 1$.
The feature recursions use
\(\theta_{-1}=\theta_0=(3,-2.5,1.5)^\top\). In this example, the fixed points of LQL and CLQL coincide numerically:
\begin{equation*}
  \theta^\star\approx(6.21874,0.75652,0.02213)^\top,
  \qquad
  \Phi\theta^\star\approx(6.99739,5.48435,6.95313,5.44009)^\top .
\end{equation*}
For the comparison, we define the relative error
\begin{equation*}
  e_k^\theta=\frac{\left\|\theta_k-\theta^\star\right\|_2}
  {\left\|\theta_0-\theta^\star\right\|_2}.
\end{equation*}
\Cref{fig:toy-feature-convergence} compares the three algorithms.
CLQL reaches the \(10^{-6}\) threshold after 22 iterations, compared with 28 iterations for
LQL. Adding the small heavy-ball momentum term further reduces the hit time to 15 iterations. All three final relative
parameter errors are at the \(10^{-15}\) scale by the end of the run. This deterministic feature example is an empirical illustration; no projected-gap certificate is claimed for these displayed parameters.

\begin{figure}[H]
\centering
\includegraphics[width=0.78\textwidth]{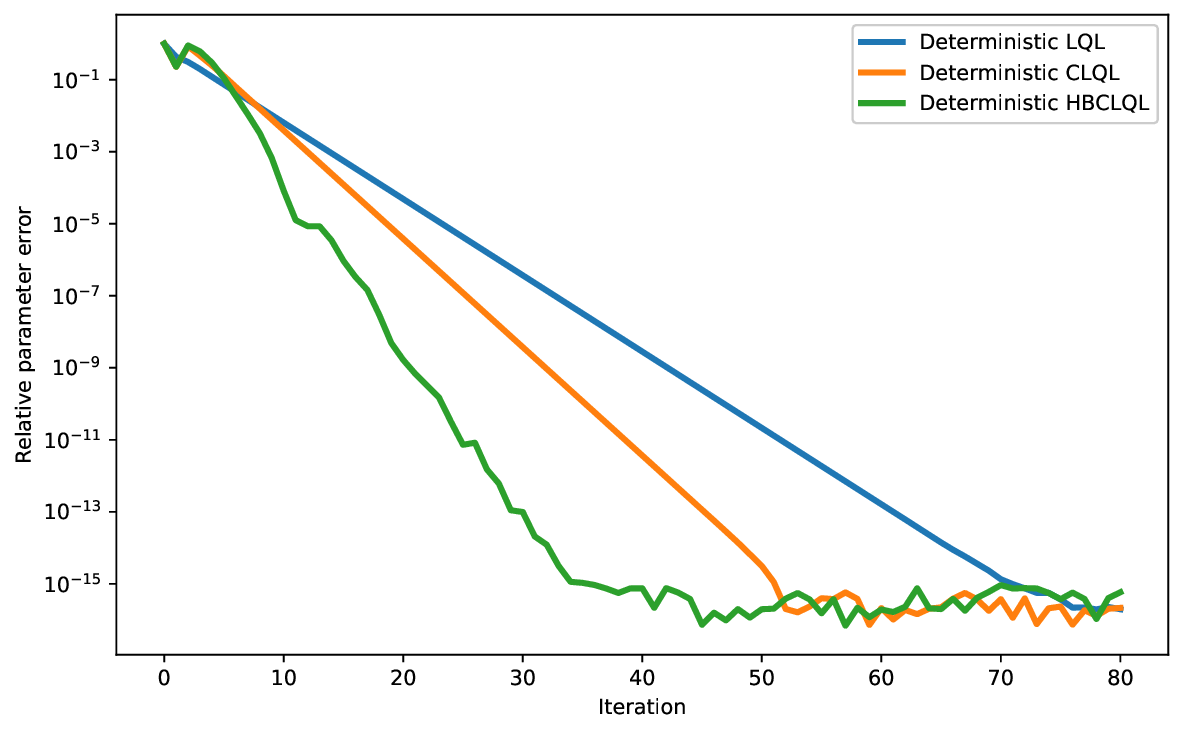}
\caption{The plot contains exactly the three method curves in
\Cref{alg:lfa-direct-phi-d-update,alg:lfa-cpd-standard-update,alg:lfa-cpd-hb-update}.}
\label{fig:toy-feature-convergence}
\end{figure}

The actual momentum coefficient remains deliberately small in this illustration. Larger
values of \(\delta\) or \(\alpha\eta\) can make the numerical comparison even more
aggressive, but this moderate setting already shows the intended ordering:
CLQL is faster than the direct feature recursion, and HBCLQL is faster than CLQL on this instance.

\section{Sampled stochastic heavy-ball linear function approximation}
\label{sec:sampled_stochastic_hb_lfa}

Based on the deterministic HBCLQL analysis, this section extends the sampled model-free construction to linear function approximation.
For comparison, we first recall sampled Q-learning with linear function approximation. Let
$Q_k=\Phi\theta_k$. Given $\theta_k$, draw $i_k=(s_k,a_k)\sim d$, draw
$s_k^\prime\sim P(\cdot\mid s_k,a_k)$, and define
\begin{equation}
  \Delta_k
  :=r(s_k,a_k,s_k^\prime)
    +\gamma\max_{b\in\mathcal A}(\Phi\theta_k)(s_k^\prime,b)
    -(\Phi\theta_k)(s_k,a_k).
  \label{eq:sampled-standard-lfa-q-learning-td-error}
\end{equation}
The corresponding standard linear Q-learning update is given in~\Cref{alg:sampled-standard-lfa-q-learning}.
\begin{algorithm}[H]
\caption{Stochastic LQL}
\label{alg:sampled-standard-lfa-q-learning}
\begin{algorithmic}
\Require $\theta_0\in\R^m$, step-sizes $\alpha_k>0$, features $\Phi$, and full-support sampling distribution $d$
\For{$k=0,1,2,\ldots$}
  \State Draw $i_k=(s_k,a_k)\sim d$
  \State Draw $s_k^\prime\sim P(\cdot\mid s_k,a_k)$ and compute $\Delta_k$ by~\Cref{eq:sampled-standard-lfa-q-learning-td-error}
  \State Update
  \[
    \theta_{k+1}\gets\theta_k+\alpha_k\Delta_k\Phi^\top e_{i_k}.
  \]
\EndFor
\end{algorithmic}
\end{algorithm}

We now consider a stochastic HBCLQL counterpart. Set $Q_k=\Phi\theta_k$ and start from
$\theta_{-1},\theta_0\in\R^m$. At time $k$, we draw $i_k\sim d$ and $j_k\sim d$
independently. If $i_k=(s_k,a_k)$, draw $s'_k\sim P(\cdot\mid s_k,a_k)$ and form
\begin{equation}
  \Delta_k
  :=r(s_k,a_k,s'_k)
    +\gamma\max_{b\in\mathcal A}(\Phi\theta_k)(s'_k,b)
    -(\Phi\theta_k)(s_k,a_k).
  \label{eq:sampled-lfa-hb-td-error}
\end{equation}
The parameter is then updated as in~\Cref{alg:sampled-lfa-hb-update}.
\begin{algorithm}[H]
\caption{Stochastic HBCLQL}
\label{alg:sampled-lfa-hb-update}
\begin{algorithmic}
\Require $\theta_{-1},\theta_0\in\R^m$, step-size $\alpha>0$, momentum gain $\eta\geq0$, scalar $\delta>0$, features $\Phi$, constant direction $c$, and full-support $d$
\For{$k=0,1,2,\ldots$}
  \State Draw $i_k=(s_k,a_k)\sim d$ and $j_k\sim d$ independently
  \State Draw $s'_k\sim P(\cdot\mid s_k,a_k)$ and compute $\Delta_k$ by~\Cref{eq:sampled-lfa-hb-td-error}
  \State Update
  \Statex \[
    \theta_{k+1}
    \gets \theta_k
    +\alpha\Delta_k\left(\Phi^\top e_{i_k}+\delta c-\Phi^\top e_{j_k}\right)
    +\alpha\eta(\theta_k-\theta_{k-1}).
  \]
\EndFor
\end{algorithmic}
\end{algorithm}

The update in~\Cref{alg:sampled-lfa-hb-update} is an unbiased stochastic approximation of the deterministic HBCLQL recursion in~\Cref{eq:lfa-cpd-hb-update}. The next lemma summarizes this conditional-mean relation.
\begin{lemma}
\label{lem:sampled-lfa-hb-unbiased}
The following equation holds:
\begin{equation}
  \mathbb E[\theta_{k+1}-\theta_k\mid \theta_k,\theta_{k-1}]
  =\alpha G\{F(\Phi\theta_k)-\Phi\theta_k\}
  +\alpha\eta(\theta_k-\theta_{k-1}).
  \label{eq:sampled-lfa-hb-conditional-mean}
\end{equation}
This is the sampled conditional-mean counterpart of
\Cref{eq:lfa-cpd-hb-update}.
\end{lemma}
\begin{proof}
The conditional identities are
\begin{align*}
  \mathbb E[\Delta_k\Phi^\top e_{i_k}\mid \theta_k]
  &=\Phi^\top D\{F(\Phi\theta_k)-\Phi\theta_k\},\\
  \mathbb E[\Delta_k\Phi^\top e_{j_k}\mid \theta_k]
  &=\bigl(d^\top\{F(\Phi\theta_k)-\Phi\theta_k\}\bigr)\Phi^\top D\mathbf{1},\\
  \mathbb E[\Delta_k\mid \theta_k]
  &=d^\top\{F(\Phi\theta_k)-\Phi\theta_k\}.
\end{align*}
Therefore
\begin{align*}
  &\mathbb E\left[
    \Delta_k\left(\Phi^\top e_{i_k}+\delta c-\Phi^\top e_{j_k}\right)
    \mid \theta_k
  \right] \\
  &\quad=\Phi^\top D\{F(\Phi\theta_k)-\Phi\theta_k\}
  +\left(\delta c-\Phi^\top D\mathbf{1}\right)
    d^\top\{F(\Phi\theta_k)-\Phi\theta_k\} \\
  &\quad=G\{F(\Phi\theta_k)-\Phi\theta_k\}.
\end{align*}
Taking the conditional mean of the update in~\Cref{alg:sampled-lfa-hb-update} gives the second claim.
\end{proof}
For the stochastic convergence proof, define the filtration by
\[
\mathcal F_0:=\sigma(\theta_{-1},\theta_0),
\qquad
\mathcal F_k:=\sigma\left(\theta_{-1},\theta_0,
\{(i_t,j_t,s'_t):0\le t\le k-1\}\right),
\quad k\ge1 .
\]
For the samples used at time $k$, set
\begin{equation}
\widehat g_k(\theta)
:=
\left(
  r(s_k,a_k,s'_k)
  +\gamma\max_{b\in\mathcal A}(\Phi\theta)(s'_k,b)
  -(\Phi\theta)(s_k,a_k)
\right)
\left(\Phi^\top e_{i_k}+\delta c-\Phi^\top e_{j_k}\right).
\label{eq:sampled-lfa-hb-sample-map}
\end{equation}
Then \Cref{alg:sampled-lfa-hb-update} can be written as
\begin{equation}
  \theta_{k+1}
  =\theta_k+\alpha\widehat g_k(\theta_k)
  +\alpha\eta(\theta_k-\theta_{k-1}).
  \label{eq:sampled-lfa-hb-vector-update}
\end{equation}
Define the martingale-difference term
\begin{equation}
  w_k
  :=\widehat g_k(\theta_k)-G\{F(\Phi\theta_k)-\Phi\theta_k\}.
  \label{eq:sampled-lfa-hb-noise}
\end{equation}
By \Cref{lem:sampled-lfa-hb-unbiased},
$\E[w_k\mid\mathcal F_k]=0$.

\begin{lemma}
\label{lem:sampled-lfa-hb-error-system}
Assume that \Cref{ass:lfa-cpd-unique-fixed-point} holds, and let $\theta^\star$ be the unique fixed point of the map in~\Cref{eq:lfa-cpd-standard-map}.  Define
\begin{equation*}
  x_k:=
  \begin{bmatrix}
    \theta_k-\theta^\star\\
    \theta_{k-1}-\theta^\star
  \end{bmatrix},
  \qquad
  \xi_k:=
  \begin{bmatrix}
    w_k\\0
  \end{bmatrix}.
\end{equation*}
Then, for each $k\ge0$, there exists an $\mathcal F_k$-measurable stochastic policy $\mu_k$ such that
\begin{equation}
  x_{k+1}=A_{\mu_k}^{\rm HBCLQL}x_k+\alpha\xi_k,
  \qquad
  \E[\xi_k\mid\mathcal F_k]=0 .
  \label{eq:sampled-lfa-hb-error-system}
\end{equation}
\end{lemma}
\begin{proof}
Using \Cref{eq:sampled-lfa-hb-noise}, the update in
\Cref{eq:sampled-lfa-hb-vector-update} becomes
\begin{equation*}
  \theta_{k+1}
  =\theta_k+\alpha G\{F(\Phi\theta_k)-\Phi\theta_k\}
  +\alpha\eta(\theta_k-\theta_{k-1})+\alpha w_k .
\end{equation*}
Subtracting $\theta^\star$ and using the fixed-point equation in
\Cref{lem:lfa-cpd-fixed-point-equation} gives the deterministic HBCLQL error
term plus $\alpha w_k$.  The deterministic term is represented by
\Cref{lem:lfa-cpd-hb-sls}; hence there is an $\mathcal F_k$-measurable Bellman-difference selector $\mu_k$ such that the stacked recursion is
\Cref{eq:sampled-lfa-hb-error-system}.  The martingale-difference identity follows from \Cref{eq:sampled-lfa-hb-noise} and
\Cref{lem:sampled-lfa-hb-unbiased}.
\end{proof}

Before deriving the finite-time error bound, we collect the product-growth and noise estimates needed for the stochastic HBCLQL recursion. For a step-size $\alpha$, write $\mathcal A_\alpha^{\rm HBCLQL}$ for the family in~\Cref{eq:lfa-cpd-hb-family} formed with that value of $\alpha$. For any $\beta\in(0,1)$, define
\begin{equation}
  K_{\beta,\alpha}
  :=
  \sup_{\ell\ge0}
  \beta^{-\ell}
  \sup_{M_0,\ldots,M_{\ell-1}\in\co(\mathcal A_\alpha^{\rm HBCLQL})}
  \left\|
  M_{\ell-1}\cdots M_0
  \right\|_2 .
  \label{eq:sampled-lfa-hb-Kbeta-def}
\end{equation}
The first lemma is the product-growth fact used whenever the JSR is strictly below one.
\begin{lemma}
\label{lem:sampled-lfa-hb-product-growth-finite}
Suppose that $\rho(\mathcal A_\alpha^{\rm HBCLQL})<1$ for a fixed $\alpha>0$. Fix $\varepsilon>0$ such that
\[
  \beta_\varepsilon:=\rho(\mathcal A_\alpha^{\rm HBCLQL})+\varepsilon<1,
\]
and let $K_{\beta_\varepsilon,\alpha}$ be defined by~\Cref{eq:sampled-lfa-hb-Kbeta-def}. Then $K_{\beta_\varepsilon,\alpha}<\infty$.
\end{lemma}
\begin{proof}
By the convex-hull invariance of the JSR in~\Cref{lem:convex-hull-jsr},
\[
  \rho(\co(\mathcal A_\alpha^{\rm HBCLQL}))
  =\rho(\mathcal A_\alpha^{\rm HBCLQL}) .
\]
Since $\beta_\varepsilon>\rho(\co(\mathcal A_\alpha^{\rm HBCLQL}))$, the definition of the JSR gives constants $C<\infty$ and $\eta\in(\rho(\co(\mathcal A_\alpha^{\rm HBCLQL})),\beta_\varepsilon)$ such that every length-$\ell$ product generated by $\co(\mathcal A_\alpha^{\rm HBCLQL})$ has Euclidean norm at most $C\eta^\ell$. Hence
\[
  \beta_\varepsilon^{-\ell}
  \sup_{M_0,\ldots,M_{\ell-1}\in\co(\mathcal A_\alpha^{\rm HBCLQL})}
  \left\|
  M_{\ell-1}\cdots M_0
  \right\|_2
  \le
  C\left(\frac{\eta}{\beta_\varepsilon}\right)^\ell
  \le C
\]
for all $\ell\ge0$. Taking the supremum over $\ell$ proves the claim.
\end{proof}

The next two lemmas make the product-growth constants $K_{\beta,\alpha}$ in~\Cref{eq:sampled-lfa-hb-Kbeta-def} uniform for small step-sizes $\alpha$ under the CLQL stability condition.
\begin{assumption}
\label{ass:sampled-lfa-hb-fixed-data-alpha}
The MDP transition kernel, reward function, discount factor, feature matrix, sampling distribution, correction direction, correction weight, and momentum gain are fixed independently of the step-size. Equivalently, $P$, $R$, $\gamma$, $\Phi$, $D$, $c$, $\delta$, and $\eta$ do not depend on $\alpha$.
\end{assumption}

Under~\Cref{ass:sampled-lfa-hb-fixed-data-alpha}, the dependence of each deterministic mode on the step-size is affine.  For each deterministic policy \(\pi\), define
\[
  B_\pi:=G(\gamma P\Pi^\pi-I)\Phi,
\]
which is independent of \(\alpha\).  Then the deterministic CLQL mode can be written as
\begin{equation}
  A_\pi^{\rm CLQL}
  =I_m+\alpha B_\pi .
  \label{eq:sampled-lfa-hb-fixed-data-clql-mode}
\end{equation}
Therefore the augmented heavy-ball mode has the affine form
\begin{equation}
\begin{aligned}
  A_\pi^{\rm HBCLQL}
  &=
  \begin{bmatrix}
    I_m+\alpha B_\pi+\alpha\eta I_m&-\alpha\eta I_m\\
    I_m&0
  \end{bmatrix} \\
  &=
  P_0
  +
  \alpha
  \begin{bmatrix}
    B_\pi+\eta I_m&-\eta I_m\\
    0&0
  \end{bmatrix},
\end{aligned}
  \label{eq:sampled-lfa-hb-fixed-data-mode}
\end{equation}
where
\[
  P_0:=
  \begin{bmatrix}
    I_m&0\\
    I_m&0
  \end{bmatrix}
\]
is the common zero-step-size projection.

\begin{assumption}
\label{ass:sampled-lfa-hb-clql-stability}
There exists \(\bar\alpha>0\) such that the underlying CLQL family at that step-size satisfies
\begin{equation}
  \rho(\mathcal A_{\bar\alpha}^{\rm CLQL})<1,
  \qquad
  \mathcal A_{\bar\alpha}^{\rm CLQL}
  :=\{I_m+\bar\alpha B_\pi:\pi\in\Theta\}.
  \label{eq:sampled-lfa-hb-clql-jsr-condition}
\end{equation}
\end{assumption}

The next lemma replaces the compatible-norm assumption used by the constant-momentum form. Because the zero-step-size matrix is now the projection \(P_0\), a suitable norm can be built directly from a JSR Lyapunov norm for the underlying CLQL family.
\begin{lemma}
\label{lem:sampled-lfa-hb-fixed-data-interpolation}
Suppose that~\Cref{ass:sampled-lfa-hb-fixed-data-alpha,ass:sampled-lfa-hb-clql-stability} hold. Then there exist a norm \(p\) on \(\R^{2m}\), constants \(c>0\) and \(\bar\alpha'\in(0,\bar\alpha]\), all independent of the step-size \(\alpha\), such that for every \(0<\alpha\le\bar\alpha'\),
\begin{equation}
  p(Mx)
  \le
  (1-c\alpha)p(x),
  \qquad
  \forall x\in\R^{2m},
  \quad
  \forall M\in\co(\mathcal A_\alpha^{\rm HBCLQL}).
  \label{eq:sampled-lfa-hb-fixed-data-interpolation}
\end{equation}
\end{lemma}
\begin{proof}
Let
\[
  \mathcal B:=\co\{B_\pi:\pi\in\Theta\}.
\]
By~\Cref{ass:sampled-lfa-hb-clql-stability}, the finite family
\(\mathcal A_{\bar\alpha}^{\rm CLQL}\) has JSR strictly less than one. Choose
\(\varepsilon>0\) such that
\[
  \bar\beta:=\rho(\mathcal A_{\bar\alpha}^{\rm CLQL})+\varepsilon<1.
\]
Let \(q\) be the JSR Lyapunov norm from~\Cref{lem:common_lyapunov_construction} applied to \(\mathcal A_{\bar\alpha}^{\rm CLQL}\). By convexity of \(q\),
\begin{equation}
  q((I_m+\bar\alpha B)z)
  \le
  \bar\beta q(z),
  \qquad
  \forall z\in\R^m,
  \quad
  \forall B\in\mathcal B .
  \label{eq:sampled-lfa-hb-q-contraction-baralpha}
\end{equation}
Consequently, for \(0<\alpha\le\bar\alpha\),
\begin{equation}
  q((I_m+\alpha B)z)
  \le
  (1-a\alpha)q(z),
  \qquad
  a:=\frac{1-\bar\beta}{\bar\alpha}>0,
  \label{eq:sampled-lfa-hb-q-contraction-alpha}
\end{equation}
again for every \(B\in\mathcal B\). Indeed, set \(t:=\alpha/\bar\alpha\in(0,1]\). Then
\begin{equation*}
  I_m+\alpha B
  =(1-t)I_m+t(I_m+\bar\alpha B),
\end{equation*}
and hence, for every \(z\in\R^m\),
\begin{align*}
  q((I_m+\alpha B)z)
  &=q\bigl((1-t)z+t(I_m+\bar\alpha B)z\bigr)\\
  &\le (1-t)q(z)+tq((I_m+\bar\alpha B)z)\\
  &\le (1-t)q(z)+t\bar\beta q(z)\\
  &=\bigl(1-t(1-\bar\beta)\bigr)q(z)\\
  &=\left(1-\frac{1-\bar\beta}{\bar\alpha}\alpha\right)q(z),
\end{align*}
which is exactly~\Cref{eq:sampled-lfa-hb-q-contraction-alpha}.

Let
\[
  L_q:=\sup_{B\in\mathcal B}\|B\|_q<\infty,
\]
where \(\|\cdot\|_q\) is the operator norm induced by \(q\). Choose \(\omega>0\) such that \(\omega L_q\le a/2\), and define, for \(x=(u,v)\in\R^m\times\R^m\),
\[
  r:=u,
  \qquad
  s:=u-v,
  \qquad
  p(u,v):=q(r)+\omega q(s).
\]
This is a norm on \(\R^{2m}\). Let \(M\in\co(\mathcal A_\alpha^{\rm HBCLQL})\). Then \(M\) has the same form as~\Cref{eq:sampled-lfa-hb-fixed-data-mode} with some \(B\in\mathcal B\). The affine step-size dependence gives the interpolation identity, now used inside the proof: if \(t:=\alpha/\bar\alpha\), then
\begin{equation*}
  \begin{bmatrix}
    I_m+\alpha B+\alpha\eta I_m&-\alpha\eta I_m\\
    I_m&0
  \end{bmatrix}
  =(1-t)P_0
  +t
  \begin{bmatrix}
    I_m+\bar\alpha B+\bar\alpha\eta I_m&-\bar\alpha\eta I_m\\
    I_m&0
  \end{bmatrix}.
\end{equation*}
If \((u^+,v^+)=M(u,v)\) and \(r^+=u^+\), \(s^+=u^+-v^+\), then
\begin{equation*}
  r^+=(I_m+\alpha B)r+\alpha\eta s,
  \qquad
  s^+=\alpha Br+\alpha\eta s.
\end{equation*}
Using~\Cref{eq:sampled-lfa-hb-q-contraction-alpha}, we obtain
\begin{align*}
  q(r^+)&\le (1-a\alpha)q(r)+\alpha\eta q(s),\\
  q(s^+)&\le \alpha L_q q(r)+\alpha\eta q(s).
\end{align*}
Therefore
\begin{equation*}
  p(u^+,v^+)
  \le
  \{1-(a-\omega L_q)\alpha\}q(r)
  +\alpha\eta(1+\omega)q(s).
\end{equation*}
Set \(c:=a/4\). Decrease \(\bar\alpha\), if necessary, to a number \(\bar\alpha'\in(0,\bar\alpha]\) such that
\[
  c\bar\alpha'\le\frac12,
  \qquad
  \bar\alpha'\eta(1+\omega)\le\frac{\omega}{2},
\]
with the second inequality interpreted as automatic when \(\eta=0\). Then, for every \(0<\alpha\le\bar\alpha'\),
\[
  1-(a-\omega L_q)\alpha\le 1-c\alpha,
  \qquad
  \alpha\eta(1+\omega)\le \omega(1-c\alpha).
\]
Combining these two inequalities gives
\[
  p(Mx)\le (1-c\alpha)p(x),
\]
which proves~\Cref{eq:sampled-lfa-hb-fixed-data-interpolation}.
\end{proof}

\begin{lemma}
\label{lem:sampled-lfa-hb-product-bound}
Suppose that~\Cref{ass:sampled-lfa-hb-fixed-data-alpha,ass:sampled-lfa-hb-clql-stability} hold. Then there exist constants $c>0$, $K<\infty$, and $\bar\alpha'>0$, independent of $\alpha$, with $c\bar\alpha'<1$, such that for every $0<\alpha\le\bar\alpha'$,
\[
  \rho(\mathcal A_\alpha^{\rm HBCLQL})
  \le
  1-c\alpha
  <1,
\]
and for every $\ell\ge0$,
\begin{equation}
  \sup_{M_0,\ldots,M_{\ell-1}\in\co(\mathcal A_\alpha^{\rm HBCLQL})}
  \left\|
  M_{\ell-1}\cdots M_0
  \right\|_2
  \le K(1-c\alpha)^\ell .
  \label{eq:sampled-lfa-hb-product-bound}
\end{equation}
\end{lemma}
\begin{proof}
Let \(p\), \(c\), and \(\bar\alpha'\) be the norm and constants from~\Cref{lem:sampled-lfa-hb-fixed-data-interpolation}.  Decrease \(c\) slightly, if necessary, so that \(c\bar\alpha'<1\) while preserving~\Cref{eq:sampled-lfa-hb-fixed-data-interpolation}.  Fix \(0<\alpha\le\bar\alpha'\), \(\ell\ge0\), and matrices
\[
  M_0,\ldots,M_{\ell-1}
  \in
  \co(\mathcal A_\alpha^{\rm HBCLQL}).
\]
For \(\ell=0\), the product below is the identity.  For \(\ell\ge1\), repeated application of~\Cref{eq:sampled-lfa-hb-fixed-data-interpolation} gives
\begin{equation*}
\begin{aligned}
  p(M_{\ell-1}\cdots M_0x)
  &\le (1-c\alpha)p(M_{\ell-2}\cdots M_0x) \\
  &\le (1-c\alpha)^2p(M_{\ell-3}\cdots M_0x) \\
  &\le \cdots \le (1-c\alpha)^\ell p(x),
  \qquad \forall x\in\R^{2m}.
\end{aligned}
\end{equation*}
Therefore
\begin{equation*}
\begin{aligned}
  \sup_{M_0,\ldots,M_{\ell-1}\in\co(\mathcal A_\alpha^{\rm HBCLQL})}
  \sup_{p(x)=1}
  p(M_{\ell-1}\cdots M_0x)
  \le
  (1-c\alpha)^\ell .
\end{aligned}
\end{equation*}
Taking \(\ell\)-th roots and passing to the limit in the JSR definition yields
\begin{equation*}
\begin{aligned}
  \rho(\co(\mathcal A_\alpha^{\rm HBCLQL}))
  &\le
  1-c\alpha .
\end{aligned}
\end{equation*}
The convex-hull invariance in~\Cref{lem:convex-hull-jsr} then gives
\begin{equation*}
\begin{aligned}
  \rho(\mathcal A_\alpha^{\rm HBCLQL})
  &=\rho(\co(\mathcal A_\alpha^{\rm HBCLQL})) \\
  &\le 1-c\alpha .
\end{aligned}
\end{equation*}
Since \(c\bar\alpha'<1\), we have \(1-c\alpha>0\) for every \(0<\alpha\le\bar\alpha'\), and the JSR bound is strict.

It remains to convert the \(p\)-norm product estimate into the Euclidean norm.  Since all norms on \(\R^{2m}\) are equivalent, there are constants \(a,b<\infty\), independent of \(\alpha\), such that
\begin{equation*}
  \|x\|_2\le a p(x),
  \qquad
  p(x)\le b\|x\|_2,
  \qquad
  \forall x\in\R^{2m}.
\end{equation*}
Consequently, for every product generated by \(\co(\mathcal A_\alpha^{\rm HBCLQL})\),
\begin{equation*}
\begin{aligned}
  \|M_{\ell-1}\cdots M_0x\|_2
  &\le a p(M_{\ell-1}\cdots M_0x) \\
  &\le a(1-c\alpha)^\ell p(x) \\
  &\le ab(1-c\alpha)^\ell\|x\|_2 .
\end{aligned}
\end{equation*}
Taking the supremum over \(\|x\|_2=1\) and setting \(K:=ab\) proves~\Cref{eq:sampled-lfa-hb-product-bound}.
\end{proof}

\begin{lemma}
\label{lem:sampled-lfa-hb-mode-difference}
Let
\begin{equation}
  L
  :=\sup_{\mu,\nu}
  \left\|G P(\Pi^\mu-\Pi^\nu)\Phi\right\|_2,
  \label{eq:sampled-lfa-hb-L-def}
\end{equation}
where the supremum is over all stochastic policies.  Then $L<\infty$ and, for
any stochastic policies $\mu$ and $\nu$,
\begin{equation}
  \|A_\mu^{\rm HBCLQL}-A_\nu^{\rm HBCLQL}\|_2
  \le \alpha\gamma L .
  \label{eq:sampled-lfa-hb-mode-difference}
\end{equation}
\end{lemma}
\begin{proof}
By \Cref{eq:lfa-cpd-hb-mode}, the two modes differ only through the Bellman
policy matrix in the upper-left block:
\begin{equation*}
\begin{aligned}
  A_\mu^{\rm HBCLQL}-A_\nu^{\rm HBCLQL}
  &=
  \begin{bmatrix}
    A_\mu^{\rm CLQL}-A_\nu^{\rm CLQL}&0\\
    0&0
  \end{bmatrix} \\
  &=
  \begin{bmatrix}
    \alpha\gamma G P(\Pi^\mu-\Pi^\nu)\Phi&0\\
    0&0
  \end{bmatrix}.
\end{aligned}
\end{equation*}
For any $u,v\in\R^m$,
\begin{equation*}
  \left\|
  \begin{bmatrix}
    \alpha\gamma G P(\Pi^\mu-\Pi^\nu)\Phi&0\\
    0&0
  \end{bmatrix}
  \begin{bmatrix}u\\v\end{bmatrix}
  \right\|_2
  =
  \alpha\gamma\left\|G P(\Pi^\mu-\Pi^\nu)\Phi u\right\|_2 .
\end{equation*}
Hence
\begin{equation*}
\begin{aligned}
  \|A_\mu^{\rm HBCLQL}-A_\nu^{\rm HBCLQL}\|_2
  &=
  \alpha\gamma
  \sup_{(u,v)\ne0}
  \frac{\left\|G P(\Pi^\mu-\Pi^\nu)\Phi u\right\|_2}
       {(\|u\|_2^2+\|v\|_2^2)^{1/2}} \\
  &\le
  \alpha\gamma\left\|G P(\Pi^\mu-\Pi^\nu)\Phi\right\|_2
  \le \alpha\gamma L .
\end{aligned}
\end{equation*}
The supremum in \Cref{eq:sampled-lfa-hb-L-def} is finite because the stochastic
policy set is a finite product of probability simplices and is therefore
compact, while the map
$(\mu,\nu)\mapsto \|G P(\Pi^\mu-\Pi^\nu)\Phi\|_2$ is continuous.
\end{proof}

\begin{lemma}
\label{lem:sampled-lfa-hb-noise-growth}
There exist finite constants $\sigma_0,\sigma_1\ge0$, depending only on the MDP,
$\Phi$, $\delta$, $c$, and $\theta^\star$, such that the noise term in
\Cref{eq:sampled-lfa-hb-noise} satisfies
\begin{equation}
  \E[\|w_k\|_2^2\mid\mathcal F_k]
  \le \sigma_0^2+\sigma_1^2\|x_k\|_2^2,
  \qquad k\ge0 .
  \label{eq:sampled-lfa-hb-noise-growth}
\end{equation}
\end{lemma}
\begin{proof}
Because the state and action spaces are finite, the three maxima
\begin{equation*}
  \max_{i,j}\|\Phi^\top e_i+\delta c-\Phi^\top e_j\|_2,
  \qquad
  \max_i\|\Phi^\top e_i\|_2,
  \qquad
  \max_{s,a,s'}|r(s,a,s')|
\end{equation*}
are finite.  For every realization of the sample at time $k$,
\begin{equation*}
\begin{aligned}
  |\Delta_k|
  &\le
  |r(s_k,a_k,s'_k)|
  +\gamma\max_{b\in\mathcal A}|(\Phi\theta_k)(s'_k,b)|
  +|(\Phi\theta_k)(s_k,a_k)| \\
  &\le
  \max_{s,a,s'}|r(s,a,s')|
  +\gamma\max_i\|\Phi^\top e_i\|_2\|\theta_k\|_2
  +\max_i\|\Phi^\top e_i\|_2\|\theta_k\|_2 \\
  &=
  \max_{s,a,s'}|r(s,a,s')|
  +(1+\gamma)\max_i\|\Phi^\top e_i\|_2\|\theta_k\|_2 .
\end{aligned}
\end{equation*}
Moreover,
\begin{equation*}
  \|\Phi^\top e_{i_k}+\delta c-\Phi^\top e_{j_k}\|_2
  \le
  \max_{i,j}\|\Phi^\top e_i+\delta c-\Phi^\top e_j\|_2 .
\end{equation*}
Therefore, by the definition of $\widehat g_k$ in
\Cref{eq:sampled-lfa-hb-sample-map},
\begin{equation*}
\begin{aligned}
  \|\widehat g_k(\theta_k)\|_2
  &\le
  \left(\max_{i,j}\|\Phi^\top e_i+\delta c-\Phi^\top e_j\|_2\right) \\
  &\quad\times
  \left(
  \max_{s,a,s'}|r(s,a,s')|
  +(1+\gamma)\max_i\|\Phi^\top e_i\|_2\|\theta_k\|_2
  \right).
\end{aligned}
\end{equation*}
Since
\begin{equation*}
  G\{F(\Phi\theta_k)-\Phi\theta_k\}
  =\E[\widehat g_k(\theta_k)\mid\mathcal F_k],
\end{equation*}
Jensen's inequality gives
\begin{equation*}
\begin{aligned}
  \|G\{F(\Phi\theta_k)-\Phi\theta_k\}\|_2
  &\le
  \E[\|\widehat g_k(\theta_k)\|_2\mid\mathcal F_k] \\
  &\le
  \left(\max_{i,j}\|\Phi^\top e_i+\delta c-\Phi^\top e_j\|_2\right) \\
  &\quad\times
  \left(
  \max_{s,a,s'}|r(s,a,s')|
  +(1+\gamma)\max_i\|\Phi^\top e_i\|_2\|\theta_k\|_2
  \right).
\end{aligned}
\end{equation*}
Combining the last two bounds with the definition of $w_k$ yields
\begin{equation*}
\begin{aligned}
  \|w_k\|_2
  &\le
  2\left(\max_{i,j}\|\Phi^\top e_i+\delta c-\Phi^\top e_j\|_2\right) \\
  &\quad\times
  \left(
  \max_{s,a,s'}|r(s,a,s')|
  +(1+\gamma)\max_i\|\Phi^\top e_i\|_2\|\theta_k\|_2
  \right).
\end{aligned}
\end{equation*}
Also,
\begin{equation*}
  \|\theta_k\|_2
  \le
  \|\theta^\star\|_2+\|\theta_k-\theta^\star\|_2
  \le
  \|\theta^\star\|_2+\|x_k\|_2 .
\end{equation*}
Substituting this bound gives
\begin{equation*}
\begin{aligned}
  \|w_k\|_2
  &\le
  2\left(\max_{i,j}\|\Phi^\top e_i+\delta c-\Phi^\top e_j\|_2\right) \\
  &\quad\times
  \left(
  \max_{s,a,s'}|r(s,a,s')|
  +(1+\gamma)\max_i\|\Phi^\top e_i\|_2\|\theta^\star\|_2
  \right) \\
  &\quad+
  2\left(\max_{i,j}\|\Phi^\top e_i+\delta c-\Phi^\top e_j\|_2\right)
  (1+\gamma)\max_i\|\Phi^\top e_i\|_2\|x_k\|_2 .
\end{aligned}
\end{equation*}
Squaring the preceding inequality and using $(u+v)^2\le2u^2+2v^2$ proves
\Cref{eq:sampled-lfa-hb-noise-growth}; for instance, one may take
\begin{align*}
  \sigma_0^2
  &=
  8\left(\max_{i,j}\|\Phi^\top e_i+\delta c-\Phi^\top e_j\|_2\right)^2
  \left(
  \max_{s,a,s'}|r(s,a,s')|
  +(1+\gamma)\max_i\|\Phi^\top e_i\|_2\|\theta^\star\|_2
  \right)^2,\\
  \sigma_1^2
  &=
  8\left(\max_{i,j}\|\Phi^\top e_i+\delta c-\Phi^\top e_j\|_2\right)^2
  (1+\gamma)^2
  \left(\max_i\|\Phi^\top e_i\|_2\right)^2 .
\end{align*}
Taking the conditional expectation preserves this deterministic upper bound
because $x_k$ is $\mathcal F_k$-measurable.
\end{proof}
We next show that, for sufficiently small step-sizes, the iterates of the stochastic HBCLQL error recursion have a uniform second-moment bound.
\begin{lemma}
\label{lem:sampled-lfa-hb-ms-bound}
Suppose that the hypothesis of \Cref{lem:sampled-lfa-hb-product-bound} holds,
and let $c$, $K$, and $\bar\alpha'$ be the constants from
\Cref{lem:sampled-lfa-hb-product-bound}.  Let $L$, $\sigma_0$, and $\sigma_1$ be
the constants from
\Cref{lem:sampled-lfa-hb-mode-difference,lem:sampled-lfa-hb-noise-growth}.
Define
\begin{equation*}
  C_0
  :=1+4K^2
  +4\left(\frac{\gamma L K^2}{c}\right)^2
\end{equation*}
and
\begin{equation*}
  C_1
  :=4\left(
  \frac{K^2}{c}
  +\frac{\gamma^2L^2K^4}{c^3}
  \right).
\end{equation*}
Set
\begin{equation*}
  \alpha_0
  :=\min\left\{
  \bar\alpha',
  \frac{1}{2C_1\sigma_1^2}
  \right\},
\end{equation*}
where the second term is interpreted as $+\infty$ when $\sigma_1=0$.  Then, for every $0<\alpha\le\alpha_0$ and every $k\ge0$, the recursion in
\Cref{eq:sampled-lfa-hb-error-system} satisfies
\begin{equation}
  \sup_{0\le t\le k}\E[\|x_t\|_2^2]
  \le
  2C_0\|x_0\|_2^2+2C_1\alpha\sigma_0^2 .
  \label{eq:sampled-lfa-hb-ms-bound}
\end{equation}
\end{lemma}
\begin{proof}
Fix a stochastic policy $\bar\mu$.  Let
\begin{equation*}
  y_{k+1}=A_{\bar\mu}^{\rm HBCLQL}y_k+\alpha\xi_k,
  \qquad y_0=x_0,
\end{equation*}
and decompose $y_k=\bar y_k+\zeta_k$, where
\begin{equation*}
  \bar y_{k+1}=A_{\bar\mu}^{\rm HBCLQL}\bar y_k,
  \quad
  \bar y_0=x_0,
  \qquad
  \zeta_{k+1}=A_{\bar\mu}^{\rm HBCLQL}\zeta_k+\alpha\xi_k,
  \quad
  \zeta_0=0.
\end{equation*}
Write $x_k-y_k=u_k+v_k$, where
\begin{align*}
  u_{k+1}
  &=A_{\mu_k}^{\rm HBCLQL}u_k
    +(A_{\mu_k}^{\rm HBCLQL}-A_{\bar\mu}^{\rm HBCLQL})\bar y_k,
  &u_0&=0,\\
  v_{k+1}
  &=A_{\mu_k}^{\rm HBCLQL}v_k
    +(A_{\mu_k}^{\rm HBCLQL}-A_{\bar\mu}^{\rm HBCLQL})\zeta_k,
  &v_0&=0.
\end{align*}
Then $x_k=\bar y_k+\zeta_k+u_k+v_k$.  By
\Cref{eq:sampled-lfa-hb-product-bound},
\begin{equation*}
  \|\bar y_k\|_2
  \le K(1-c\alpha)^k\|x_0\|_2
  \le K\|x_0\|_2.
\end{equation*}
The stochastic part of the fixed-policy filter has the explicit form
\begin{equation*}
  \zeta_k
  =
  \alpha\sum_{t=0}^{k-1}
  (A_{\bar\mu}^{\rm HBCLQL})^{k-1-t}\xi_t .
\end{equation*}
Expanding the square gives
\begin{equation*}
\begin{aligned}
\E[\|\zeta_k\|_2^2]
&=
\alpha^2
\sum_{s=0}^{k-1}\sum_{t=0}^{k-1}
\E\!\left[
\xi_s^\top
\{(A_{\bar\mu}^{\rm HBCLQL})^{k-1-s}\}^\top
(A_{\bar\mu}^{\rm HBCLQL})^{k-1-t}\xi_t
\right].
\end{aligned}
\end{equation*}
If $s<t$, then the factor multiplying $\xi_t$ is $\mathcal F_t$-measurable, so
\begin{equation*}
\begin{aligned}
&\E\!\left[
\xi_s^\top
\{(A_{\bar\mu}^{\rm HBCLQL})^{k-1-s}\}^\top
(A_{\bar\mu}^{\rm HBCLQL})^{k-1-t}\xi_t
\right] \\
&\qquad=
\E\!\left[
\xi_s^\top
\{(A_{\bar\mu}^{\rm HBCLQL})^{k-1-s}\}^\top
(A_{\bar\mu}^{\rm HBCLQL})^{k-1-t}
\E[\xi_t\mid\mathcal F_t]
\right]
=0 .
\end{aligned}
\end{equation*}
The case $t<s$ is the same after conditioning on $\mathcal F_s$.  Thus only the
diagonal terms remain.  With
\begin{equation*}
  M_k:=\sup_{0\le t\le k}\E[\|x_t\|_2^2],
\end{equation*}
\Cref{eq:sampled-lfa-hb-product-bound,eq:sampled-lfa-hb-noise-growth} imply, for
$k\ge1$,
\begin{equation*}
\begin{aligned}
\E[\|\zeta_k\|_2^2]
&\le
\alpha^2K^2
\sum_{t=0}^{k-1}(1-c\alpha)^{2(k-1-t)}
\E[\|\xi_t\|_2^2] \\
&\le
\alpha^2K^2
\sum_{t=0}^{k-1}(1-c\alpha)^{2(k-1-t)}
\{\sigma_0^2+\sigma_1^2M_{k-1}\} \\
&\le
\frac{\alpha K^2}{c}
\{\sigma_0^2+\sigma_1^2M_{k-1}\},
\end{aligned}
\end{equation*}
where the last step uses
\begin{equation*}
  \sum_{j=0}^{k-1}(1-c\alpha)^{2j}
  \le
  \frac{1}{1-(1-c\alpha)^2}
  \le
  \frac{1}{c\alpha},
  \qquad 0<c\alpha<1 .
\end{equation*}
Next, unrolling the recursion for $u_k$ gives
\begin{equation*}
  u_k=
  \sum_{t=0}^{k-1}
  A_{\mu_{k-1}}^{\rm HBCLQL}\cdots A_{\mu_{t+1}}^{\rm HBCLQL}
  (A_{\mu_t}^{\rm HBCLQL}-A_{\bar\mu}^{\rm HBCLQL})\bar y_t,
\end{equation*}
with the empty product interpreted as the identity.  Therefore,
\begin{equation*}
\begin{aligned}
  \|u_k\|_2
  &\le
  \sum_{t=0}^{k-1}
  K(1-c\alpha)^{k-1-t}\alpha\gamma L
  K(1-c\alpha)^t\|x_0\|_2 \\
  &=
  \alpha\gamma L K^2 k(1-c\alpha)^{k-1}\|x_0\|_2 \\
  &\le
  \frac{\gamma L K^2}{c}\|x_0\|_2,
\end{aligned}
\end{equation*}
where the last inequality follows from
$kq^{k-1}\le(1-q)^{-1}$ with $q=1-c\alpha$.  Similarly,
\begin{equation*}
  v_k=
  \sum_{t=0}^{k-1}
  A_{\mu_{k-1}}^{\rm HBCLQL}\cdots A_{\mu_{t+1}}^{\rm HBCLQL}
  (A_{\mu_t}^{\rm HBCLQL}-A_{\bar\mu}^{\rm HBCLQL})\zeta_t.
\end{equation*}
Thus
\begin{equation*}
  \|v_k\|_2
  \le
  \alpha\gamma L K
  \sum_{t=0}^{k-1}(1-c\alpha)^{k-1-t}\|\zeta_t\|_2 .
\end{equation*}
Using the weighted Cauchy inequality,
\begin{equation*}
  \left(\sum_{t=0}^{k-1}q_t\|\zeta_t\|_2\right)^2
  \le
  \left(\sum_{t=0}^{k-1}q_t\right)
  \left(\sum_{t=0}^{k-1}q_t\|\zeta_t\|_2^2\right),
  \qquad q_t:=(1-c\alpha)^{k-1-t},
\end{equation*}
and the preceding bound on $\E[\|\zeta_t\|_2^2]$, we get
\begin{equation*}
\begin{aligned}
\E[\|v_k\|_2^2]
&\le
\alpha^2\gamma^2L^2K^2
\left(\sum_{t=0}^{k-1}q_t\right)
\sum_{t=0}^{k-1}q_t\E[\|\zeta_t\|_2^2] \\
&\le
\alpha^2\gamma^2L^2K^2
\left(\frac{1}{c\alpha}\right)^2
\frac{\alpha K^2}{c}
\{\sigma_0^2+\sigma_1^2M_{k-1}\} \\
&=
\frac{\alpha\gamma^2L^2K^4}{c^3}
\{\sigma_0^2+\sigma_1^2M_{k-1}\}.
\end{aligned}
\end{equation*}
Using
$\|z_1+z_2+z_3+z_4\|_2^2\le4\sum_{j=1}^4\|z_j\|_2^2$ with
$z_1=\bar y_k$, $z_2=\zeta_k$, $z_3=u_k$, and $z_4=v_k$, and then taking the
supremum over $0\le t\le k$, gives
\begin{equation*}
  M_k
  \le
  C_0\|x_0\|_2^2+C_1\alpha\{\sigma_0^2+\sigma_1^2M_{k-1}\}.
\end{equation*}
For $0<\alpha\le\alpha_0$, one has $C_1\alpha\sigma_1^2\le1/2$ and
$M_{k-1}\le M_k$.  Hence
\begin{equation*}
  M_k
  \le
  C_0\|x_0\|_2^2+C_1\alpha\sigma_0^2+\frac12 M_k,
\end{equation*}
which implies \Cref{eq:sampled-lfa-hb-ms-bound}.  For $k=0$, the same bound
holds because $M_0=\|x_0\|_2^2$ and $C_0\ge1$.
\end{proof}
We now combine the preceding estimates to obtain a finite-time error bound for the stochastic HBCLQL recursion.
\begin{theorem}
\label{thm:sampled-lfa-hb-small-alpha-error}
Assume \Cref{ass:lfa-cpd-unique-fixed-point}.  Suppose that \Cref{ass:sampled-lfa-hb-fixed-data-alpha,ass:sampled-lfa-hb-clql-stability} hold, and let $c$ and $K$ be the
constants from \Cref{lem:sampled-lfa-hb-product-bound}. Let $\theta^\star$ be
the unique fixed point of the map in \Cref{eq:lfa-cpd-standard-map}, and define
\begin{equation*}
  x_0:=
  \begin{bmatrix}
    \theta_0-\theta^\star\\
    \theta_{-1}-\theta^\star
  \end{bmatrix}.
\end{equation*}
Let $L$, $\sigma_0$, $\sigma_1$, $C_0$, $C_1$, and $\alpha_0$ be defined as in
\Cref{lem:sampled-lfa-hb-mode-difference,lem:sampled-lfa-hb-noise-growth,lem:sampled-lfa-hb-ms-bound}. Then, for every $0<\alpha\le\alpha_0$ and every $k\ge0$, the stochastic HBCLQL recursion in
\Cref{alg:sampled-lfa-hb-update} satisfies
\begin{equation}
\begin{aligned}
\E[\|\theta_k-\theta^\star\|_2]
\le\;&
K\left(1-\frac{c}{2}\alpha\right)^k
\|x_0\|_2 \\
&+
\alpha\gamma K^2 k
\left(1-\frac{c}{2}\alpha\right)^{k-1}
L\|x_0\|_2 \\
&+
\left(1+\frac{2\gamma K L}{c}\right)
K\sqrt{\frac{2\alpha}{c}} \\
&\qquad\times
\left(
\sigma_0^2+
\sigma_1^2\left(
2C_0\|x_0\|_2^2+2C_1\alpha\sigma_0^2
\right)
\right)^{1/2},
\end{aligned}
\label{eq:sampled-lfa-hb-small-alpha-error}
\end{equation}
with the convention that
$k(1-c\alpha/2)^{k-1}=0$ when $k=0$.  In particular, for fixed
$\theta_{-1}$ and $\theta_0$, the last term in
\Cref{eq:sampled-lfa-hb-small-alpha-error} is $O(\sqrt\alpha)$ as
$\alpha\downarrow0$.
\end{theorem}
\begin{proof}
By \Cref{lem:sampled-lfa-hb-error-system}, the augmented error satisfies
\begin{equation*}
  x_{k+1}=A_{\mu_k}^{\rm HBCLQL}x_k+\alpha\xi_k,
  \qquad
  \E[\xi_k\mid\mathcal F_k]=0 .
\end{equation*}
Every stochastic-policy mode \(A_{\mu_k}^{\rm HBCLQL}\) belongs to \(\co(\mathcal A_\alpha^{\rm HBCLQL})\).  Set
\begin{equation*}
  \beta_\alpha:=1-\frac{c}{2}\alpha .
\end{equation*}
Since \(1-c\alpha<\beta_\alpha<1\), \Cref{eq:sampled-lfa-hb-product-bound} implies that every product of length \(\ell\) from the convexified family satisfies
\begin{equation*}
  \|M_{\ell-1}\cdots M_0\|_2
  \le
  K(1-c\alpha)^\ell
  \le
  K\beta_\alpha^\ell,
  \qquad \ell\ge0 .
\end{equation*}

Fix an arbitrary stochastic policy \(\bar\mu\).  Introduce the fixed-policy reference filter
\begin{equation*}
  y_{k+1}=A_{\bar\mu}^{\rm HBCLQL}y_k+\alpha\xi_k,
  \qquad
  y_0=x_0 .
\end{equation*}
Write \(y_k=\bar y_k+\zeta_k\), where
\begin{equation*}
  \bar y_{k+1}=A_{\bar\mu}^{\rm HBCLQL}\bar y_k,
  \quad
  \bar y_0=x_0,
  \qquad
  \zeta_{k+1}=A_{\bar\mu}^{\rm HBCLQL}\zeta_k+\alpha\xi_k,
  \quad
  \zeta_0=0 .
\end{equation*}
The product bound gives the deterministic estimate
\begin{equation*}
  \|\bar y_k\|_2
  \le
  K\beta_\alpha^k\|x_0\|_2 .
\end{equation*}

Next define \(r_k:=x_k-y_k\).  Subtracting the reference filter from the actual recursion gives
\begin{equation*}
\begin{aligned}
  r_{k+1}
  &=A_{\mu_k}^{\rm HBCLQL}x_k-A_{\bar\mu}^{\rm HBCLQL}y_k \\
  &=A_{\mu_k}^{\rm HBCLQL}r_k
    +\left(A_{\mu_k}^{\rm HBCLQL}-A_{\bar\mu}^{\rm HBCLQL}\right)y_k,
  \qquad r_0=0 .
\end{aligned}
\end{equation*}
Split \(r_k=u_k+v_k\) according to \(y_k=\bar y_k+\zeta_k\):
\begin{align*}
  u_{k+1}
  &=A_{\mu_k}^{\rm HBCLQL}u_k
    +\left(A_{\mu_k}^{\rm HBCLQL}-A_{\bar\mu}^{\rm HBCLQL}\right)\bar y_k,
  &u_0&=0,\\
  v_{k+1}
  &=A_{\mu_k}^{\rm HBCLQL}v_k
    +\left(A_{\mu_k}^{\rm HBCLQL}-A_{\bar\mu}^{\rm HBCLQL}\right)\zeta_k,
  &v_0&=0 .
\end{align*}
For \(k=0\), the following sums are empty.  For \(k\ge1\), unrolling the recursion for \(u_k\) gives
\begin{equation*}
  u_k=
  \sum_{t=0}^{k-1}
  A_{\mu_{k-1}}^{\rm HBCLQL}\cdots A_{\mu_{t+1}}^{\rm HBCLQL}
  \left(A_{\mu_t}^{\rm HBCLQL}-A_{\bar\mu}^{\rm HBCLQL}\right)\bar y_t,
\end{equation*}
where the product is the identity when \(t=k-1\).  Using the product bound, \Cref{eq:sampled-lfa-hb-mode-difference}, and the bound on \(\bar y_t\),
\begin{equation*}
\begin{aligned}
  \|u_k\|_2
  &\le
  \sum_{t=0}^{k-1}
  K\beta_\alpha^{k-1-t}\,\alpha\gamma L\,
  K\beta_\alpha^t\|x_0\|_2 \\
  &=
  \alpha\gamma L K^2 k\beta_\alpha^{k-1}\|x_0\|_2 .
\end{aligned}
\end{equation*}

We now bound the martingale part \(\zeta_k\).  Its explicit form is
\begin{equation*}
  \zeta_k
  =
  \alpha\sum_{t=0}^{k-1}
  \left(A_{\bar\mu}^{\rm HBCLQL}\right)^{k-1-t}\xi_t .
\end{equation*}
Expanding the squared norm yields
\begin{equation*}
\begin{aligned}
  \E[\|\zeta_k\|_2^2]
  =\alpha^2
  \sum_{s=0}^{k-1}\sum_{t=0}^{k-1}
  \E\!\left[
    \xi_s^\top
    \left\{\left(A_{\bar\mu}^{\rm HBCLQL}\right)^{k-1-s}\right\}^\top
    \left(A_{\bar\mu}^{\rm HBCLQL}\right)^{k-1-t}
    \xi_t
  \right].
\end{aligned}
\end{equation*}
If \(s<t\), then the factor multiplying \(\xi_t\) is \(\mathcal F_t\)-measurable, and therefore
\begin{equation*}
\begin{aligned}
&\E\!\left[
    \xi_s^\top
    \left\{\left(A_{\bar\mu}^{\rm HBCLQL}\right)^{k-1-s}\right\}^\top
    \left(A_{\bar\mu}^{\rm HBCLQL}\right)^{k-1-t}
    \xi_t
  \right] \\
&\qquad=
  \E\!\left[
    \xi_s^\top
    \left\{\left(A_{\bar\mu}^{\rm HBCLQL}\right)^{k-1-s}\right\}^\top
    \left(A_{\bar\mu}^{\rm HBCLQL}\right)^{k-1-t}
    \E[\xi_t\mid\mathcal F_t]
  \right]
  =0 .
\end{aligned}
\end{equation*}
The case \(t<s\) is identical after conditioning on \(\mathcal F_s\).  Thus only the diagonal terms remain.  Let
\begin{equation*}
  M_\alpha:=2C_0\|x_0\|_2^2+2C_1\alpha\sigma_0^2 .
\end{equation*}
By \Cref{lem:sampled-lfa-hb-noise-growth,lem:sampled-lfa-hb-ms-bound}, for every \(0\le t\le k\),
\begin{equation*}
\begin{aligned}
  \E[\|\xi_t\|_2^2]
  &=\E[\|w_t\|_2^2] \\
  &\le
  \sigma_0^2+
  \sigma_1^2
  \sup_{0\le j\le t}\E[\|x_j\|_2^2] \\
  &\le
  \sigma_0^2+\sigma_1^2M_\alpha .
\end{aligned}
\end{equation*}
Consequently,
\begin{equation*}
\begin{aligned}
  \E[\|\zeta_k\|_2^2]
  &\le
  \alpha^2K^2
  \sum_{t=0}^{k-1}\beta_\alpha^{2(k-1-t)}
  \left(\sigma_0^2+\sigma_1^2M_\alpha\right) \\
  &\le
  \frac{\alpha^2K^2}{1-\beta_\alpha^2}
  \left(\sigma_0^2+\sigma_1^2M_\alpha\right).
\end{aligned}
\end{equation*}
Jensen's inequality gives the first-moment bound
\begin{equation*}
  \E[\|\zeta_k\|_2]
  \le
  \frac{\alpha K}{\sqrt{1-\beta_\alpha^2}}
  \left(\sigma_0^2+\sigma_1^2M_\alpha\right)^{1/2}.
\end{equation*}

The remaining switching correction satisfies
\begin{equation*}
  v_k=
  \sum_{t=0}^{k-1}
  A_{\mu_{k-1}}^{\rm HBCLQL}\cdots A_{\mu_{t+1}}^{\rm HBCLQL}
  \left(A_{\mu_t}^{\rm HBCLQL}-A_{\bar\mu}^{\rm HBCLQL}\right)\zeta_t .
\end{equation*}
Hence, using the product bound, \Cref{eq:sampled-lfa-hb-mode-difference}, and the preceding bound on \(\zeta_t\),
\begin{equation*}
\begin{aligned}
  \E[\|v_k\|_2]
  &\le
  \alpha\gamma L K
  \sum_{t=0}^{k-1}\beta_\alpha^{k-1-t}\E[\|\zeta_t\|_2] \\
  &\le
  \alpha\gamma L K
  \sum_{t=0}^{k-1}\beta_\alpha^{k-1-t}
  \frac{\alpha K}{\sqrt{1-\beta_\alpha^2}}
  \left(\sigma_0^2+\sigma_1^2M_\alpha\right)^{1/2} \\
  &\le
  \frac{\alpha\gamma L K}{1-\beta_\alpha}
  \frac{\alpha K}{\sqrt{1-\beta_\alpha^2}}
  \left(\sigma_0^2+\sigma_1^2M_\alpha\right)^{1/2} .
\end{aligned}
\end{equation*}

Since \(x_k=y_k+r_k=\bar y_k+\zeta_k+u_k+v_k\) and \(\theta_k-\theta^\star\) is the first block of \(x_k\), we have \(\|\theta_k-\theta^\star\|_2\le\|x_k\|_2\).  Combining the four estimates gives
\begin{equation*}
\begin{aligned}
  \E[\|\theta_k-\theta^\star\|_2]
  \le\;&
  K\beta_\alpha^k\|x_0\|_2
  +\alpha\gamma L K^2 k\beta_\alpha^{k-1}\|x_0\|_2 \\
  &+
  \left(
  1+\frac{\alpha\gamma L K}{1-\beta_\alpha}
  \right)
  \frac{\alpha K}{\sqrt{1-\beta_\alpha^2}}
  \left(\sigma_0^2+\sigma_1^2M_\alpha\right)^{1/2}.
\end{aligned}
\end{equation*}
Finally,
\begin{equation*}
  1-\beta_\alpha=\frac{c\alpha}{2},
  \qquad
  1-\beta_\alpha^2
  =(1-\beta_\alpha)(1+\beta_\alpha)
  \ge \frac{c\alpha}{2}.
\end{equation*}
Therefore
\begin{equation*}
  1+\frac{\alpha\gamma L K}{1-\beta_\alpha}
  =
  1+\frac{2\gamma K L}{c},
  \qquad
  \frac{\alpha K}{\sqrt{1-\beta_\alpha^2}}
  \le
  K\sqrt{\frac{2\alpha}{c}} .
\end{equation*}
This proves~\Cref{eq:sampled-lfa-hb-small-alpha-error}.
The last term is \(O(\sqrt\alpha)\) for fixed \(\theta_{-1}\) and \(\theta_0\), because the constants \(c\), \(K\), \(L\), \(\sigma_0\), \(\sigma_1\), \(C_0\), and \(C_1\) are independent of \(\alpha\).
\end{proof}

Thus the sampled stochastic HBCLQL iterates track the fixed point up to a first-moment neighborhood whose radius is of order $\sqrt\alpha$.

\section{Numerical illustration}
\label{sec:sampled_stochastic_feature_example}

Consider again the example in~\Cref{sec:feature_numerical_toy_example}, where the MDP, feature matrix, sampling distribution $d$, constant direction $c$, initial parameter, and fixed point $\theta^\star$ are the same as in that section. We use the constant step-size $\alpha=0.002$, the correction parameter $\delta=2$, and the actual momentum coefficient $\alpha\eta=0.1$ for HBCLQL. Stochastic CLQL is obtained from~\Cref{alg:sampled-lfa-hb-update} by setting $\eta=0$, while stochastic LQL uses~\Cref{alg:sampled-standard-lfa-q-learning}. The three methods are run for $12{,}000$ iterations over $400$ independent sample paths. To make the comparison depend on the update rule rather than sampling variation, the same sampled state-action pairs, centering samples, and next-state draws are used across the methods whenever applicable.

The quantity in~\Cref{fig:toy-sampled-feature-convergence} is the median, over the $400$ paths, of the relative parameter error $e_k^\theta$ defined in~\Cref{sec:feature_numerical_toy_example}. In this run, stochastic HBCLQL reaches median error $10^{-1}$, $5\times10^{-2}$, and $2\times10^{-2}$ after $4074$, $5418$, and $7177$ iterations, respectively. Stochastic CLQL reaches the same thresholds after $4560$, $6036$, and $8202$ iterations, while stochastic LQL reaches the first two thresholds after $6840$ and $9219$ iterations and does not reach $2\times10^{-2}$ within $12{,}000$ iterations. Thus the sampled experiment preserves the early-stage ordering observed in the deterministic example in~\Cref{sec:feature_numerical_toy_example}. At later iterations, stochastic fluctuations produce a small neighborhood around the fixed point, so the final curves need not remain strictly ordered.
\begin{figure}[H]
\centering
\includegraphics[width=0.78\textwidth]{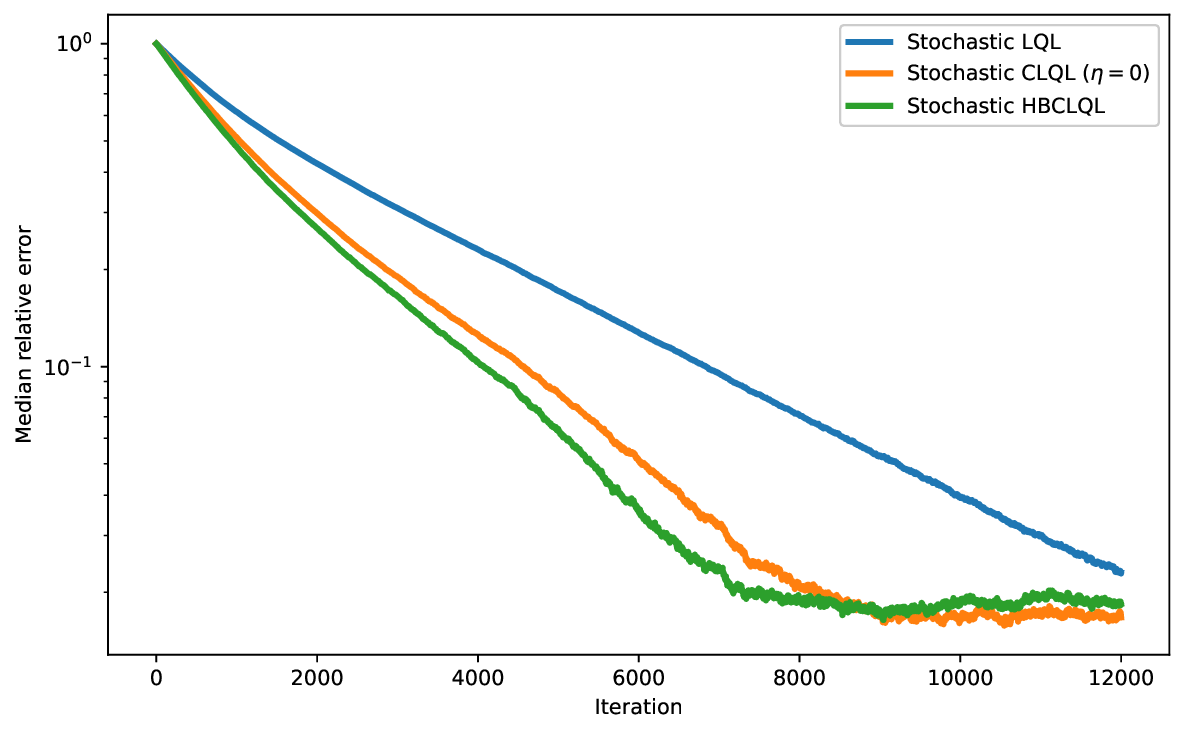}
\caption{Sampled stochastic version of the example in~\Cref{sec:feature_numerical_toy_example}. The curves give the median relative parameter error over $400$ independent paths.}
\label{fig:toy-sampled-feature-convergence}
\end{figure}

\section{Conclusion}
\label{sec:conclusion}

This paper proposed CQL and combined it with heavy-ball momentum, where
$\eta$ is a gain and the actual momentum coefficient is $\alpha\eta$. The analysis
used SLS and JSR tools to separate the dynamics along a common eigenvector direction from the dynamics along the orthogonal directions. Under the stated projected-gap and small-actual-momentum conditions,
HBCQL has a smaller fixed-step-size JSR certificate than CQL, and analogous statements hold for linear function approximation. With the gain fixed and $\alpha\downarrow0$, the certified improvement is second order in $\alpha$.

\clearpage
\appendix
\section{Auxiliary results}
\label{app:auxiliary-results}

\subsection{Nonnegative matrix row-sum bounds}
\label{app:nonnegative-row-sum-bounds}

The following standard row-sum bound is the form of the Collatz--Wielandt
inequality used below; see, for example, \citep[Chapter~8]{meyer2000matrix}.
\begin{lemma}
\label{lem:collatz-wielandt-row-sum}
Let $B\in\R^{m\times m}$ be nonnegative. Then
\begin{equation}
  \min_{1\le i\le m}\sum_{j=1}^m B_{ij}
  \le \rho(B)
  \le \max_{1\le i\le m}\sum_{j=1}^m B_{ij} .
  \label{eq:collatz-wielandt-row-sum}
\end{equation}
\end{lemma}
\begin{proof}
The Collatz--Wielandt inequality for nonnegative matrices gives
\begin{equation*}
  \min_i\frac{(Bx)_i}{x_i}
  \le \rho(B)
  \le \max_i\frac{(Bx)_i}{x_i}
\end{equation*}
for every vector $x$ with strictly positive entries. Taking
$x=\mathbf 1$ gives~\Cref{eq:collatz-wielandt-row-sum}.
\end{proof}

\begin{lemma}
\label{lem:direct-d-nonnegative-row-sum}
Assume
\begin{equation*}
  0<\alpha d(s,a)\le 1
  \quad\text{for every }(s,a)\in\mathcal S\times\mathcal A.
\end{equation*}
Then each matrix $A_\pi^{\rm QL}$ in the direct family
\Cref{eq:direct-d-standard-family} is nonnegative. Moreover, its row indexed by
$(s,a)$ has row sum
\begin{equation}
  1-\alpha d(s,a)(1-\gamma),
  \label{eq:direct-d-row-sum}
\end{equation}
and every deterministic policy $\pi$ satisfies
\begin{equation}
  \rho(A_\pi^{\rm QL})
  \ge 1-\alpha d_{\max}(1-\gamma).
  \label{eq:direct-d-cw-lower-bound}
\end{equation}
\end{lemma}
\begin{proof}
For a deterministic policy $\pi$, the mode in~\Cref{eq:direct-d-standard-mode}
is
\begin{equation*}
  A_\pi^{\rm QL}=I-\alpha D+\alpha\gamma D P\Pi^\pi .
\end{equation*}
The diagonal part $I-\alpha D$ is nonnegative by the assumption
$0<\alpha d(s,a)\le 1$, and $D P\Pi^\pi$ is nonnegative because $D$ is
nonnegative diagonal and $P\Pi^\pi$ is row-stochastic. Hence
$A_\pi^{\rm QL}$ is nonnegative. Since every row of $P\Pi^\pi$ sums to one,
the row indexed by $(s,a)$ has sum
\begin{equation*}
  1-\alpha d(s,a)+\alpha\gamma d(s,a)
  =1-\alpha d(s,a)(1-\gamma),
\end{equation*}
which proves~\Cref{eq:direct-d-row-sum}. Applying
\Cref{lem:collatz-wielandt-row-sum} gives
\begin{equation*}
  \rho(A_\pi^{\rm QL})
  \ge \min_{(s,a)}\{1-\alpha d(s,a)(1-\gamma)\}
  =1-\alpha d_{\max}(1-\gamma),
\end{equation*}
which proves~\Cref{eq:direct-d-cw-lower-bound}.
\end{proof}

\subsection{Standard JSR lemmas}
\label{app:standard-jsr-and-selector-lemmas}

\begin{lemma}
\label{lem:convex-hull-jsr}
For a finite matrix family $\mathcal H$, one has
\begin{align*}
  \rho(\co(\mathcal H))=\rho(\mathcal H),
\end{align*}
where the left-hand side is computed over all convex combinations of matrices in
$\mathcal H$.
\end{lemma}
\begin{proof}
This is the convex-hull invariance property of the JSR; see,
for example, \citep[Chapter~1]{jungers2009joint}. We omit the proof.
\end{proof}

\begin{lemma}
\label{lem:similarity-jsr}
Let $S$ be nonsingular. For every bounded matrix family
$\mathcal H\subset\R^{m\times m}$,
\begin{equation*}
  \rho\left(\left\{S^{-1}AS:A\in\mathcal H\right\}\right)
  =\rho(\mathcal H).
\end{equation*}
\end{lemma}
This standard property follows from norm equivalence in the definition of the
JSR; see, for example, \citep[Chapter~1]{jungers2009joint}.

\begin{lemma}
\label{lem:block-triangular-jsr}
Let $\mathcal T$ be a bounded family of block upper triangular matrices of the
form
\begin{align*}
  T_i=\begin{bmatrix}B_i&E_i\\0&D_i\end{bmatrix}.
\end{align*}
Then
\begin{align*}
  \rho(\mathcal T)
  =\max\left\{
    \rho\left(\left\{B_i:T_i\in\mathcal T\right\}\right),
    \rho\left(\left\{D_i:T_i\in\mathcal T\right\}\right)
  \right\}.
\end{align*}
The same statement holds with any finite number of diagonal blocks.
\end{lemma}
\begin{proof}
This is the block upper triangular decomposition property of the joint spectral
radius; see, for example, \citep[Chapter~1]{jungers2009joint}. We omit the
proof.
\end{proof}

\subsection{Feature-space two-point difference identity}
\label{app:lfa-two-point-difference-identity}

The following identity is the two-point version of the fixed-point error system
in~\Cref{lem:lfa-cpd-difference-system}.  It is placed in the appendix because
the main feature-space development only needs the fixed-point form.

\begin{lemma}
\label{lem:app-lfa-cpd-two-point-difference}
For any two parameter vectors $\theta,\bar\theta\in\R^m$, there exists a
stochastic Bellman-difference selector $\mu_{\theta,\bar\theta}$ such that the
map in~\Cref{eq:lfa-cpd-standard-map} satisfies
\begin{equation}
  T(\theta)-T(\bar\theta)
  =A_{\mu_{\theta,\bar\theta}}^{\rm CLQL}(\theta-\bar\theta),
  \label{eq:lfa-cpd-difference-system}
\end{equation}
where
\begin{equation*}
  T(\theta):=\theta+\alpha G\{F(\Phi\theta)-\Phi\theta\}.
\end{equation*}
\end{lemma}
\begin{proof}
By~\Cref{lem:bellman_difference_selector}, there exists a stochastic policy
$\mu_{\theta,\bar\theta}$ such that
\begin{equation*}
  F(\Phi\theta)-F(\Phi\bar\theta)
  =\gamma P\Pi^{\mu_{\theta,\bar\theta}}\Phi(\theta-\bar\theta).
\end{equation*}
Subtracting the two feature updates gives
\begin{align*}
  T(\theta)-T(\bar\theta)
  &=(\theta-\bar\theta)
    +\alpha G\{F(\Phi\theta)-F(\Phi\bar\theta)
       -\Phi(\theta-\bar\theta)\}\\
  &=(\theta-\bar\theta)
    +\alpha G(\gamma P\Pi^{\mu_{\theta,\bar\theta}}-I)
       \Phi(\theta-\bar\theta)\\
  &=\{I_m+\alpha G(\gamma P\Pi^{\mu_{\theta,\bar\theta}}-I)\Phi\}
       (\theta-\bar\theta)\\
  &=A_{\mu_{\theta,\bar\theta}}^{\rm CLQL}(\theta-\bar\theta),
\end{align*}
which proves~\Cref{eq:lfa-cpd-difference-system}.
\end{proof}

\subsection{Jury stability criterion}
\label{app:jury-stability-criterion}

\begin{lemma}
\label{lem:jury-quadratic}
For the real quadratic
\begin{align*}
  p(z)=z^2-az+b,
\end{align*}
all roots of $p$ belong to $\left\{z\in\mathbb C:|z|<1\right\}$ if and only if
\begin{equation*}
  1-a+b>0,
  \qquad
  1+a+b>0,
  \qquad
  1-b>0.
\end{equation*}
This is the second-order Jury stability criterion \citep{jury1964ztransform}.
\end{lemma}

The proof is omitted; see \citep{jury1964ztransform}.

\subsection{Scalar heavy-ball acceleration conditions}
\label{app:scalar-heavy-ball-conditions}

\begin{proposition}
\label{app:prop-cpd-constant-mode-acceleration}
Under~\Cref{ass:cpd-constant-mode-range},
\begin{equation*}
  \rho(C)<\bigl(1-\alpha\delta(1-\gamma)\bigr)
  \quad\Longleftrightarrow\quad
  0<\alpha\eta<\bigl(1-\alpha\delta(1-\gamma)\bigr)^2.
\end{equation*}
\end{proposition}
\begin{proof}
By~\Cref{ass:cpd-constant-mode-range},
\begin{equation*}
  0<1-\alpha\delta(1-\gamma)<1.
\end{equation*}
The characteristic polynomial of $C$ is
\begin{equation*}
  \lambda^2-
  \left(1-\alpha\delta(1-\gamma)+\alpha\eta\right)\lambda+
  \alpha\eta=0.
\end{equation*}
The inequality
\begin{equation*}
  \rho(C)<1-\alpha\delta(1-\gamma)
\end{equation*}
is equivalent to all roots of the polynomial obtained by setting
\(\lambda=(1-\alpha\delta(1-\gamma))z\) lying in the open unit disk.  After
dividing by \((1-\alpha\delta(1-\gamma))^2\), this polynomial is
\begin{equation*}
  z^2-
  \frac{1-\alpha\delta(1-\gamma)+\alpha\eta}{1-\alpha\delta(1-\gamma)}z+
  \frac{\alpha\eta}{(1-\alpha\delta(1-\gamma))^2}=0.
\end{equation*}
For a real quadratic $z^2-az+b$, the second-order Jury criterion in~\Cref{lem:jury-quadratic} says that all
roots lie in the open unit disk if and only if
\begin{equation*}
  1-a+b>0,
  \qquad
  1+a+b>0,
  \qquad
  1-b>0.
\end{equation*}
Here these three inequalities become
\begin{align*}
  &1-
  \frac{1-\alpha\delta(1-\gamma)+\alpha\eta}{1-\alpha\delta(1-\gamma)}+
  \frac{\alpha\eta}{(1-\alpha\delta(1-\gamma))^2} \\
  &\qquad=
  \frac{\alpha\eta\,\alpha\delta(1-\gamma)}{(1-\alpha\delta(1-\gamma))^2}>0,\\
  &1+
  \frac{1-\alpha\delta(1-\gamma)+\alpha\eta}{1-\alpha\delta(1-\gamma)}+
  \frac{\alpha\eta}{(1-\alpha\delta(1-\gamma))^2} \\
  &\qquad=
  2+
  \frac{\alpha\eta}{1-\alpha\delta(1-\gamma)}+
  \frac{\alpha\eta}{(1-\alpha\delta(1-\gamma))^2}>0,\\
  &1-
  \frac{\alpha\eta}{(1-\alpha\delta(1-\gamma))^2}>0.
\end{align*}
Because \(0<1-\alpha\delta(1-\gamma)<1\), the first two inequalities hold
exactly when $\alpha\eta>0$, while the third is equivalent to
\(\alpha\eta<(1-\alpha\delta(1-\gamma))^2\).  This proves the claimed equivalence.
\end{proof}

\begin{proposition}
\label{app:prop-lfa-cpd-hb-common-acceleration}
Under~\Cref{ass:cpd-constant-mode-range},
\begin{equation*}
  \rho(C)<\bigl(1-\alpha\delta(1-\gamma)\bigr)
  \quad\Longleftrightarrow\quad
  0<\alpha\eta<\bigl(1-\alpha\delta(1-\gamma)\bigr)^2.
\end{equation*}
\end{proposition}
\begin{proof}
Under~\Cref{ass:cpd-constant-mode-range},
$0<1-\alpha\delta(1-\gamma)<1$, and
\begin{equation*}
  C=
  \begin{bmatrix}
    1-\alpha\delta(1-\gamma)+\alpha\eta&-\alpha\eta\\
    1&0
  \end{bmatrix}.
\end{equation*}
Its characteristic polynomial is $\lambda^2-
\left(1-\alpha\delta(1-\gamma)+\alpha\eta\right)\lambda+\alpha\eta$. The same rescaling
$\lambda=\bigl(1-\alpha\delta(1-\gamma)\bigr)z$ used in
\Cref{app:prop-cpd-constant-mode-acceleration} gives the second-order Jury
inequalities~\citep{jury1964ztransform}
\begin{align*}
  &\frac{\alpha\eta\,\alpha\delta(1-\gamma)}{(1-\alpha\delta(1-\gamma))^2}>0,\\
  &2+\frac{\alpha\eta}{1-\alpha\delta(1-\gamma)}
  +\frac{\alpha\eta}{(1-\alpha\delta(1-\gamma))^2}>0,\\
  &1-\frac{\alpha\eta}{(1-\alpha\delta(1-\gamma))^2}>0.
\end{align*}
Since $0<1-\alpha\delta(1-\gamma)<1$, these inequalities are equivalent to
$0<\alpha\eta<\bigl(1-\alpha\delta(1-\gamma)\bigr)^2$, which proves the claimed equivalence.
\end{proof}

\end{document}